%% file: main.tex
\documentclass{article}

\usepackage[utf8]{inputenc}
\usepackage[T1]{fontenc}
\usepackage[english]{babel}
\usepackage{hyperref}
\usepackage{dsfont}
\usepackage{booktabs}       
\usepackage{amsfonts}       
\usepackage{amsmath, amssymb}
\usepackage[a4paper,hscale=0.7,vscale=0.75,centering]{geometry}
\usepackage{fullpage}
\usepackage{bbm}
\usepackage{graphicx}
\usepackage{float}
\usepackage{xcolor}
\usepackage{subcaption}
\usepackage{wrapfig}
\usepackage{algorithm}
\usepackage[noend]{algpseudocode}
\usepackage[square,sort]{natbib}
\usepackage{authblk}
\usepackage{amsmath, amsthm}
\usepackage{subcaption}
\usepackage{graphicx}
\usepackage{mathtools}
\usepackage{multirow}
\usepackage{amssymb}
\usepackage{bm}
\usepackage{upgreek}
\usepackage{microtype}
\usepackage{amsmath,amsfonts,amssymb}
\usepackage{nicefrac}       
\usepackage{mathtools}
\usepackage{soul}

\usepackage[acronym,nowarn]{glossaries}
\makeglossaries

\newacronym{DALMC}{{\textsc{\small DALMC}}}{diffusion annealed Langevin Monte Carlo}
\newacronym{DALD}{{\textsc{\small DALD}}}{diffusion annealed Langevin dynamics}

\newacronym{SGM}{SGM}{score-based generative model}
\newacronym{OU}{OU}{Ornstein-Uhlenbeck}
\newacronym{SDE}{SDE}{stochastic differential equation}
\newacronym{KL}{KL}{Kullback-Leibler}

\newcommand{\md}{\mathrm{d}}

\usepackage{graphicx}
\usepackage{dsfont}
\DeclareMathOperator*{\argmin}{arg\,min}

\DeclareMathOperator*{\kl}{KL}
\DeclareMathOperator*{\pid}{\pi_{\text{data}}}
\DeclareMathOperator*{\score}{\text{score}}

\newtheorem{theorem}{Theorem}[section]

\newtheorem{proposition}[theorem]{Proposition}
\newtheorem{lemma}[theorem]{Lemma}
\newtheorem{remark}[theorem]{Remark}
\newtheorem{corollary}[theorem]{Corollary}

\newcommand\independent{\protect\mathpalette{\protect\independenT}{\perp}}
\def\independenT#1#2{\mathrel{\rlap{$#1#2$}\mkern2mu{#1#2}}}

\usepackage{xcolor}

\usepackage{hyperref}

\usepackage[capitalize,noabbrev]{cleveref}
\newtheorem{assumption}{\textbf{A}\hspace{-2pt}}
\Crefname{assumption}{\textbf{A}\hspace{-3pt}}{\textbf{H}\hspace{-3pt}}
\crefname{assumption}{\textbf{A}}{\textbf{A}}

\title{Non-asymptotic Analysis of Diffusion Annealed Langevin Monte Carlo for Generative Modelling}
\author[1]{Paula Cordero-Encinar}
\author[1]{\"{O}. Deniz Akyildiz}
\author[1, 2]{Andrew B. Duncan}
\affil[1]{Imperial  College London}
\affil[2]{The Alan Turing Institute}
\begin{document}

\maketitle

\vskip 0.3in

\begin{abstract}
We investigate the theoretical properties of general diffusion (interpolation) paths and their Langevin Monte Carlo implementation, referred to as \gls*{DALMC}, under weak conditions on the data distribution. Specifically, we analyse and provide non-asymptotic error bounds for the annealed Langevin dynamics where the path of distributions is defined as Gaussian convolutions of the data distribution as in diffusion models. We then extend our results to recently proposed heavy-tailed (Student's $t$) diffusion paths, demonstrating their theoretical properties for heavy-tailed data distributions for the first time. Our analysis provides theoretical guarantees for a class of score-based generative models that interpolate between a simple distribution (Gaussian or Student's $t$) and the data distribution in \textit{finite time}. This approach offers a broader perspective compared to standard score-based diffusion approaches, which are typically based on a forward \gls*{OU} noising process.
\end{abstract}

\section{Introduction}
\label{section:introduction}
\Glspl*{SGM} \citep{song2020score, ho2020denoising} have become immensely popular in recent years due to their excellent performance in generating high-quality data. 
This success has led to widespread adoption across various generative modelling tasks, e.g., image generation \citep{dhariwal2021diffusion, Rombach_2022_CVPR, saharia2022photorealistic}, audio generation \citep{Ruan_2023_CVPR}, 
reward maximisation \citep{pmlr-v162-janner22a, he2023diffusion}. 
Additionally, their remarkable performance has sparked significant interest within the theoretical community to better understand the structure and properties of these models \citep{lee2022convergence, Chen2022ImprovedAO,
chen2023sampling, benton2024nearly}.

The goal of generative modelling is to learn the underlying probability distribution $\pid$ from a given set of samples.
Diffusion models, a particular class of \glspl*{SGM}, achieve this by using a \textit{forward process}, typically an \gls*{OU} process, to construct a path of probability distributions from the data distribution towards a simpler one -- a Gaussian. 
The time-reversed process can be characterised \citep{anderson1982reverse} but necessitates the knowledge of the scores of the marginal distributions along this path. 
These scores are usually intractable - hence they are learnt by noising the data and applying score matching techniques \citep{hyvarinen2005estimation, pascal_score_matching, pmlr-v115-song20a}. 
The learnt scores are then used to sample from the path by discretising the time-reversed diffusion process \citep{song2020score}.

\begin{figure}[t]
\vskip 0.2in
  \centering
  \begin{subfigure}{0.33\textwidth}
         \includegraphics[trim={45 65 45 65}, clip, width=\textwidth]{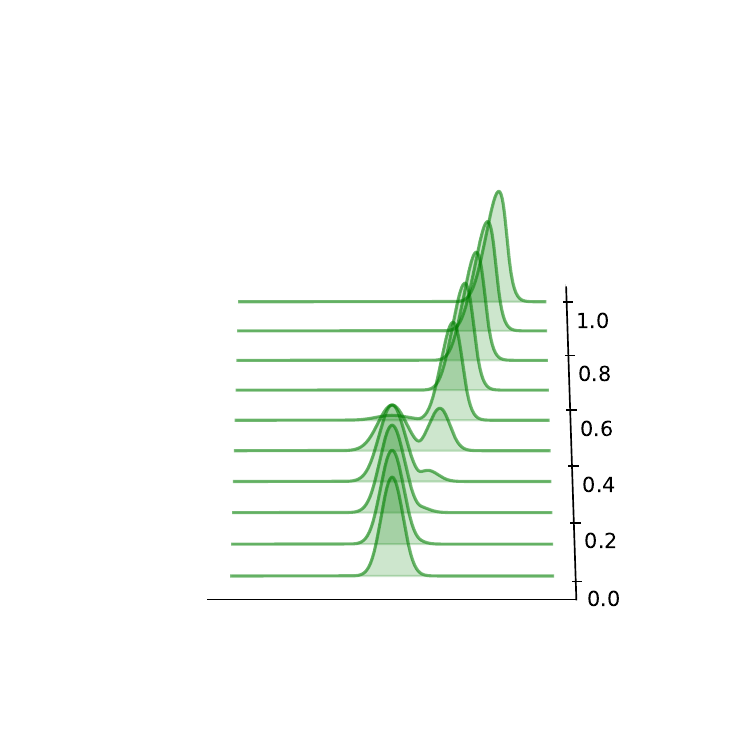}
         \caption{Geometric path}
         \label{fig:geometric_path_mixture_u_g}
     \end{subfigure}
     \begin{subfigure}{0.33\textwidth}
         \includegraphics[trim={45 65 45 65}, clip, width=\textwidth]{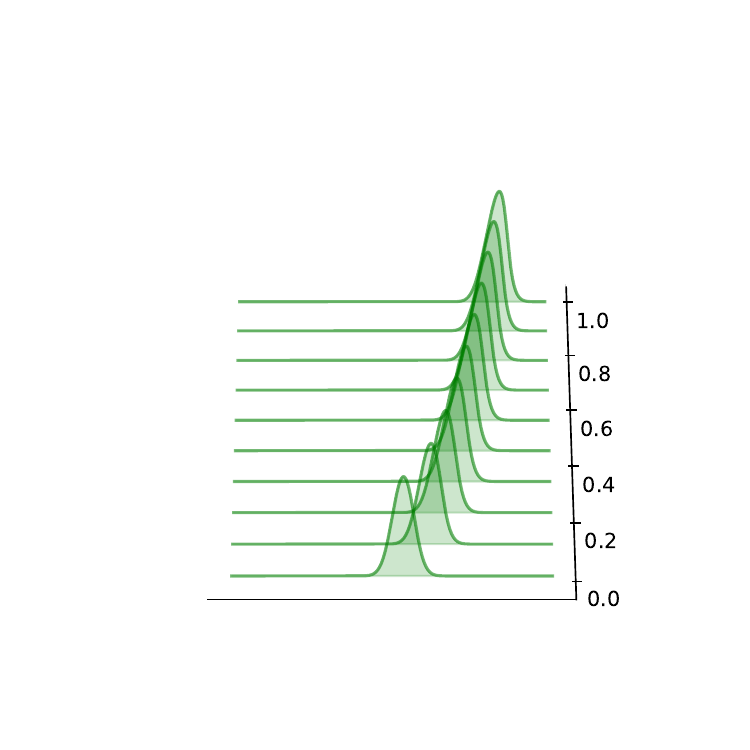}
        \caption{Diffusion path}
         \label{fig:convolutional_path_mixture_u_g}
     \end{subfigure}
    
     \caption{A visual comparison of the geometric path versus the diffusion path for $(\mu_t)_{t\in[0,1]}$. The \textit{base} distribution is given by $\mu_0 := \mathcal{N}(0,1)$ and the data distribution, $\mu_1 := \pid$, is a mixture of a Gaussian and a smoothed uniform distribution (see Section~\ref{sec:gaussian_diffusion}). As observed by \citet{chehab2024provableconvergencelimitationsgeometric}, the geometric path in (a) creates intermediate multimodal distributions which are hard to sample from. In contrast, the diffusion path in (b) stays unimodal throughout, offering more favourable properties.}
  \label{fig:convolutional_vs_geometric}
\vskip -0.2in
\end{figure}
While the forward \gls*{OU} process is mathematically convenient, it does not capture the whole idea of bridging distributions and  requires infinite time to interpolate between the data distribution $\pid$ and a Gaussian measure. In practice, however, diffusion models consider the evolution of the \gls*{OU} process only up to a finite final time $T$. Thus, the path does not fully bridge $\pid$ and a standard Gaussian.
During generation, these models instead evolve samples along a sequence of interpolated distributions between the final marginal distribution of the OU process at time $T$ and $\pid$ (although, in practice, they are initialised from a Gaussian). 
Specifically, this interpolation is characterised by defining intermediate random variables $X_t \sim \mu_t$ \footnote{In our case, the \text{base} (simple) distribution is defined at time $0$ as $\mu_0$, and the data distribution is defined at time $T$, $\mu_T = \pid$. This contrasts with standard diffusion models where the base distribution is defined at time $T$ and the data distribution at time $0$.} \citep{chehab2024practicaldiffusionpathsampling} as
\begin{align}\label{eq:one-sided-interpolant}
X_t = \sqrt{\lambda_t} X + \sqrt{1-\lambda_t} Z,
\end{align}
for $t \in [0, T]$, where $X \sim \pi_{\text{data}}$, $Z \sim \mathcal{N}(0, I)$ is independent of $X$ and a  schedule $\lambda_t = \min\{1, e^{-2(T-t)}\}$. 

The interpolation perspective of diffusion models has been investigated, see, e.g., \citet{albergo2023stochastic, gao2024gaussian}. Notably, the path in Eq.~\eqref{eq:one-sided-interpolant} is a special case of the \textit{one-sided stochastic interpolants} \citep{albergo2023stochastic}. 
As outlined in these works, the  reverse process can be made to exactly interpolate between a base distribution $\nu$ and $\pi_{\text{data}}$ in finite time by using an appropriate schedule $\lambda_t$ and introducing control terms in the corresponding \glspl*{SDE}. 
Similarly to the score term in diffusion models, these control terms are intractable and need to be learnt.

In this work, we adopt a practical approach to general linear interpolation paths between a simple base distribution $\nu$ and $\pi_{\text{data}}$, that is, $X_t = \sqrt{\lambda_t} X + \sqrt{1-\lambda_t} Z$, where $X\sim\pid$, $Z\sim\nu$ independent of $X$ and $\lambda_t\in[0,1]$, $\lambda_T=1$. 
In particular, we explore 
the behaviour of Langevin dynamics driven by the gradients of $\log\mu_t$ for $t \in [0, T]$, where $\mu_t$ are the intermediate distributions, i.e., $X_t\sim \mu_t$. 
Our approach is akin to earlier generative modelling methods based on \textit{annealed Langevin dynamics} \citep{song2019generative} which led to the development of diffusion models. However, there has been limited work analysing these methods under minimal assumptions on $\pid$. 
\citet{block2022generativemodelingdenoisingautoencoders} provide the first theoretical analysis in Wasserstein distance under smoothness and dissipativity of the data distribution. They show that the error depends exponentially on the dimension.
In contrast, \citet{lee2022convergence} provides a non-asymptotic bound in total variation under smoothness conditions and a bounded log-Sobolev constant of the data distribution. Specifically, we make the following contributions.

\paragraph{Contributions} 
\begin{itemize}
\item We provide an analysis of annealed Langevin dynamics methods driven by general linear interpolation paths between $\nu$ and $\pi_{\text{data}}$, which we term diffusion annealed Langevin Monte Carlo (\gls*{DALMC}).
In the case where $\nu$ is a Gaussian distribution, we derive non-asymptotic convergence bounds in \gls*{KL} divergence under different assumptions.

By assuming that $\pid$ has a finite second-order moment $M_2$, 
 $\log \pid$ has Lipschitz gradients and either $\pid$ is strongly convex outside a ball or $\nabla^2\log\pid$ decays to $0$ sufficiently fast (as is the case for Student's $t$-like distributions), we show in Corollary~\ref{corollary:complexity_bounds} that \gls*{DALMC} requires $\mathcal{O}\left(\frac{d (M_2 \vee d)^2 L_{\max}^2}{\varepsilon^6}\right)$ steps to achieve $\varepsilon^2$-accurate sampling from $\pid$ in \gls*{KL} divergence with a sufficiently accurate score estimator. Here, $d$ is the dimension of the data and $L_{\max} := \max_{t \in [0,T]} L_t$ where $L_t$ denotes the Lipschitz constant of $\nabla\log\mu_t$, which we prove to be finite in Lemma~\ref{lem:regularity_of_gaussian_path}, improving the results of \citet[Proposition 20]{gao2024gaussian} under the specified conditions.
Furthermore, under slightly less restrictive assumptions involving smoothness of $\pid$ with constant $L_\pi$, bounded second order moment $M_2$ and $\mathbb{E}_{\pid} \left\Vert\nabla \log\pid\left(Y\right)\right\Vert^{8}\leq K_\pi^2$, we demonstrate that the data distribution can be approximated to $\varepsilon^2$-accuracy in \gls*{KL} divergence with $\mathcal{O}\left(\frac{(M_2 \vee d)^2(d^2\vee L_\pi^2d \vee K_\pi) L_\pi}{\varepsilon^6}\right)$ steps. 
To the best of our knowledge, these are the first results obtained in \gls*{KL} divergence for these Langevin-dynamics driven generative models \citep{song2019generative}.

\item We then extend this analysis into recent heavy-tailed diffusion models \citep{pandey2024heavy} based on Student's $t$ noising distributions, that is, when the base distribution $\nu$ is chosen to be a Student's $t$ distribution.
In this case, assuming that the data distribution is smooth, has a finite second-order moment and exhibits a tail behaviour similar to that of a multivariate Student's $t$ distribution, we show that \gls*{DALMC} can be used to sample from the data distribution with the same complexity as the Gaussian case.  
As far as we are aware, this is the first analysis of heavy-tailed diffusion models with explicit complexity estimates.
\item We show that, under certain conditions on the covariances, a mixture of Gaussians with different covariances satisfy smoothness conditions and is strongly log-concave outside of a ball, implying a finite log-Sobolev constant. 
This result is of independent interest, as most analyses of Gaussian mixtures in the literature primarily focus on the equal covariance setting.
\end{itemize}

The rest of the paper is organised as follows. Section~\ref{sec:background} presents our setting and necessary background. Section~\ref{sec:gaussian_diffusion}, provides a non-asymptotic analysis of the general diffusion paths with Gaussian base distribution and their implementation via Langevin dynamics.
In Section~\ref{sec:heavy_tailed_diffusion}, we extend our analysis to heavy-tailed diffusion models.
Section~\ref{sec:related_work} discusses related literature, followed by the conclusion.

\paragraph{Notation}
Let $d$ be the dimension of data. Let $A, B$ be square matrices of the same dimension, we say $A\preccurlyeq B$ if $B-A$ is a positive semidefinite matrix and $\Vert\cdot\Vert_F$ denotes the Frobenius norm. For $a, b>0$, we write $a\lesssim b$ or $a=\mathcal{O}(b)$ to indicate that $a\leq C b$ for an absolute constant $C\geq 0$, and $a\asymp b$ if $a=\mathcal{O}(b)$ and $b=\mathcal{O}(a)$. For $f:\mathbb{R}^d\to\mathbb{R}^d$ and a probability measure $\mu$ on $\mathbb{R}^d$, we define $\Vert f\Vert_{L^2(\mu)}:=\left(\int\Vert f\Vert \md \mu\right)^{1/2}$ and  $M_{2} := \mathbb{E}_{\pid}[\Vert X\Vert^2]$.

\section{Generative Modelling via Diffusion Paths}
\label{sec:background}

We present the background and setting for our analysis.

\subsection{Diffusion Paths}
In practice, implementing the reverse process in diffusion models consists in sampling along a path of probability distributions $(\mu_t)_{t\in [0,T]}$, which starts at a simple distribution $\mu_0$ and ends at an arbitrarily complex data distribution $\mu_T = \pi_{\text{data}}$. 
In particular, when the forward process is an \gls*{OU} process and evolves the data distribution for time $T$, the starting distribution of the reversed process takes the form ${e^{dT}}\pid(e^T x)*\ \mathcal{N}(0, (1-e^{-2T})I)$ and the interpolated distributions $(\mu_t)_{t}$ are the marginals of the \gls*{OU} process.
Building on this, we can describe a more general version of the probability distribution paths that diffusion models attempt to sample from \citep{chehab2024practicaldiffusionpathsampling}, as 
\begin{align}
    \mu_t(x) = \frac{\pi_{\text{data}}(x/\sqrt{\lambda_t})}{{\lambda_t}^{d/2}}  * \frac{\nu \left({x}/{\sqrt{1-\lambda_t}}\right)}{(1-\lambda_t)^{d/2}},\label{eq:convolutional_path}
\end{align}
where $*$ denotes the convolution operation, $\nu$ describes the base or \textit{noising} distribution, and $\lambda_t$ is an increasing function called schedule, such that, $\lambda_t\in[0, 1]$ and $\lambda_T=1$. We refer to the probability path $(\mu_t)_{t\in [0,T]}$ in~\eqref{eq:convolutional_path} as the \textit{diffusion path}. In the setting of the \gls*{OU} (i.e. variance preserving) process, $\lambda_t$ corresponds to $\lambda_t = \min\{1, e^{-2(T-t)}\}$.

The diffusion path with the \gls*{OU} schedule has demonstrated very good performance in the generative modelling literature and has recently started to be explored for sampling \citep{huang2024reverse, richter2024improved, vargas2024transport}. For instance, \citet{phillips2024particle} empirically observed that the diffusion path may have a more favourable geometry for the Langevin sampler than the geometric path, obtained by taking the geometric mean of the base and target distributions, as is typically done in annealing due to the tractability of the score (Figure \ref{fig:convolutional_vs_geometric}).
 
While successful, the use of the \gls*{OU} process presents some challenges in practice. 
As mentioned earlier, the forward \gls*{OU} process cannot reach $\nu$ in finite time, meaning that, in theory the reversed path starts from a non-Gaussian distribution $\mu_0$. However, in practice, the paths are initialised from Gaussians, introducing a bias that is present in error bounds \citep{Chen2022ImprovedAO, chen2023sampling, benton2024nearly}. In our setting, by selecting an appropriate schedule for the diffusion path \eqref{eq:convolutional_path}, which satisfies $\lambda_0=0$ and $\lambda_T=1$, the path of probability distributions $(\mu_t)_{t \in [0, T]}$ can interpolate exactly between $\mu_0 = \nu$ and $\mu_T=\pid$ in finite time, unlike the \gls*{OU} process. This formulation is equivalent to that of linear one-sided stochastic interpolants which can also be realised through \glspl*{SDE} \citep[Theorem 5.3]{albergo2023stochastic}.

We will next explore an alternative approach for generative modelling with general linear diffusion paths, namely, running annealed Langevin dynamics on paths $(\mu_t)_{t\in [0, T]}$ that are constructed to meet the correct marginals.

\subsection{Annealed Langevin Dynamics for Diffusion Paths}

For general diffusion paths, the ``reverse process'' cannot be described by a closed form \gls*{SDE}. 
While \citet{albergo2023stochastic}, estimate the intractable drift term of the \gls*{SDE} using neural networks, their approach can experience numerical instabilities at $t=T$ (see \citet[Section 6]{albergo2023stochastic}) due to singularities in the drift term.
Therefore, in this work, we focus on \textit{annealed Langevin dynamics} \citep{song2019generative} to explicitly implement a sampler along the diffusion path, avoiding the extra control terms introduced in \citet{albergo2023stochastic}. Note that the score at each time $t$ can be learnt via score matching techniques, as in \citet{song2019generative}.

Our annealed Langevin dynamics consists of running a time-inhomogeneous Langevin \gls*{SDE}, where the drifts are given by the scores of reparametrised probability distributions from the diffusion path $(\hat{\mu}_t = \mu_{\kappa t})_{t\in[0,T/\kappa]}$, for some $0<\kappa<1$. That is, we will use the following \gls*{SDE}
\begin{equation}\label{eq:annealed_langevin_sde}
 \md X_t = \nabla \log \hat{\mu}_t(X_t) \md t + \sqrt{2} \md B_t\quad t\in[0, T/\kappa],
\end{equation}
where {$X_0 \sim \mu_0=\nu$ and $(B_t)_{t\geq 0}$ is a Brownian motion. 
We refer to \eqref{eq:annealed_langevin_sde} as \textit{diffusion annealed Langevin dynamics}. 
This strategy does provide a viable alternative to implement interpolation paths as the scores can be learnt. 
In particular, we consider the diffusion annealed Langevin Monte Carlo (\gls*{DALMC}) algorithm given by a simple Euler-Maruyama discretisation of \eqref{eq:annealed_langevin_sde}  and the use of a score approximation function $s_\theta(x, t)$ \citep{song2019generative}:
\begin{equation}\label{eq:annealed_langevin_mcmc_algorithm_score_approx}
    X_{l+1} = X_l + h_l s_{\theta}(X_l, t_l) + \sqrt{2 h_l} \xi_l,
\end{equation}
where $h_l > 0$ is the step size, $\xi_k\sim \mathcal{N}(0,I)$, $s_\theta(x, t)$ approximates $\nabla\log\hat{\mu}_{t}(x)$, $l\in\{1,\dots, M\}$ and $0=t_0<\dots<t_M=T/\kappa$ is a discretisation of the interval $[0,T/\kappa]$.

It is important to note that, even if simulated exactly, diffusion annealed Langevin dynamics introduces a bias, as the marginal distributions of the solution of the \gls*{SDE} \eqref{eq:annealed_langevin_sde} do not exactly correspond to $(\hat{\mu}_t)_t$, unlike in the stochastic interpolants formulation \citep{albergo2023stochastic}. One of the contributions of our work will be to quantify this bias non-asymptotically.
A key component in determining the effectiveness of the diffusion annealed Langevin dynamics will be the action of the curve of probability measures $\mu =(\mu_t)_{t\in[0, T]}$ interpolating between the base distribution and the data distribution, denoted by $\mathcal{A}(\mu)$. As noted by \citet{guo2024provablebenefitannealedlangevin}, the action serves as a measure of the cost of transporting $\nu$ to $\pid$ along the given path. Formally, the action of an absolutely continuous curve of probability measures \citep{lisini2007characterization} with finite second-order moment is defined as follows 
\begin{equation*}
  \mathcal{A}(\mu):=\int_0^T \lim_{\delta\to 0}\frac{W_2(\mu_{t+\delta}, \mu_t)}{\vert \delta\vert}.  
\end{equation*}
Based on Theorem 1 from \citet{guo2024provablebenefitannealedlangevin}, we have that the \gls*{KL} divergence between the path measure of the diffusion annealed Langevin dynamics \eqref{eq:annealed_langevin_sde}, $\mathbb{P}_{\text{DALD}} = (p_{t,\text{DALD}})_{t\in[0, T/\kappa]}$, and that of a reference \gls*{SDE} such that the marginals at each time have distribution $\hat{\mu}_t$, $\mathbb{P}=(\hat{\mu}_{t})_{t\in[0, T/\kappa]}$, can be bounded in terms of the action.
In particular, when $p_0 = p_{0, \text{DALD}}$, it follows from Girsanov's theorem that
\begin{equation*}
 \kl\left(\mathbb{P}\;||\mathbb{P}_{\text{DALD}}\right)\leq \kappa\mathcal{A}(\mu).
\end{equation*}
See Theorem \ref{theorem:preliminaries_continuous_time_kl_bound} in Appendix \ref{appendix:background} for the proof and further details. Note that by the data processing inequality, we have that $ \kl\left(\pid\;||p_{T/\kappa,\text{DALD}}\right)\leq \kl\left(\mathbb{P}\;||\mathbb{P}_{\text{DALD}}\right)$, meaning that the \gls*{KL} divergence between the data distribution and the final marginal distribution of the diffusion annealed Langevin dynamics \eqref{eq:annealed_langevin_sde} is bounded, provided that the action is finite. In that case by choosing $\kappa = \mathcal{O}(\varepsilon^2/\mathcal{A}(\mu))$, we ensure that $\kl\left(\pid\;||p_{T/\kappa,\text{DALD}}\right)\lesssim \varepsilon^2$.

\subsection{Initial Assumptions}
In what follows, we will provide an in-depth analysis of the \gls*{DALMC} algorithm when the base distribution $\nu$ is Gaussian or multivariate Student's $t$ distribution. The latter relates to recent heavy-tailed diffusion models \citep{pandey2024heavy}.  
Our results in both cases are based on the following assumptions, with additional ones introduced later as necessary.

First, as is typical in the diffusion model literature we require an $L^2$ accurate score estimator \citep{Chen2022ImprovedAO, chen2023sampling}.
\begin{assumption}\label{assumption:score_approximation}
    The score approximation function $s_\theta(x, t)$ satisfies 
    \begin{equation*}
        \sum_{l=0}^{M-1} h_l\mathbb{E}_{\hat{\mu}_t}\left[\left\Vert \nabla \log \hat{\mu}_l(X_{t_l}) - s_\theta(X_{t_l}, t_l)\right\Vert^2\right] \leq \varepsilon_{score}^2.
    \end{equation*}
    where $0=t_0<t_1<\dots<t_M=T/\kappa$ is a discretisation of the interval $[0,T/\kappa]$.
\end{assumption}

\begin{assumption}\label{assumption:finite_second_order_moment} The data distribution $\pi_{\text{data}}$ has a finite second-order moment, that is, $M_2 = \mathbb{E}_{\pid}[\Vert X\Vert^2] < \infty$.
\end{assumption}

\section{Gaussian Diffusion Paths}\label{sec:gaussian_diffusion}
\label{section:analysis}
In this section, we focus on analysing algorithms to simulate the diffusion path $(\mu_t)_{t\in [0,T]}$ defined in \eqref{eq:convolutional_path} when the base distribution $\nu$ is Gaussian, $\nu\sim \mathcal{N}(m_\nu, \sigma^2 I)$. For simplicity, we will assume that $\pid$ has mean $0$ and set $m_\nu=0$.  
This diffusion path has the remarkable property, illustrated in Figure~\ref{fig:convolutional_vs_geometric}, that when $\pid$ has finite log-Sobolev and Poincaré constants, these constants remain uniformly bounded along the entire path, as summarised in the following result.
\begin{proposition}\label{prop:bounded_log_sobolev} If $\pid$ has a finite log-Sobolev constant $C_{\text{LSI}}(\pid)$, respectively Poincaré constant $C_{\text{PI}}(\pid)$, the Gaussian diffusion path $(\mu_t)_{t \in [0,T]}$ defined in \eqref{eq:convolutional_path} with base distribution $\nu\sim\mathcal{N}(0, \sigma^2I)$ satisfies for all $t\in[0, T]$
\begin{align*}
    C_{\text{LSI}}(\mu_t)&\leq \lambda_t C_{\text{LSI}}(\pid) + (1-\lambda_t) C_{\text{LSI}}(\nu), \\
    C_{\text{PI}}(\mu_t)&\leq \lambda_t C_{\text{PI}}(\pid) + (1-\lambda_t) C_{\text{PI}}(\nu),
\end{align*}
respectively, where $C_{\text{LSI}}(\nu) = C_{\text{PI}}(\nu) = \sigma^{2}$.
\end{proposition}
The proof follows immediately from \citet[Propositions 2.3.3 and 2.3.7]{sinho_book}. This result is highly favourable, as, unlike geometric annealing \citep{chehab2024provableconvergencelimitationsgeometric}, the log-Sobolev and Poincaré constants remain uniformly bounded along the entire path by the worst constant independently of the distance between $\pid$ and $\nu$. 
We can visually observe this in Figure \ref{fig:convolutional_vs_geometric}, when the data distribution is given by a mixture of a Gaussian and a smoothed uniform distribution, $\pid = (1-e^{-m^2/4)}\mathcal{N}(m, 1) + e^ {-m^2/4} u_m$,
where $u_m$ is the smoothed uniform distribution on $I_m=[-m , 2m]$ for $m=10$ \citep{chehab2024provableconvergencelimitationsgeometric}.

Motivated by this, we analyse the diffusion annealed Langevin dynamics \eqref{eq:annealed_langevin_sde} to simulate from \eqref{eq:convolutional_path}. 

\subsection{Analysis of the Gaussian Diffusion Path}\label{sec:gaussian_diffusion_analysis}
We start by analysing the properties of the \textit{Gaussian diffusion path} $(\mu_t)_{t\in [0, T]}$. 
\paragraph{Smoothness of $(\mu_t)_t$.} We consider the following assumption on the Lipschitz continuity of the scores $\nabla\log\mu_t$.
\begin{assumption}\label{assumption:lipschitz_score_across_convolutional_diffusion}
    For all $t\in[0, T]$, the scores of the intermediate distributions $\nabla\log\mu_t(x)$ are Lipschitz with finite constant $L_t$.
\end{assumption}
\Cref{assumption:finite_second_order_moment} and \Cref{assumption:lipschitz_score_across_convolutional_diffusion} are sufficient for one of our non-asymptotic analyses of the Gaussian diffusion path (Theorem \ref{theorem:discretisation_analysis_convolutional_path}). 
However, \Cref{assumption:lipschitz_score_across_convolutional_diffusion} is generally difficult to verify. 
Therefore, we introduce two \textit{alternative} assumptions, \Cref{assumption:grad_log_lipschitzness_convexity_outside_of_a_ball} and \Cref{assumption:grad_log_lipschitzness_hessian_decay}, which separately ensure that $(\nabla \log\mu_t)_{t\in[0, T]}$ satisfies assumption \Cref{assumption:lipschitz_score_across_convolutional_diffusion}. In particular, we show that \Cref{assumption:grad_log_lipschitzness_convexity_outside_of_a_ball} is satisfied by a mixture of Gaussians with different covariances, given certain conditions on the covariances. While assumption \Cref{assumption:grad_log_lipschitzness_hessian_decay} is shown to hold for heavy-tailed data distributions. 
\begin{assumption}\label{assumption:grad_log_lipschitzness_convexity_outside_of_a_ball}
     The data distribution $\pi_{\text{data}}$ has density with respect to Lebesgue, which we write  $\pi_{\text{data}} \propto e^{-V_\pi}$. The potential $V_\pi$ has Lipschitz continuous gradients, with Lipschitz constant $L_\pi$. In addition, $V_\pi$ is strongly convex outside of a ball of radius $r$ with convexity parameter $M_\pi>0$, that is,
     \begin{equation*}
         \inf_{\Vert x\Vert \geq r} \nabla^2 V_\pi \succcurlyeq M_{\pi} I, \quad \inf_{\Vert x\Vert < r} \nabla^2 V_\pi \succcurlyeq -L_{\pi} I.
     \end{equation*}
\end{assumption}

In Lemma \ref{lemma:implications_between_assumptions} of the Appendix, we demonstrate that assumption \Cref{assumption:grad_log_lipschitzness_convexity_outside_of_a_ball} extends the standard assumption on the data distribution that $\pid$ is modelled as a convolution of a compactly supported distribution $\Tilde{\pi}$ and a Gaussian, see, e.g., \citet[Theorem~1]{saremi2024chain} or \citet[Assumption~0]{grenioux2024stochastic}, under some conditions on the compact support of $\Tilde{\pi}$.
Additionally, we prove that a mixture of Gaussians with different covariances satisfies assumption \Cref{assumption:grad_log_lipschitzness_convexity_outside_of_a_ball} under some mild conditions on the covariances (see Lemma \ref{lemma:example_mixture_gaussians_satisfies_assumption} and Remark \ref{remark:counter_example_mixture_gaussian} for a further discussion).
However, Lemma  \ref{lemma:d_mixture_gaussians_not_expressed_as_convoltuion_with_compactly_supported} shows that, in general, a mixture of Gaussians cannot be expressed as a convolution of a compactly supported measure with a Gaussian.
We consider the results regarding the mixture of Gaussians to hold independent significance, as we could not find explicit results in the literature addressing the smoothness properties in this case.

Leveraging the existence of a smooth strongly convex approximation of $V_\pi$ \citep{doi:10.1073/pnas.1820003116} and the Holley-Stroock perturbation lemma \citep{RefWorks:RefID:85-holley1987logarithmic}, we show that under \Cref{assumption:grad_log_lipschitzness_convexity_outside_of_a_ball}, $\pi_{\text{data}}$  satisfies a log-Sobolev inequality with a finite constant which itself implies a finite Poincaré constant (Lemma \ref{lemma:assumption_implies_LSI}) -- which is sufficient for Proposition~\ref{prop:bounded_log_sobolev} to hold. As a consequence, \Cref{assumption:grad_log_lipschitzness_convexity_outside_of_a_ball} implies that the data distribution $\pid$ has finite second order moment (i.e. \Cref{assumption:grad_log_lipschitzness_convexity_outside_of_a_ball} $\Rightarrow$ \Cref{assumption:finite_second_order_moment}).

On the other hand, heavy-tailed data distributions, such as Student's $t$-like distributions, do not satisfy assumption \Cref{assumption:grad_log_lipschitzness_convexity_outside_of_a_ball}, since their potential $V_\pi$ is not strongly convex outside of a ball. Specifically, the Hessian of the potential tends to zero as $\Vert x\Vert$ tends to infinity. 
We provide the following alternative assumption for heavy-tailed data distributions.
\begin{assumption}\label{assumption:grad_log_lipschitzness_hessian_decay}
     The data distribution $\pid$ has density with respect to Lebesgue, which we write  $\pi\propto e^{-V_\pi}$. The potential $V_\pi$ has Lipschitz continuous gradients, with Lipschitz constant $L_\pi$. In addition, $\nabla^2 V_\pi(x)$ decays to 0 with order $\mathcal{O}(\Vert x\Vert^{-2}I)$ as $\Vert x\Vert$ tends to $\infty$. That is, outside of a ball of radius $r$ we have that 
     \begin{equation*}
        - \frac{I}{\alpha_1 +\alpha_2 \Vert x\Vert^{2}} \preccurlyeq \nabla^2 V_\pi(x) \preccurlyeq \frac{I}{\beta_1 +\beta_2 \Vert x\Vert^{2}} \; \; \Vert x\Vert > r, 
     \end{equation*}
     where $\alpha_1, \alpha_2,\beta_1,\beta_2\in\mathbb{R}$.
\end{assumption}
In Appendix \ref{appendix:comments_assumption_hessian_decay} we show that multivariate Student's $t$ distributions of the form
\begin{equation*}
    \pid(x) \propto \left( 1 + \frac{1}{\alpha} (x-\mu)^{\intercal} \Sigma^{-1} (x-\mu)\right)^{-( \alpha + d)/2}
\end{equation*}
satisfy assumption \Cref{assumption:grad_log_lipschitzness_hessian_decay}.

The following Lemma establishes that assumption \Cref{assumption:lipschitz_score_across_convolutional_diffusion} holds when the data distribution satisfies either \Cref{assumption:grad_log_lipschitzness_convexity_outside_of_a_ball} or \Cref{assumption:grad_log_lipschitzness_hessian_decay}.

\begin{lemma}\label{lem:regularity_of_gaussian_path} Under \Cref{assumption:grad_log_lipschitzness_convexity_outside_of_a_ball} or \Cref{assumption:grad_log_lipschitzness_hessian_decay}, we have that for all $t\in[0,T]$ the score $\nabla \log \mu_t(x)$ is Lipschitz continuous with constant $L_t$ provided in the proof (i.e. \Cref{assumption:grad_log_lipschitzness_convexity_outside_of_a_ball} $\Rightarrow$ \Cref{assumption:lipschitz_score_across_convolutional_diffusion} and \Cref{assumption:grad_log_lipschitzness_hessian_decay} $\Rightarrow$ \Cref{assumption:lipschitz_score_across_convolutional_diffusion}).
\end{lemma}
An important element in the proof, given in Appendix \ref{proof:lem:smoothness_of_gaussian_path}, is the generalisation of the Poincaré inequality for vector-valued functions, which is presented in Lemma \ref{lemma:PI_for_vector_valued_functions}. 
Notably, these bounds improve those in \citet[Proposition 20]{gao2024gaussian} under the specified conditions.

It is important to emphasise that a significant number of works in the diffusion models literature, e.g. \citet{lee2022convergence, Chen2022ImprovedAO}; \citet{chen2023sampling}, assume that $\nabla\log\mu_t$ is Lipschitz for all $t$, with the Lipschitz constant bounded over time. 
In contrast, we have demonstrated that this condition arises naturally under assumptions \Cref{assumption:grad_log_lipschitzness_convexity_outside_of_a_ball} or \Cref{assumption:grad_log_lipschitzness_hessian_decay} on the target distribution.

\paragraph{Action of $(\mu_t)_{t}$.}  
To derive a bound on the action necessary for the convergence analysis, we make the following assumption on the schedule. 
\begin{assumption}\label{assumption:schedule_form}
    Let $\lambda_t:\mathbb{R}^+\to[0,1]$ be non-decreasing in $t$ and weakly differentiable, such that there exists a constant $C_\lambda$ satisfying either of the following conditions
    \begin{equation*}
        \max_{t\in[0,T]}\vert \partial_t{\log \lambda_t}\vert \leq C_\lambda
    \end{equation*}
    or
  \begin{equation*}
        \max_{t\in[0,T]}\left\vert \frac{\partial_t{\lambda_t}}{\sqrt{\lambda_t(1-\lambda_t)}}\right\vert \leq C_\lambda.
    \end{equation*}
\end{assumption}
Notably schedules of the form $\lambda_t = 0.5(1+\cos(\pi(1-(t/T)^\phi)))$, $0.5(1+\tanh(\phi(t/T-0.5)))$ with $\phi\in\mathbb{R}^+$, sigmoid-type schedules, or the schedule corresponding to the \gls*{OU} process $e^{-2(T-t)}$, among others, satisfy the previous assumption. 
When $\lambda_0=0$, the first condition in \Cref{assumption:schedule_form} requires that the derivative of the schedule at $t=0$ is close to zero, meaning that the schedule grows very slowly at the beginning. 
This intuitively captures the importance of the initial stages in the Langevin diffusion generation process. For instance, when the data distribution consists of two distant modes, the diffusion needs to allocate the correct proportion of mass to each mode. During the early stages, as the mass separates towards each mode, employing a slower-increasing schedule can aid in this process. As the mass approaches each mode, the probability of it jumping between modes decreases rapidly, making a slow initial increase essential for effective separation.
The second condition ensures that the schedule also becomes flat as it approaches $\lambda_T=1$. 
This promotes a more refined and detailed generation process, enabling the model to converge more precisely to the data distribution.

Under assumption \Cref{assumption:schedule_form} on the schedule, we derive the following bound on the action of $\mu=(\mu_t)_{t\in[0, T]}$.
\begin{lemma}\label{lemma:action_bound} 
If $\pid$ and $\lambda_t$ satisfy assumptions \Cref{assumption:finite_second_order_moment} and \Cref{assumption:schedule_form}, respectively, 
the action for the Gaussian diffusion path $\mathcal{A}_\lambda(\mu)$ can be upper bounded by
\begin{equation*}
    \mathcal{A}_\lambda(\mu)\lesssim C_\lambda \left(\mathbb{E}_{\pid}\left[\Vert X\Vert^2\right] +  d\right)\lesssim M_{2} \vee d.
\end{equation*}
\end{lemma}
The proof is given in Appendix \ref{proof:lemma:bound_action}. It is worth highlighting that unlike for the geometric path \citep{guo2024provablebenefitannealedlangevin}, for the diffusion path we get an explicit bound on the action under a mild assumption on the schedule. Furthermore, we observe in the proof that selecting the mean and variance of the base distribution $\nu$ close to that of the target results in a tighter bound for the action.

\subsection{Analysis of the Gaussian \gls*{DALMC} Algorithm}
 We now analyse the convergence of the \gls*{DALMC} algorithm \eqref{eq:annealed_langevin_mcmc_algorithm_score_approx} for a Gaussian base distribution.

\begin{theorem}\label{theorem:discretisation_analysis_convolutional_path}
Under \Cref{assumption:finite_second_order_moment}, \Cref{assumption:lipschitz_score_across_convolutional_diffusion} and \Cref{assumption:schedule_form}, the \gls*{DALMC} algorithm \eqref{eq:annealed_langevin_mcmc_algorithm_score_approx} initialised at $X_0\sim\hat{\mu}_0$ and with an approximate score which satisfies \Cref{assumption:score_approximation}, yields the following bound  
\begin{align*}
    \kl\left(\mathbb{P}\;||\mathbb{Q}_\theta\right)
    \lesssim &\left(1+\frac{L_{\max}^2}{M^2\kappa^4}\right) \kappa  \left(\mathbb{E}_{\pid}\left[\Vert X\Vert^2\right] +  d\right)\\
    &+ \frac{d}{M\kappa^2}\left(1+  \frac{L_{\max}}{M\kappa}\right)\int_{0}^{T}L_{ t}^2 \ \md t +  \varepsilon_{\score}^2,
\end{align*}
where $\mathbb{Q}_\theta= (q_{\theta, \lambda_t})_{t\in[0,T]}$ is the path measure of the continuous-time interpolation of  \eqref{eq:annealed_langevin_mcmc_algorithm_score_approx}, $\mathbb{P}$ is that of a reference \gls*{SDE} such that the marginals at each time t have distribution $\hat{\mu}_t$, $M$ denotes the number of steps, $T/\kappa = \sum_{l=1}^M h_l$ and $L_t$ is the Lipschitz constant of $\nabla\log \mu_t$, $L_{\max} = \max_{[0,T]} L_t$.
\end{theorem}
The proof, included in Appendix \ref{proof:theorem:discretisation_path_analysis}, mainly relies on an application of Girsanov's theorem, the bound on the action and the Lipschitzness of $\nabla\log\mu_t$. 
Additionally, we note in the proof that, smaller step sizes $h_l$ are preferred when the Lipschitz constant $L_t$ is larger to obtain a tighter bound.

This result allows us to provide a bound on the iteration complexity of the \gls*{DALMC} algorithm.
\begin{corollary}\label{corollary:complexity_bounds}
For $T\geq 1$, $\kappa<1$ and $M$, there always exists a sequence of step sizes $h_k= T_{k}-T_{k-1}$ such that $\sum_{k=1}^{M} h_k = T/\kappa$. Then, if we take $\kappa = \mathcal{O}\left(\frac{\varepsilon_{\score}^2}{M_2 \vee d}\right)$ and $M = \mathcal{O}\left(\frac{d (M_2 \vee d)^2L_{\max}^2}{\varepsilon_{\score}^6}\right)$, we have $\kl\left(\mathbb{P}\;||\mathbb{Q}_\theta\right) = \mathcal{O}(\varepsilon_{\emph{\text{score}}}^2)$. Hence, for any $\varepsilon = \mathcal{O}(\varepsilon_{\score})$, and under assumptions \Cref{assumption:finite_second_order_moment}, \Cref{assumption:lipschitz_score_across_convolutional_diffusion} and \Cref{assumption:schedule_form}, the \gls*{DALMC} algorithm \eqref{eq:annealed_langevin_mcmc_algorithm_score_approx} initialised at $X_0\sim\hat{\mu}_0$  requires at most $\mathcal{O}\left(\frac{d (M_2 \vee d)^2L_{\max}^2}{\varepsilon^6}\right)$ steps to approximate $\pid$ to within $\varepsilon^2$ \gls*{KL} divergence, that is $\kl(\pid\Vert q_{\theta, \lambda_T})\leq \varepsilon^2$, assuming a sufficiently accurate score estimator.
\end{corollary} 

It is important to note that this bounds are less favourable than those of diffusion models \citep{Chen2022ImprovedAO, chen2023sampling, benton2024nearly}, which explains the success of these models compared to diffusion annealed Langevin-based algorithms. This difference mainly arises because the Langevin \gls*{SDE} implementation \eqref{eq:annealed_langevin_sde} introduces an implicit bias, whereas the reverse \gls*{SDE} in diffusion models ensures that the law of the solution of the \gls*{SDE} exactly matches the intermediate marginal distributions.
Additionally, the use of an exponential integrator scheme in diffusion models, benefiting from the linear term in the drift of the reverse \gls*{SDE}, contrasts with the Euler-Maruyama discretisation used here, leading to an improvement in the discretisation error. 

\subsection{Analysis under Relaxed Assumptions}
In this section, we introduce a less restrictive assumption for the data distribution that generalises \Cref{assumption:grad_log_lipschitzness_convexity_outside_of_a_ball} and \Cref{assumption:grad_log_lipschitzness_hessian_decay}. Under this assumption, we derive an error bound for the \gls*{DALMC} algorithm without relying on the smoothness of $\log\mu_t$ along the diffusion path, in contrast to the proof of Theorem~\ref{theorem:discretisation_analysis_convolutional_path}.

\begin{assumption}\label{assumption:pi_data_conditions_less_restrictive}
    The data distribution $\pid$ has density with respect to Lebesgue, which we write $\pid\propto e^{-V_\pi}$, and a finite second order moment. The potential $V_\pi$ has Lipschitz continuous gradient, with Lipschitz constant $L_\pi$, and 
    \begin{equation*}
        \mathbb{E}_{\pid} \left\Vert\nabla V_\pi\left(X\right)\right\Vert^{8} \leq K_\pi^2.
    \end{equation*} 
\end{assumption}
In Appendix~\ref{appendix:comments_new_less_restrictiv_assumption}, we show that both \Cref{assumption:grad_log_lipschitzness_convexity_outside_of_a_ball} and \Cref{assumption:grad_log_lipschitzness_hessian_decay} (with finite second-order moment) imply assumption \Cref{assumption:pi_data_conditions_less_restrictive}. 
Besides, since under \Cref{assumption:pi_data_conditions_less_restrictive} the data distribution has a finite second-order moment, if the schedule also satisfies  \Cref{assumption:schedule_form}, then the bound on the action established in Lemma~\ref{lemma:action_bound} remains valid. 
This enables us to obtain the following complexity guarantees for the \gls*{DALMC} algorithm under this new assumption.
\begin{theorem}\label{theorem:convergence_relaxed_assumption}
    Under \Cref{assumption:schedule_form} and \Cref{assumption:pi_data_conditions_less_restrictive}, the \gls*{DALMC} algorithm \eqref{eq:annealed_langevin_mcmc_algorithm_score_approx} initialised at $X_0\sim\hat{\mu}_0$  and with an approximate score which satisfies \Cref{assumption:score_approximation} leads to 
\begin{align*}
    \kl\left(\mathbb{P}\;||\mathbb{Q}_\theta\right)\lesssim&\frac{dL_\pi}{M^2\kappa^3} + \frac{(d^2\vee L_\pi^2d \vee K_\pi)}{M\kappa^2} \\
    &+ \kappa (\mathbb{E}_{\pid}[\Vert X\Vert^2] + d) + \varepsilon_{\text{score}}^2,
\end{align*}
where $\mathbb{Q}_\theta= (q_{\theta, \lambda_t})_{t\in[0,T]}$ is the path measure corresponding to the continuous-time interpolation of algorithm \eqref{eq:annealed_langevin_mcmc_algorithm_score_approx}, $\mathbb{P}$ is that of a reference \gls*{SDE} such that the marginals at each time t have distribution $\hat{\mu}_t$ and $M$ denotes the number of steps. Therefore, under assumptions  \Cref{assumption:schedule_form} and \Cref{assumption:pi_data_conditions_less_restrictive}, the \gls*{DALMC} algorithm \eqref{eq:annealed_langevin_mcmc_algorithm_score_approx} initialised at $X_0\sim\hat{\mu}_0$ with approximate scores requires at most $\mathcal{O}\left(\frac{(M_2\vee d)^2(d^2\vee L_\pi^2 d\vee K_\pi)L_\pi}{\varepsilon^6}\right)$ steps to approximate $\pid$ to within $\varepsilon^2$ $\kl$ divergence, that is $\kl(\pid\Vert q_{\theta, \lambda_T})\leq \varepsilon^2$, assuming a sufficiently accurate score estimator, i.e. $\varepsilon_{\score} = \mathcal{O}(\varepsilon)$. If $M_2= \mathcal{O}(d)$, $L_\pi= \mathcal{O}(\sqrt{d})$ and $K_\pi= \mathcal{O}(d^2)$, then $M = \mathcal{O}\left(\frac{d^4L_\pi}{\varepsilon^6}\right)$.
\end{theorem}
See Appendix~\ref{appendix:convergence_relaxed_assumption} for the proof.
Relaxing the assumptions results in a dimensional dependence on the number of steps that is one order worse compared to Corollary \ref{corollary:complexity_bounds}.

\section{Heavy-Tailed Diffusion Paths} \label{sec:heavy_tailed_diffusion}
We now analyse the annealed Langevin diffusion path \eqref{eq:annealed_langevin_sde} when the base  distribution $\nu\in\mathcal{P}(\mathbb{R}^d)$ is a Student's $t$-distribution, $\nu\sim t(0, \sigma^2I, \alpha)$, with tail index $\alpha>2$
\begin{equation*}
    \nu(x) \propto \left( 1 + \frac{\Vert x\Vert^2}{\alpha\sigma^2}\right)^{-( \alpha + d)/2}.
\end{equation*}
It is worth noting that the $t$-distribution is not a stable distribution, unlike the Gaussian family, meaning that the convolution of two $t$-distributions is not necessarily a $t$-distribution. \citet{NADARAJAH2005715} provides explicit expressions for the density function of the convolution of one dimensional $t$-distributions with unit variance, but only when both degrees of freedom are odd. Closed-form expressions cannot be derived in one dimension when either of the two degrees of freedom  is even. In the $d$-dimensional case, only closed forms can be derived when $\alpha + d$ is even.

\subsection{Analysis of the Heavy-Tailed Diffusion Path}
\paragraph{Smoothness of $(\mu_t)_t$.} We require for the discretisation analysis that the intermediate distributions of the heavy-tailed diffusion path satisfy smoothness conditions given in \Cref{assumption:lipschitz_score_across_convolutional_diffusion}. 
We show below that this assumption holds when the data distribution $\pid$ satisfies the following conditions.
\begin{assumption}\label{assumption:heavy_tailed_target_condition}
The data distribution $\pid$ has density with respect to the Lebesgue measure. $\nabla\log\pid$ is Lipschitz continuous with constant $L_\pi$ and $\Vert\nabla\log\pid\Vert^2\leq C_\pi$ almost surely.
\end{assumption}

In particular, Lemma~\ref{lem:regularity_of_heavy_tailed_path} below demonstrates that this assumption holds when the data distribution $\pid$ can be expressed as the convolution of a compactly supported measure and a $t$-distribution. 
\begin{assumption}\label{assumption:compact_plus_t_distribution_target}
     Let $ X$ be a $d$-dimensional random vector $X \sim\pid$, such that $X = U + G$, where $\Vert U - m_{\pi}\Vert^2 \leq  d R^2$ holds almost surely, and $G \sim t(0, \tau^2 I, \Tilde{\alpha})$ is independent of $U$. 
\end{assumption}
Lemma~\ref{lemma:wegihted_pi} in the Appendix shows that if $\pid$ satisfies assumption \Cref{assumption:compact_plus_t_distribution_target}, then it has a finite weighted Poincaré constant.
This extends the result of \citet{functional_inequalities_compactly_supported_assumption} to convolutions of compactly supported measures with $t$-distributions.
However, unlike the multivariate Gaussian case, our bound on the weighted Poincaré constant is not dimension-free.

The following result shows that \Cref{assumption:heavy_tailed_target_condition} $\Rightarrow$ \Cref{assumption:lipschitz_score_across_convolutional_diffusion} and \Cref{assumption:compact_plus_t_distribution_target} $\Rightarrow$ \Cref{assumption:heavy_tailed_target_condition}.
\begin{lemma}\label{lem:regularity_of_heavy_tailed_path} 
Under \Cref{assumption:heavy_tailed_target_condition} and taking the base distribution $\nu\sim t(0, \sigma^2 I, \alpha)$, we have that for all $t\in[0,T]$ $\nabla \log \mu_t(x)$ is Lipschitz (\Cref{assumption:lipschitz_score_across_convolutional_diffusion}) with constant $L_t$ provided in the proof. Besides, \Cref{assumption:heavy_tailed_target_condition} holds when $\pid$ is a convolution of a compactly supported measure and a multivariate $t$ distribution (\Cref{assumption:compact_plus_t_distribution_target}).
\end{lemma}
The proof is provided in Appendix~\ref{proof:lem:smoothness_of_heavy_tail_path}.
\paragraph{Action of $(\mu_t)_{t}$.} To derive a bound on the action, necessary for the discretisation analysis, we introduce an assumption on the schedule similar to that of \Cref{assumption:schedule_form}.
\begin{assumption}\label{assumption:schedule_form_heavy_tail_diffusion}
    Let $\lambda_t:\mathbb{R}^+\to[0,1]$ be non-decreasing in $t$ and weakly differentiable, such that there exists a constant $C_\lambda$ satisfying
    \begin{equation*}
        \max_{t\in[0,T]}\left\vert \frac{\partial_t{\lambda_t}}{\sqrt{\lambda_t(1-\lambda_t)}}\right\vert \leq C_\lambda.
    \end{equation*}
\end{assumption}
Intuitively, schedules with derivatives close to $0$ as $t$ approaches $0$ and $T$ satisfy the previous assumption.
In particular, schedules of the form $\lambda_t = 0.5(1+\cos(\pi(1-(t/T)^\phi)))$, $0.5 ( 1 + \tanh(\phi(t-0.5)))$, where $\phi\in\mathbb{R}^+$, fulfil \Cref{assumption:schedule_form_heavy_tail_diffusion}. Under the previous assumption, we compute the following bound on the action. The proof is provided in Appendix \ref{proof:lem:action_bound_heavy_tail_diffusion}. 

\begin{lemma}\label{lemma:action_bound_heavy_tail_diffusion}  
If $\pid$ has finite second-order moment (\Cref{assumption:finite_second_order_moment}), $\nu\sim t(0, \sigma^2 I, \alpha)$ with tail index $\alpha>2$ and $\lambda_t$ satisfies \Cref{assumption:schedule_form_heavy_tail_diffusion}, 
the action for the heavy-tailed diffusion path $\mathcal{A}_\lambda(\mu)$ can be effectively upper bounded as follows
\begin{equation*}
    \mathcal{A}_\lambda(\mu)\leq \frac{C_\lambda\pi}{8} \left(\mathbb{E}_{\pid}\left[\Vert X\Vert^2\right] + \frac{\sigma^2 d\alpha}{\alpha-2}\right).
\end{equation*}
\end{lemma}

\subsection{Analysis of the Heavy-Tailed \gls*{DALMC} Algorithm} 
The following theorem establishes a bound for the discretisation error of the heavy-tailed \gls*{DALMC} algorithm with an approximated score. The proof is given in Appendix \ref{proof:theorem:discretisation_path_analysis_heavy_tail_diffusion}. 
\begin{theorem}\label{theorem:discretisation_analysis_convolutional_heavy_tailed_diffusion}
Assume the data distribution $\pid$ satisfies assumption  \Cref{assumption:finite_second_order_moment} (finite second-order moment) and \Cref{assumption:lipschitz_score_across_convolutional_diffusion} (which holds under \Cref{assumption:heavy_tailed_target_condition}),
$\nu\sim t(0, \sigma^2 I, \alpha)$ with $\alpha>2$ and
let the schedule satisfy \Cref{assumption:schedule_form_heavy_tail_diffusion}, with $\lambda_{\kappa t}/\lambda_{\kappa (t + \delta)} = \mathcal{O}(1+ \delta)$, $\delta<<1$. 
Then, the heavy-tailed \gls*{DALMC} algorithm with an approximated score satisfying \Cref{assumption:score_approximation} and initialised at $X_0\sim\hat{\mu}_0$, guarantees that 
\small
\begin{align*}
    \kl&\left(\mathbb{P}||\mathbb{Q}_\theta\right)
    \lesssim \left(1+\frac{L_{\max}^2}{M^2\kappa^2}+\frac{1}{M^2\kappa^4}\right) \kappa \left(\mathbb{E}_{\pid}\left[\Vert X\Vert^2\right] + d\right) \\
    &+ \frac{d}{M\kappa^2}\left(1 + \frac{\alpha}{\alpha -2}+  \frac{L_{\max}}{M\kappa}\right)\int_{0}^{T}L_{ t}^2 \ \md t + \varepsilon_{\score}^2,
\end{align*}
\normalsize
where $\mathbb{Q}_\theta = (q_{\theta, \lambda_t})_{t\in[0,T]}$ is the path measure of the continuous-time interpolation of \eqref{eq:annealed_langevin_mcmc_algorithm_score_approx}, $\mathbb{P}$ is that of a reference \gls*{SDE} such that the marginals at each time t have distribution $\hat{\mu}_t$, $M$ denotes the number of steps, $T/\kappa = \sum_{l=1}^M h_l$ and $L_t$ is the Lipschitz constant of $\nabla\log \mu_t$, $L_{\max} = \max_{[0,T]} L_t$. Therefore, under  \Cref{assumption:finite_second_order_moment}, \Cref{assumption:lipschitz_score_across_convolutional_diffusion} and \Cref{assumption:schedule_form_heavy_tail_diffusion}, the heavy-tailed \gls*{DALMC} algorithm \eqref{eq:annealed_langevin_mcmc_algorithm_score_approx} initialised at $X_0\sim\hat{\mu}_0$ requires at most $\mathcal{O}\left(\frac{d (M_2 \vee d)^2 L_{\max}^2}{\varepsilon^6}\right)$ steps to approximate $\pid$ to within $\varepsilon^2$ \gls*{KL} divergence, that is $\kl(\pid\Vert q_{\theta, \lambda_T})\leq \varepsilon^2$, assuming a sufficiently accurate score estimator, i.e. $\varepsilon_{\score} = \mathcal{O}(\varepsilon)$.
\end{theorem}
Note that as the tail index $\alpha$ tends to $\infty$, which corresponds to $\nu$ approaching a Gaussian distribution, the bound on $\kl\left(\mathbb{P} || \mathbb{Q}_\theta\right)$ recovers that of Theorem \ref{theorem:discretisation_analysis_convolutional_path}. Furthermore, since $\alpha/(\alpha-2)\leq 3$ for $\alpha>2$, the iteration complexity of the heavy-tailed \gls*{DALMC} algorithm is identical to that of the Gaussian \gls*{DALMC} algorithm corresponding to the Gaussian diffusion path.

\section{Related Work}\label{sec:related_work}

\paragraph{Score-based generative models.
}
Our approach is similar to earlier generative modelling techniques based on annealed Langevin dynamics \citep{song2019generative}, which inspired the advancement of diffusion models.
The existing literature analysing these Langevin Monte Carlo algorithms is limited.
\citet{block2022generativemodelingdenoisingautoencoders} derive an error bound in Wasserstein distance that scales exponentially with the data dimension, while \citet{lee2022convergence} establish a non-asymptotic bound in total variation, which is weaker than our bound in \gls*{KL}, as implied by Pinsker's inequality.

On the other hand, the convergence of diffusion models \citep{song2020score, ho2020denoising} has been extensively studied. 
Early results either established non quantitative bounds \citep{pidstrigach2022scorebased}, relied on restrictive assumptions about the data distribution, such as functional inequalities \citep{lee2022convergence}, or exhibited exponential dependence on the problem parameters \citep{ de_bortoli2022convergence}. 
Recent works have established polynomial convergence bounds under more relaxed assumptions \citep{Chen2022ImprovedAO, chen2023sampling, benton2024nearly, li2024towards}. In particular, \citet{Chen2022ImprovedAO} introduce two bounds on the $\kl$ error: a linear bound in the data dimension under smoothness conditions along the entire diffusion path, and a second one which scales quadratically with $d$, achieved through early stopping and the assumption of a finite second-order moment on $\pid$. In contrast, \citet{benton2024nearly} provide a bound that is linear in the data dimension, up to logarithmic factors, assuming only that the data distribution has a finite second-order moment. Their proof exploits the specific structure of the \gls*{OU} process to control the error arising from discretising the reverse \gls*{SDE}.

\paragraph{Stochastic interpolants.}

Stochastic interpolants \citep{albergo2023stochastic} are generative models that unify flow-based and diffusion-based methods. These models make use of a broad class of continuous-time stochastic processes designed to bridge any two arbitrary probability density functions exactly in finite time, akin to our work. 
Specifically, the formulation of linear one-sided stochastic interpolants \citep{albergo2023stochastic, gao2024gaussian}, which interpolate between a Gaussian and the data distribution, is equivalent to the Gaussian diffusion path \eqref{eq:convolutional_path}. 
Unlike our approach, they incorporate intractable control terms into the drift of the \gls*{SDE} to ensure the marginals have the desired distributions. This may result in  numerical instabilities caused by singularities in the drift at $t= T$ \citep[Section 6]{albergo2023stochastic}. In contrast, we implement the diffusion path using Langevin dynamics. Furthermore, their theoretical analysis does not include explicit non-asymptotic convergence bounds.

\paragraph{Tempering.}
Tempering \citep{PhysRevLett.57.2607, geyer_mcmc_92, Marinari_1992} is a well-known technique in the sampling literature that involves sampling the system at multiple temperatures: starting with higher temperatures to facilitate transitions between modes, gradually cooling the system to focus on the local structure of the target distribution. 
The sequence of tempered target distributions is typically defined using the geometric path, as it can be computed in closed form when the target density is known up to a normalising constant.
Recently, several works have established theoretical guarantees for the convergence of geometric annealed Langevin Monte Carlo for non-log-concave distributions. In particular, \citet{guo2024provablebenefitannealedlangevin} provides a bound on the $\kl$  similar to that of Theorem~\ref{theorem:discretisation_analysis_convolutional_path}. However, they are unable to obtain a closed-form expression for the action of the path. 
Besides, \citet{chehab2024provableconvergencelimitationsgeometric} derive upper and lower convergence bounds for the $\kl$ of the marginals, based on functional inequalities assumptions. 
In particular, they demonstrate that in some cases the log-Sobolev constant of the intermediate distributions along the path can deteriorate exponentially compared to those of the base and data distributions, unlike for the diffusion path.

\section{Conclusions} 
\label{sec:conclusions}

In this work we provided a rigorous non-asymptotic analysis of Diffusion Annealed Langevin Monte Carlo (\gls*{DALMC}) for generative modelling, focusing on both Gaussian and heavy-tailed diffusion paths. By examining general diffusion paths that interpolate between complex data distributions and simpler base distributions, we have obtained theoretical insights into the convergence behaviour of \gls*{DALMC} under a range of assumptions.
For Gaussian diffusion paths, we derived explicit non-asymptotic path-wise error bounds in KL divergence, improving upon prior results by relaxing smoothness assumptions and addressing the bias introduced through discretisation. Extending the framework to heavy-tailed diffusion paths, such as those based on Student’s $t$-distributions, we presented the first theoretical guarantees for these models, demonstrating comparable complexity to Gaussian diffusion paths under mild conditions.

Our analysis highlighted how smoothness assumptions, such as Lipschitz continuity of the scores and properties of the data distribution (e.g., bounded second-order moment and convexity or heavy-tailed behaviour), naturally ensure bounded action and efficient convergence. This generalisation broadens the applicability of \gls*{DALMC} beyond the  settings considered in prior work. While \gls*{DALMC} introduces some bias compared to reverse \gls*{SDE} implementations, it avoids numerical instabilities and provides a simpler approach, making it a compelling alternative for score-based generative modelling, in certain settings.

Looking ahead, further work could focus on developing more efficient numerical schemes, reducing dimensional dependencies in error bounds, and applying this framework to other generative models.

\newpage
\section*{Acknowledgments}
PCE would like to thank Arnaud Guillin, Paul Felix Valsecchi Oliva, Pierre Monmarché and Yanbo Tang for their insightful discussions.
PCE is supported by EPSRC through the Modern Statistics and Statistical Machine Learning (StatML) CDT programme, grant no. EP/S023151/1.

\bibliography{refs}
\bibliographystyle{plainnat}

% APPENDIX

\newpage
\appendix
\onecolumn
\section{Background}\label{appendix:background}
We introduce some concepts from optimal transport and the Girsanov theorem which will be useful for the subsequent analysis.

\paragraph{Optimal transport.}
Let $v=(v_t:\mathbb{R}^d\to\mathbb{R}^d)$ be a vector field and $\mu=(\mu_t)_{t\in[a,b]}$ be a curve of probability measures on $\mathbb{R}^d$ with finite second-order moments. $\mu$ is generated by the vector field $v$ if the continuity equation
\begin{equation*}
    \partial_t\mu_t + \nabla\cdot(\mu_tv_t)=0,
\end{equation*}
holds for all $t\in[a,b]$. The metric derivative of $\mu$ at $t\in[a,b]$ is then defined as
\begin{equation*}
    \left\vert\Dot{\mu}\right\vert_t:=\lim_{\delta\to0}\frac{W_2(\mu_{t+\delta}, \mu_t)}{\left\vert\delta\right\vert}.
\end{equation*}
If $\left\vert\Dot{\mu}\right\vert_t$ exists and is finite for all $t\in[a,b]$, we say that $\mu$ is an absolutely continuous curve of probability measures. 
\citet{ambrosio2000rectifiable} establish weak conditions under which a curve of probability measures with finite second-order moments is absolutely continuous. 

By \citet[Theorem 8.3.1]{gradients_flows_book} we have that among all velocity fields $v_t$ which produce the same flow $\mu$, there is a unique optimal one with smallest $L^p(\mu_t; X)$-norm. This is summarised in the following lemma.
\begin{lemma}[Lemma 2 from \citet{guo2024provablebenefitannealedlangevin}]\label{lemma:optimal_vector_field}
    For an absolutely continuous curve of probability measures $\mu =(\mu_t)_{t\in[a,b]}$, any vector field $(v_t)_{t\in[a,b]}$ that generates $\mu$ satisfies $\left\vert\Dot{\mu}\right\vert_t\leq\Vert v_t\Vert_{L^2(\mu_t)}$ for almost every $t\in[a, b]$. Moreover, there exists a unique vector field $v_t^\star$ generating $\mu$ such that $\left\vert\Dot{\mu}\right\vert_t = \Vert v_t^\star\Vert_{L^2(\mu_t)}$ almost everywhere.
\end{lemma}
We also introduce the action of the absolutely continuous curve $(\mu_t)_{t\in[a,b]}$ since it will play a key role in our convergence results. In particular, we define the action $\mathcal{A}(\mu)$ as
\begin{equation*}
  \mathcal{A}(\mu):=\int_a^b\left\vert\Dot{\mu}\right\vert_t^2  \md t.  
\end{equation*}

\paragraph{Girsanov's theorem.} Consider the \gls*{SDE}
\begin{equation*}
    \md X_t = b(X_t, t)\md t + \sigma(X_t,t)\md B_t,
\end{equation*}
for $t\in[0,T]$, where $(B_t)_{t\in[0,T]}$ is a standard Brownian motion in $\mathbb{R}^d$. Denote by $\mathbb{P}^X$ the \emph{path measure} of the solution $X = (X_t)_{t\in[0,T]}$ of the \gls*{SDE}, which characterises the distribution of $X$ over the sample space $\Omega$. 

The $\kl$ divergence between two path measures can be characterised as a consequence of Girsanov's theorem \citep{karatzas}. In particular, the following result will be central in our analysis.
\begin{lemma}\label{lemma:girsanov_theorem}
    Consider the following two \glspl*{SDE} defined on a common probability space $(\Omega, \mathcal{F}, \mathbb{P})$
    \begin{equation*}
        \md X_t = a_t(X)\md t+ \sqrt{2}\md B_t,\quad\quad \md Y_t = b_t(Y)\md t+ \sqrt{2}\md B_t, \quad\quad t\in[0, T]
    \end{equation*}
with the same initial conditions $X_0, Y_0\sim \mu_0$. Denote by $\mathbb{P}^X$ and $\mathbb{P}^Y$ the path measures of the processes $X$ and $Y$, respectively. It follows that 
\begin{equation*}
    \kl (\mathbb{P}^X\Vert\mathbb{P}^Y) = \frac{1}{4}\mathbb{E}_{X\sim \mathbb{P}^X}\left[\int_0^T\Vert a_t(X)-b_t(X)\Vert^2\md t\right].
\end{equation*}
\end{lemma}

\paragraph{Preliminary results.}
\citet[Theorem 1]{guo2024provablebenefitannealedlangevin}  provide convergence guarantees for the continuous-time geometric annealed Langevin dynamics based on the action of the curve of probability measures given by the geometric mean of the base and target distributions.
Their result can be adapted to our setting as follows.
\begin{theorem}[Theorem 1 \citep{guo2024provablebenefitannealedlangevin}]\label{theorem:preliminaries_continuous_time_kl_bound}
    Let \emph{$\mathbb{P}_{\text{DALD}} = (p_{t,\text{DALD}})_{t\in[0, T/\kappa]}$} be  the path measure of the diffusion annealed Langevin dynamics \eqref{eq:annealed_langevin_sde}, and $\mathbb{P}=(\hat{\mu}_{t})_{t\in[0, T/\kappa]}$ that of a reference \gls*{SDE} such that the marginals at each time have distribution $\hat{\mu}_t$. If \emph{$ p_{0, \text{DALD}}= p_0$}, the \gls*{KL} divergence between the path measures is upper bounded by
    \emph{\begin{equation*}     \kl(\mathbb{P}\Vert\mathbb{P}_{\text{DALD}})\leq \kappa \mathcal{A}(\mu).
    \end{equation*}}
\end{theorem}
\begin{proof}
Let $\mathbb{P}$ be the path measure corresponding to the following reference SDE
\begin{equation*}
    \md Y_t = (\nabla\log \hat{\mu}_t + v_t)(Y_t)\md t + \sqrt{2}\md B_t,\; t\in[0,T/\kappa].
\end{equation*}  
The vector field $v = (v_t)_{t\in[0,T/\kappa]}$ is designed such that $Y_t\sim\hat{\mu}_t$ for all $t\in[0, T/\kappa]$. 
Using the Fokker-Planck equation, we have that 
\begin{equation*}
    \partial\hat{\mu}_t = \nabla\cdot\left(\hat{\mu}_t(\nabla\log\hat{\mu}_t + v_t)\right)  + \Delta \hat{\mu}_t = -\nabla\cdot(\hat{\mu}_t v_t), \; t\in[0,T/\kappa].
\end{equation*}
This implies that $v_t$ satisfies the continuity equation and hence generates the curve of probability measures $(\hat{\mu}_t)_t$. Leveraging Lemma \ref{lemma:optimal_vector_field}, we choose $v$ to be the one that minimises the $L^2(\hat{\mu}_t)$ norm, resulting in $\Vert v_t\Vert_{L^2(\hat{\mu}_t)} = \left\vert\Dot{\hat{\mu}}\right\vert_t$ being the metric derivative.
Using the form of Girsanov's theorem given in Lemma \ref{lemma:girsanov_theorem} we have
\begin{align*}
    \kl\left(\mathbb{P}\;||\mathbb{P}_{\text{DALD}}\right) &= \frac{1}{4}\mathbb{E}_{\mathbb{P}}\left[\int_0^{T/\kappa}\left\Vert  v_t(X_t)\right\Vert^2\md t\right] = \frac{1}{4}\int_0^{T/\kappa}\left\Vert  v_t(X_t)\right\Vert_{L^2(\hat{\mu})}^2\md t = \frac{1}{4}\int_0^{T/\kappa}\left\vert\Dot{\hat{\mu}}\right\vert_t^2\md t \\
    & = \frac{\kappa}{4}\int_0^{T}\left\vert\Dot{{\mu}}\right\vert_t^2\md t = \frac{\kappa \mathcal{A}(\mu)}{4},
\end{align*}
where we have used that $\left\vert\Dot{\hat{\mu}}\right\vert_t = \kappa \left\vert\Dot{\mu}\right\vert_t $ and the change of variable formula. 
\end{proof}

\section{Proofs of Section~\ref{sec:gaussian_diffusion}}

\subsection{Comments on Assumption \Cref{assumption:grad_log_lipschitzness_convexity_outside_of_a_ball}}\label{appendix:comments_assumption}
\input{appendix/assumption_1}

\subsection{Comments on Assumption \Cref{assumption:grad_log_lipschitzness_hessian_decay}}\label{appendix:comments_assumption_hessian_decay}
\input{appendix/assumption_hessian_decay}

\subsection{Proof of Lemma~\ref{lem:regularity_of_gaussian_path}}\label{proof:lem:smoothness_of_gaussian_path}
\input{appendix/vector_pi}
\input{appendix/proof_smoothness}

\input{appendix/proof_smoothness_heavy_tail_target}

\subsection{Proof of Lemma~\ref{lemma:action_bound}}
\input{appendix/action_analysis}\label{proof:lemma:bound_action}

\subsection{Proof of Theorem~\ref{theorem:discretisation_analysis_convolutional_path}}\label{proof:theorem:discretisation_path_analysis}
\input{appendix/theorem_discretisation_reparametrised}

\subsection{Proof of Corollary~ \ref{corollary:complexity_bounds}}\label{proof:corollaries:discretisation_approx_score_and_complexity}
\input{appendix/corollary_discretisation_score_approx}

\subsection{Comment on Assumption~\Cref{assumption:pi_data_conditions_less_restrictive}}\label{appendix:comments_new_less_restrictiv_assumption}
\input{appendix/comments_new_less_restrictiv_assumption}

\subsection{Proof of Theorem~\ref{theorem:convergence_relaxed_assumption}}\label{appendix:convergence_relaxed_assumption}
\input{appendix/additional_proof_less_assumption}

%%%%%%%%%%%%%%
\section{Proofs of Section~\ref{sec:heavy_tailed_diffusion}}

\subsection{Comments on Assumption \Cref{assumption:compact_plus_t_distribution_target}}\label{appendix:comments_assumption_compact_plus_heavy_tail}
\input{appendix/assumption_compact_plus_heavy_tail}

\subsection{Proof of Lemma~\ref{lem:regularity_of_heavy_tailed_path}}\label{proof:lem:smoothness_of_heavy_tail_path}
\input{appendix/proof_smoothness_of_heavy_tail_path_alternative}

\subsection{Proof of Lemma~\ref{lemma:action_bound_heavy_tail_diffusion}}\label{proof:lem:action_bound_heavy_tail_diffusion}
\input{appendix/action_bound_heavy_tail_diffusion}

\subsection{Proof of Theorem~\ref{theorem:discretisation_analysis_convolutional_heavy_tailed_diffusion}
}\label{proof:theorem:discretisation_path_analysis_heavy_tail_diffusion}
\input{appendix/theorem_discretisation_reparametrised_heavy_tailed_diffusion}

\end{document}

%% file: appendix/assumption_1.tex
A typical assumption in the literature \citep{saremi2024chain, grenioux2024stochastic} considers the data distribution is given by the convolution of a compactly supported measure and a Gaussian distribution, which can be formalised as follows.
\begin{assumption}\label{assumption:compact_plus_gaussian_target_1}
     Let $ X$ be a $d$-dimensional random vector $X \sim\pi_{\text{data}}$, such that $X = U + G$, where $\Vert U - m_{\pi}\Vert^2 \leq  d R^2$ holds almost surely and $G \sim \mathcal{N}(0, \tau^2 I)$ is independent of $U$. 
\end{assumption}
We demonstrate that assumption \Cref{assumption:compact_plus_gaussian_target_1} implies that the potential $V_\pi$ has Lipschitz continuous gradients, satisfies the dissipativity condition and, also ensures that assumption \Cref{assumption:lipschitz_score_across_convolutional_diffusion} is satisfied. Furthermore, under additional assumptions on the compactly supported measure in \Cref{assumption:compact_plus_gaussian_target_1}, we show that \Cref{assumption:compact_plus_gaussian_target_1} entails \Cref{assumption:grad_log_lipschitzness_convexity_outside_of_a_ball}.
\begin{lemma}\label{lemma:implications_between_assumptions}
Let $\pid\propto e^{-V_\pi}$,    assumption \Cref{assumption:compact_plus_gaussian_target_1} implies that $V_\pi$ has Lipschitz continuous gradients and satisfies the dissipativity inequality
\begin{equation*}
    \langle\nabla V_\pi(x), x\rangle\geq a_\pi \Vert x\Vert^2 - b_\pi,
\end{equation*}
with constants $a_\pi, b_\pi>0$. Furthermore, $\pid$ has a finite log-Sobolev constant.
\end{lemma}
\begin{proof}
 First we show that if $\pi_{\text{data}}\propto e^{-V_\pi}$ satisfies \Cref{assumption:compact_plus_gaussian_target_1} then $V_\pi$ is gradient Lipschitz. 
 Let $X\sim\pi_{\text{data}}$, $U\sim\Tilde{\pi}$ with compact support and $G\sim\gamma=\mathcal{N}(0, \tau^2I)$ independent of $U$. By assumption \Cref{assumption:compact_plus_gaussian_target_1} we have $X = U + G$ or equivalently $\pi_{\text{data}} = \Tilde{\pi}*\gamma$. 
 Using that if $g$ is a differentiable function then $\nabla(f*g) = f * (\nabla g)$, we have that
\begin{align}
-\nabla\log\pid(x) &= \frac{1}{\tau^2}\left(x-\mathbb{E}_{\rho_x}[Y]\right) \label{eq:covariance_compact_plus_gaussian_score}\\
    -\nabla^2 \log\pi_{\text{data}}(x) &= \frac{1}{\tau^2}\left(I -\frac{1}{\tau^2} \text{Cov}_{{\rho_x}} \left[Y\right]\right), \label{eq:covariance_compact_plus_gaussian_hessian}
\end{align}
where $\rho_x(y) \propto \Tilde{\pi}(y)\gamma(x-y)$ and $\text{Cov}_{\rho_x}[Y] = \mathbb{E}_{\rho_x}[Y Y^\intercal] - \mathbb{E}_{\rho_x}[Y] \mathbb{E}_{\rho_x}[Y]^\intercal$.
Note that $\rho_x(y)$ has bounded support independent of $x$ by \Cref{assumption:compact_plus_gaussian_target_1}. Therefore, the eigenvalues of $\text{Cov}_{{\rho_x}} \left[Y\right]$ can be upper bounded by a constant independent of $x$. 
That is, for any $a\in\mathbb{R}^d$ with $\Vert a\Vert=1$
\begin{align*}
    -a^{\intercal}\left(\nabla^2 \log \pi_{\text{data}}(x)\right) a &= \tau^{-2} -\tau^{-4} a^{\intercal} \text{Cov}_{\rho_x}\left[Y\right]a\geq \tau^{-2}  - \tau^{-4} a^\intercal \mathbb{E}_{\rho_x}[Y Y^\intercal] a\\
    &\geq \tau^{-2} -\tau^{-4} \mathbb{E}_{\rho_x}\left[(a^{\intercal}Y)^2\right] \geq \tau^{-2}  -\tau^{-4}  \mathbb{E}_{\rho_x}[\Vert Y \Vert^2] \\
    &\geq \tau^{-2}-\tau^{-4} (m_\pi + dR^2),
\end{align*}
where we have used Cauchy-Schwarz inequality and \Cref{assumption:compact_plus_gaussian_target_1}.
On the other hand, since the covariance matrix is positive semidefinite, we have
\begin{equation*}
    -\nabla^2 \log\pi_{\text{data}} (x)\preccurlyeq \tau^{-2} I.
\end{equation*}
Therefore, the Hessian $-\nabla^2\log\pi_{\text{data}}$ satisfies
\begin{equation*}
     \left(\tau^{-2}-\tau^{-4} (m_\pi + dR^2)\right) I \preccurlyeq -\nabla^2 \log \pi_{\text{data}}(x) \preccurlyeq \tau^{-2} I,
\end{equation*}
proving that $-\nabla\log\pid$ is gradient Lipschitz with constant $L_\pi\leq \max\left\lbrace \tau^{-2}, \vert\tau^{-2}-\tau^{-4} (m_\pi + dR^2)\vert \right\rbrace$.

On the other hand, using the expression for $\nabla V_\pi = -\nabla\log\pid$ in \eqref{eq:covariance_compact_plus_gaussian_score}, we have 
\begin{align*}
    \langle\nabla V_\pi(x), x\rangle &= \frac{1}{\tau^2}\left\langle x- \mathbb{E}_{\rho_x}[Y], x\right\rangle = \frac{1}{\tau^2}\left(\Vert x\Vert^2-\mathbb{E}_{\rho_x}[\langle Y, x\rangle]\right)\geq \frac{1}{\tau^2}\left(\Vert x\Vert^2-\Vert x\Vert\mathbb{E}_{\rho_x}[\Vert Y\Vert]\right)\\
    &\geq \frac{1}{\tau^2}\left(\Vert x\Vert^2-\Vert x\Vert(m_\pi + \sqrt{d}R)\right) \geq \frac{\Vert x\Vert^2}{2\tau} -(m_\pi + \sqrt{d}R)^2,
\end{align*}
where we have used that $\rho_x(y) \propto \Tilde{\pi}(y)\gamma(x-y)$ has bounded support. This establishes the dissipativity inequality.

Finally, thanks to the dissipativity condition together with Lipschitz gradient, it follows from \citet{cattiaux2010note} that $\pid$ has a finite log-Sobolev constant.
\end{proof}

We now demonstrate the \Cref{assumption:compact_plus_gaussian_target_1} implies \Cref{assumption:lipschitz_score_across_convolutional_diffusion}.
\begin{lemma}
    If $\pid$ satisfies assumption \Cref{assumption:compact_plus_gaussian_target_1}, then the scores $\nabla\log\mu_t$ of the intermediate probability densities of the Gaussian diffusion path are Lipschitz continuous for all $t$, that is, assumption \Cref{assumption:lipschitz_score_across_convolutional_diffusion} is satisfied.
\end{lemma}
\begin{proof}
    Recall that the intermediate random variables of the Gaussian diffusion path, are given by
    \begin{equation*}
        X_t = \sqrt{\lambda_t} X + \sqrt{1-\lambda_t} \sigma Z
    \end{equation*}
    where $X\sim\pid$ and $Z\sim\mathcal{N}(0, I)$ independent of $X$. Using assumption \Cref{assumption:compact_plus_gaussian_target_1}, it follows that 
    \begin{equation*}
        X_t \overset{d}{=} \sqrt{\lambda_t} U + \sqrt{(1-\lambda_t)\sigma^2 + \lambda_t\tau^2} Z,
    \end{equation*}
    where $U\sim\Tilde{\pi}$ is compactly supported and $Z\sim\mathcal{N}(0, I)$ independent of $\Tilde{\pi}$. 
    By applying the result from the previous lemma (Lemma \ref{lemma:implications_between_assumptions}), we conclude that $\nabla\log\mu_t$, where $\mu_t\sim X_t$, is Lipschitz continuous with constant
    \begin{equation*}
        L_t \leq \max\left\lbrace \tau_t^{-1}, \vert\tau_t^{-1}-\tau_t^{-2} (m_\pi + dR^2)\vert \right\rbrace,
    \end{equation*}
    where $\tau_t^2 = (1-\lambda_t)\sigma^2 + \lambda_t\tau^2$. This completes the proof.
\end{proof}

We note that in general, \Cref{assumption:compact_plus_gaussian_target_1} does not imply strong convexity outside of a ball. 
For example, consider the following example in $\mathbb{R}^2$. Let $\pid = \Tilde{\pi} * \gamma$ and
$\Tilde{\pi} = \frac{1}{2}(\delta_{y_1} +\delta_{y_2})$, where $y_i = (0, (-1)^iR)$. Consider a point in $\mathbb{R}^2$ of the form $(x, 0)$. We have that the conditional measure $\rho_{(x,0)}$, where $\rho_x(y) \propto \Tilde{\pi}(y)\gamma(x-y)$, satisfies
\begin{equation*}
    \rho_{(x,0)}(y) = \Tilde{\pi}(y).
\end{equation*}
Therefore, the covariance term in the expression of the Hessian in \eqref{eq:covariance_compact_plus_gaussian_hessian} is given by
\begin{equation*}
    \text{Cov}_{\rho_{(x,0)}}(Y) = R^2\left(\begin{matrix}
        0& 0\\
        0& 1
    \end{matrix}\right).
\end{equation*}
Substituting this into the expression of the Hessian, it follows that
\begin{equation*}
    -\nabla^2\log\pid ((x, 0)) = \frac{I}{\tau^2} - \frac{R^2}{\tau^4}\left(\begin{matrix}
        0& 0\\
        0& 1
    \end{matrix}\right).
\end{equation*}
Taking $R$ to be greater than $\tau$, we have that  $-\nabla^2\log\pid$ evaluated at points on the axis $x=0$ is not positive definite, meaning that $\pid$ cannot be strongly convex outside of any ball. 
Under the additional assumption that the support of the compactly supported measure $\Tilde{\pi}$ is convex and dense in its ambient space, \Cref{assumption:compact_plus_gaussian_target_1} implies \Cref{assumption:grad_log_lipschitzness_convexity_outside_of_a_ball}.
This result is formalised in the following lemma.
\begin{lemma}
    Let $\pid=\Tilde{\pi}*\gamma\in\mathcal{P}(\mathbb{R}^d)$ satisfy assumption \Cref{assumption:compact_plus_gaussian_target_1}. If the support of $\Tilde{\pi}$ is convex and dense in $\mathbb{R}^d$, then $\pid$ satisfies assumption \Cref{assumption:grad_log_lipschitzness_convexity_outside_of_a_ball}.
\end{lemma}
\input{appendix/strong_convexity_outside_ball}
 
We now show that a mixture of Gaussians with different covariances satisfies assumption \Cref{assumption:grad_log_lipschitzness_convexity_outside_of_a_ball} under mild conditions, but it does not generally satisfy assumption \Cref{assumption:compact_plus_gaussian_target_1}.

\begin{lemma}\label{lemma:example_mixture_gaussians_satisfies_assumption}
Let $\pi=\sum_{i=1}^M w_i p_i$ be a mixture of Gaussians in $\mathbb{R}^d$ where $w_i$ and $p_i$ denote the weight and the probability density function, respectively, of the $i$-th component of the mixture which has mean $\mu_i$ and covariance $\Sigma_i$. If for any pair $\{i, j\}$ with $\Sigma_i\neq \Sigma_j$ there exists a unit vector $u$ such that
\begin{equation*}
    u^\intercal\left(\Sigma_i^{-1}-\Sigma_j^{-1}\right)u = 0,
\end{equation*}
but one of the following conditions hold
\begin{enumerate}
    \item[(i)] $u$ is an eigenvector of $\left(\Sigma_i^{-1}-\Sigma_j^{-1}\right)$ with eigenvalue $0$.
    \item[(ii)] $u^\intercal\left(\Sigma_i^{-1}\mu_i-\Sigma_j^{-1}\mu_j\right)\neq 0$.
    \item[(iii)] There exists $k\in\{1, \dots, M\}$ such that $ u^\intercal\left(\Sigma_i^{-1}-\Sigma_k^{-1}\right)u>0$ or $ u^\intercal\left(\Sigma_j^{-1}-\Sigma_k^{-1}\right)u>0$.
\end{enumerate}
Then $\nabla \log \pi$ is Lipschitz. Moreover, if also for any pair $\{i, j\}$ with $\Sigma_i\mu_i\neq \Sigma_j\mu_j$ there exists a unit vector $u$ such that $u^\intercal\left(\Sigma_i^{-1}-\Sigma_j^{-1}\right)u = 0$ but either condition $(ii)$ or $(iii)$ hold,
then $\pi$ satisfies assumption \Cref{assumption:grad_log_lipschitzness_convexity_outside_of_a_ball}.
\end{lemma}
Before presenting the proof, we observe that in one dimension, a mixture of Gaussians with different variances always satisfies the condition stated in the previous lemma, and thus assumption \Cref{assumption:grad_log_lipschitzness_convexity_outside_of_a_ball} holds.
\begin{proof}
We first establish that $\nabla\log\pi$ is Lipschitz continuous by bounding the spectral norm of the Hessian $\nabla^2 \log\pi$. We have the following expressions for $\nabla\log\pi$ and $\nabla^2 \log\pi$
\begin{equation*}
    \nabla\log\pi = \frac{\sum_i w_i p_i \nabla\log p_i}{\sum_i w_i p_i}
\end{equation*}
\begin{align}
    \nabla^2\log\pi & = \frac{\sum_i w_i \nabla^2 p_i}{\sum_i w_i p_i} -\frac{(\sum_i w_i p_i \nabla\log p_i) (\sum_i w_i p_i \nabla\log p_i^\intercal)}{(\sum_i w_i p_i)^2}.\label{eq:hessian_expression_gaussian_mixture_different_covariances}
\end{align}
Observe that 
\begin{equation*}
    \nabla^2 p_i =  p_i \nabla^2\log p_i + p_i  \nabla\log p_i (\nabla\log p_i)^{\intercal}.
\end{equation*}
Substituting this we have
\begin{align*}
    -\nabla^2\log\pi  =&  \frac{\sum_i w_i p_i \Sigma_i^{-1}}{\sum_i w_i p_i}-\frac{\sum_{i,j} w_iw_j p_i p_j [\nabla\log p_i (\nabla\log p_i)^\intercal-\nabla\log p_i (\nabla\log p_j)^\intercal]}{(\sum_i w_i p_i)^2} \\
    =&  \frac{\sum_i w_i p_i \Sigma_i^{-1}}{\sum_i w_i p_i}-\frac{\sum_{i,j} w_iw_j p_i p_j [\Sigma_i^{-1}(x-\mu_i)(x-\mu_i)^\intercal\Sigma_i^{-1} -\Sigma_i^{-1}(x-\mu_i)(x-\mu_j)^\intercal\Sigma_j^{-1}]}{(\sum_i w_i p_i)^2}\\
    =&  \frac{\sum_i w_i p_i \Sigma_i^{-1}}{\sum_i w_i p_i}-\frac{\sum_{i,j} w_iw_j p_i p_j [\Sigma_i^{-1}xx^\intercal(\Sigma_i^{-1}-\Sigma_j^{-1}) +\frac{1}{2}(\Sigma_i^{-1}\mu_i-\Sigma_j^{-1}\mu_j)(\Sigma_i^{-1}\mu_i-\Sigma_j^{-1}\mu_j)^\intercal]}{(\sum_i w_i p_i)^2}\\
    & + \frac{1}{2} \frac{\sum_{i,j} w_iw_j p_i p_j [(\Sigma_i^{-1}\mu_i -\Sigma_j^{-1}\mu_j)x^\intercal(\Sigma_i^{-1}-\Sigma_j^{-1}) + (\Sigma_i^{-1} -\Sigma_j^{-1})x(\Sigma_i^{-1}\mu_i -\Sigma_j^{-1}\mu_j)^\intercal]}{(\sum_i w_i p_i)^2}\\
    =&  \frac{\sum_i w_i p_i \Sigma_i^{-1}}{\sum_i w_i p_i}-\frac{1}{2}\frac{\sum_{i,j} w_iw_j p_i p_j [(\Sigma_i^{-1}-\Sigma_j^{-1})xx^\intercal(\Sigma_i^{-1}-\Sigma_j^{-1}) +(\Sigma_i^{-1}\mu_i-\Sigma_j^{-1}\mu_j)(\Sigma_i^{-1}\mu_i-\Sigma_j^{-1}\mu_j)^\intercal]}{(\sum_i w_i p_i)^2}\\
     & + \frac{\sum_{i,j} w_iw_j p_i p_j [\Sigma_i^{-1}\mu_ix^\intercal(\Sigma_i^{-1}-\Sigma_j^{-1}) + (\Sigma_i^{-1} -\Sigma_j^{-1})x\mu_i^\intercal \Sigma_i^{-1}]}{(\sum_i w_i p_i)^2}.
\end{align*}
Note that in the case of equal covariances $(\Sigma_i = \Sigma_j)$ the terms involving $x$ cancel out. Since the covariance matrices satisfy $\sigma_{i,\min} I\preccurlyeq\Sigma_i\preccurlyeq \sigma_{i,\max} I$ and the norm of the means $\Vert \mu_i\Vert$ is finite for all $i$, the following terms of the previous expression 
\begin{align*}
    A_x + B_x = \frac{\sum_i w_i p_i \Sigma_i^{-1}}{\sum_i w_i p_i}-\frac{1}{2 }\frac{\sum_{i,j} w_iw_j p_i p_j (\Sigma_i^{-1}\mu_i-\Sigma_j^{-1}\mu_j)(\Sigma_i^{-1}\mu_i-\Sigma_j^{-1}\mu_j)^\intercal}{(\sum_i w_i p_i)^2}
\end{align*}
are uniformly bounded above and below for all $x$.
We now focus on the remaining terms which can be rewritten as
\begin{equation*}
     C_x = -\frac{1}{4}\frac{\sum_{i,j} w_iw_j p_i p_j [M_x + M_x^{\intercal}]}{(\sum_i w_i p_i)^2}, \quad M_x = (\Sigma_i^{-1}-\Sigma_j^{-1})x\left[x^\intercal(\Sigma_i^{-1}-\Sigma_j^{-1}) -4 \mu_i^{\intercal}\Sigma_i^{-1}\right]
\end{equation*}
aiming to establish an upper bound for the spectral norm of $C_x$ when $\Vert x\Vert$ tends to $\infty$. Hence, from this point onwards, we consider $x$ such that $\Vert x \Vert>\max_i{\Vert \mu_i \Vert}$.
Using the triangle inequality and the submultiplicativity property of the spectral norm we obtain
\begin{align}
    &\quad\quad\quad\quad\quad\Vert C_x\Vert_2 \leq \frac{1}{4}\sum_{i, j}\frac{ w_iw_j p_i p_j }{(\sum_i w_i p_i)^2}\Vert M_x + M_x^{\intercal} \Vert_2 \label{eq:sum_spectral_norm}\\
    \Vert M_x + M_x^{\intercal} \Vert_2 &\leq  2\left\Vert\Sigma_i^{-1}-\Sigma_j^{-1}\right\Vert_2^2 \Vert x x^\intercal\Vert_2 + 4 \left\Vert\Sigma_i^{-1}-\Sigma_j^{-1}\right\Vert_2
    \left\Vert\Sigma_i^{-1}\right\Vert_2\left\Vert x\mu_i^{\intercal} + \mu_ix^{\intercal}\right\Vert_2 \nonumber
    \\&\leq 2\Vert x\Vert^2\left(\left\Vert\Sigma_i^{-1}-\Sigma_j^{-1}\right\Vert_2^2 + 4 \left\Vert\Sigma_i^{-1}-\Sigma_j^{-1}\right\Vert_2
    \left\Vert\Sigma_i^{-1}\right\Vert_2\right) \leq 10\Vert x\Vert^2\left(\left\Vert\Sigma_i^{-1}\right\Vert_2+ \left\Vert\Sigma_j^{-1}\right\Vert_2\right)^2.\nonumber
\end{align}
Let us define the following sets
\begin{align*}
   D_1 &= \{\{i, j\}|1\leq i, j\leq M, \Sigma_i\neq \Sigma_j \;\text{and}\; \nexists\; u\; |\;\Vert u\Vert = 1\;\text{and}\; u^\intercal\left(\Sigma_i^{-1}-\Sigma_j^{-1}\right)u = 0\},\\
   D_2 &= \{\{i, j\}|1\leq i, j\leq M, \Sigma_i\neq \Sigma_j \;\text{and}\; \exists\; u\; |\;\Vert u\Vert = 1\;\text{and}\; u^\intercal\left(\Sigma_i^{-1}-\Sigma_j^{-1}\right)u = 0\;\text{and}\; (i)\;\text{holds}\},\\
   D_3 &= \{\{i, j\}|1\leq i, j\leq M, \Sigma_i\neq \Sigma_j \;\text{and}\; \exists\; u\; |\;\Vert u\Vert = 1\;\text{and}\; u^\intercal\left(\Sigma_i^{-1}-\Sigma_j^{-1}\right)u = 0\;\text{and}\; (ii)\;\text{holds}\},\\
   D_4 &= \{\{i, j\}|1\leq i, j\leq M, \Sigma_i\neq \Sigma_j \;\text{and}\; \exists\; u\; |\;\Vert u\Vert = 1\;\text{and}\; u^\intercal\left(\Sigma_i^{-1}-\Sigma_j^{-1}\right)u = 0\;\text{and}\; (iii)\;\text{holds}\},
\end{align*}
We also consider the partition of the unit sphere 
$S^{d-1}\subset \mathbb{R}^d$ into disjoint subsets
\begin{equation*}
    S^{d-1} = P_{+}^{(j, i)} \cup P_{-}^{(j, i)} \cup P_{0}^{(j, i)},
\end{equation*}
where $P_{+}^{(j, i)} = \{u\in S^{d-1}| u^{\intercal}(\Sigma_j^{-1}-\Sigma_i^{-1})u>0\}$, with $P_{-}^{(j, i)} $ and $ P_{0}^{(j, i)}$ defined analogously. 

We analyse the terms in the sum \eqref{eq:sum_spectral_norm} separately, depending on the set  $D_k$ to which the pair $\{i, j\}$ belongs. 
\begin{enumerate}
    \item[(1)] Consider the pairs $\{i, j\}\in D_1$, it follows that
\begin{align*}
    \sum_{\{i, j\}\in D_1}\frac{ w_iw_j p_i p_j }{(\sum_i w_i p_i)^2}\Vert M_x + M_x^{\intercal} \Vert_2  \leq 10 \max_{\{i, j\}\in D_1}\left\{\left(\left\Vert\Sigma_i^{-1}\right\Vert_2+ \left\Vert\Sigma_j^{-1}\right\Vert_2\right)^2\right\}\sum_{\{i,j\}\in D_1}\frac{ w_iw_j p_i p_j }{(\sum_i w_i p_i)^2}\Vert x\Vert^2.
\end{align*}
To analyse each term in the previous sum, we first consider the case where one  covariance matrix majorises the other. Specifically, without loss of generality, we assume that for the pair $\{i,j\}\in D_1$ we have $\Sigma_i\succ \Sigma_j$, which implies $\Sigma_{j}^{-1}-\Sigma_i^{-1}\succ \alpha I$ for some $\alpha>0$. We observe that 
\begin{align}
 - \frac{1}{2} x^{\intercal}\left(\Sigma_j^{-1}-\Sigma_i^{-1}\right)x + x^{\intercal}\left(\Sigma_j^{-1}\mu_j-\Sigma_i^{-1}\mu_i\right) &\leq -\frac{1}{2} \alpha  \Vert x\Vert^2 + 2 \Vert x\Vert \left(\left \Vert \Sigma_j^{-1/2}\right \Vert_2^2 \Vert \mu_j\Vert + \left \Vert \Sigma_i^{-1/2}\right\Vert_2^2 \Vert \mu_i\Vert   \right) \nonumber\\
 &\leq -\frac{1}{2} \alpha  \Vert x\Vert^2 + 2 \Vert x\Vert \left(\sigma_{j, \min}^{-1} \Vert \mu_j\Vert + \sigma_{i, \min}^{-1} \Vert \mu_i\Vert   \right), \label{eq:mixture_gaussians_term_bound_1}
\end{align}
which gives
\begin{align}
    \frac{ w_iw_j p_i p_j  \Vert x\Vert^2}{(\sum_i w_i p_i)^2}&\leq \frac{ w_iw_j p_i p_j \Vert x\Vert^2}{(w_i p_i)^2} =\frac{ w_j p_j  \Vert x\Vert^2}{w_i p_i}\nonumber\\
    &=  \frac{w_j \det(\Sigma_i)^{1/2}}{w_i \det(\Sigma_j)^{1/2}}  e^{-\frac{1}{2}\left(\mu_j^\intercal \Sigma_j^{-1} \mu_j-\mu_i^\intercal \Sigma_i^{-1} \mu_i\right)} \Vert x\Vert^2 e^{-\frac{1}{2} x^{\intercal}\left(\Sigma_j^{-1}-\Sigma_i^{-1}\right)x + x^{\intercal}\left(\Sigma_j^{-1}\mu_j-\Sigma_i^{-1}\mu_i\right)}  \xrightarrow{\Vert x\Vert \to \infty} 0.\label{eq:mixture_gaussians_term_bound_2}
\end{align}
On the other hand, when $\Sigma_j^{-1}-\Sigma_i^{-1}$ is neither positive-definite nor negative-definite, for every $x\in\mathbb{R}^d$ we can write $x = \Vert x\Vert u$, where $u$ is a unit vector satisfying $u\in P_{+}^{(j, i)}$ or $u\in P_{-}^{(j, i)}$ because $P_0^{(j, i)}$ is empty by definition for $\{i, j\}\in D_1$.
These two cases can be treated simultaneously since the indices $i, j$ are interchangeable. Without loss of generality, we assume that $u\in P_{+}^{(j, i)}$. Following a similar approach to equations \eqref{eq:mixture_gaussians_term_bound_1} and \eqref{eq:mixture_gaussians_term_bound_2}, we have 
    \begin{align*}
        \frac{ w_iw_j p_i p_j  \Vert x\Vert^2}{(\sum_i w_i p_i)^2}\leq \frac{ w_iw_j p_i p_j \Vert x\Vert^2}{(w_i p_i)^2} =\frac{ w_j p_j  \Vert x\Vert^2}{w_i p_i} \xrightarrow{\Vert x\Vert \to \infty} 0.
    \end{align*}
Therefore, for every unit vector $u\in S^{d-1}$, the limit of each term in the sum over $\{i, j\}\in D_1$  along the line $\Vert x\Vert u$ is zero as $\Vert x\Vert$ tends to $\infty$. Since the sum contains finitely many terms, this implies  
\begin{equation*}
    \lim_{\Vert x\Vert\to\infty} f_{D_1}(\Vert x\Vert u) = \lim_{\Vert x\Vert\to\infty}\sum_{\{i, j\}\in D_1}\frac{ w_iw_j p_i(\Vert x\Vert u) p_j(\Vert x\Vert u)}{(\sum_i w_i p_i(\Vert x\Vert u))^2}\left\Vert M_{\Vert x\Vert u} + M_{\Vert x\Vert u}^{\intercal} \right\Vert_2  = 0.
\end{equation*}
Since $f_{D_1}$ is a continuous function and $S^{d-1}$ is compact, the behaviour of $f_{D_1}$ can be controlled uniformly across all directions. That is, for every $\varepsilon>0$ there exists $R>0$ such that $\Vert x\Vert>R$ implies
\begin{equation*}
    \left\vert\sum_{\{i, j\}\in D_1}\frac{ w_iw_j p_i p_j }{(\sum_i w_i p_i)^2}\Vert M_x + M_x^{\intercal} \Vert_2\right\vert  <\varepsilon.
\end{equation*}
\item[(2)] For $\{i,j\}\in D_2\cup D_3 \cup D_4$, following the same reasoning as above, the limit of the spectral norm of each term when $\Vert x\Vert$ tends to $\infty$ along the directions $u\in P_{+}^{(j, i)}$ or $u\in P_{-}^{(j, i)}$ is $0$. 
To analyse the limit along the directions $u\in P_{0}^{(j, i)}$, we consider the following cases.
\begin{itemize}
    \item $\{i, j\} \in D_2$. For every $x$ such that $x = \Vert x\Vert u$ with $u\in P_{0}^{(j, i)}$, we have that $\left(\Sigma_i^{-1}-\Sigma_j^{-1}\right) u = 0_d$. Consequently, $M_x = 0_{d\times d}$, which demonstrates that the limit along these directions is $0$.
    \item $\{i, j\} \in D_3$. Take $u\in P_{0}^{(j, i)}$, by condition $(ii)$ we have that either $u^\intercal\left(\Sigma_j^{-1}\mu_j-\Sigma_i^{-1}\mu_i\right)<0$ or $u^\intercal\left(\Sigma_j^{-1}\mu_j-\Sigma_i^{-1}\mu_i\right)> 0$. Without loss of generality, assume that $u^\intercal\left(\Sigma_j^{-1}\mu_j-\Sigma_i^{-1}\mu_i\right)< 0$. Then for every $x = \Vert x\Vert u$ we have
\small
\begin{align*}
    &\frac{ w_iw_j p_i p_j }{(\sum_i w_i p_i)^2}\Vert M_x + M_x^{\intercal} \Vert_2 \leq 10 \left(\left\Vert\Sigma_i^{-1}\right\Vert_2+ \left\Vert\Sigma_j^{-1}\right\Vert_2\right)^2 \frac{w_j p_j }{w_i p_i}\Vert x \Vert^2\\
    &\leq 10 \left(\left\Vert\Sigma_i^{-1}\right\Vert_2+ \left\Vert\Sigma_j^{-1}\right\Vert_2\right)^2 \frac{w_j \det(\Sigma_i)^{1/2}}{w_i \det(\Sigma_j)^{1/2}}  e^{-\frac{1}{2}\left(\mu_j^\intercal \Sigma_j^{-1} \mu_j-\mu_i^\intercal \Sigma_i^{-1} \mu_i\right)} \Vert x\Vert^2 e^{\Vert x\Vert u^{\intercal}\left(\Sigma_j^{-1}\mu_j-\Sigma_i^{-1}\mu_i\right)}  \xrightarrow{\Vert x\Vert \to \infty} 0.
\end{align*}
\normalsize

    \item $\{i, j\} \in D_4$. Take $u\in P_{0}^{(j, i)}$, by condition $(iii)$, there exists $k$ such that $u\in P_{+}^{(j, k)}$ or $u\in P_{+}^{(i, k)}$. These two cases are symmetric and can be treated together. Without loss of generality, assume $u\in P_{+}^{(j, k)}$. In this case, following a similar argument to those in equations \eqref{eq:mixture_gaussians_term_bound_1} and \eqref{eq:mixture_gaussians_term_bound_2}, we have that for every $x = \Vert x\Vert u$
\begin{equation*}
    \frac{ w_iw_j p_i p_j }{(\sum_i w_i p_i)^2}\Vert M_x + M_x^{\intercal} \Vert_2 \leq 10 \left(\left\Vert\Sigma_i^{-1}\right\Vert_2+ \left\Vert\Sigma_j^{-1}\right\Vert_2\right)^2 \frac{w_j p_j }{w_k p_k}\Vert x \Vert^2 \xrightarrow{\Vert x\Vert \to \infty} 0.
\end{equation*}
\end{itemize}
\end{enumerate} 
Since the sum in equation \eqref{eq:sum_spectral_norm} contains a finite number of terms, we have that for every $u\in S^{d-1}$
\begin{equation*}
    \lim_{\Vert x\Vert\to\infty}\left\Vert C_{\Vert x\Vert u}\right\Vert_2  = 0.
\end{equation*}
Furthermore, because the function $\Vert C_x\Vert_2$ is continuous and $S^{d-1}$ is compact, the limit $\lim_{\Vert x\Vert\to\infty}\left\Vert C_{x}\right\Vert_2$ exists and is equal to zero.
Consequently, $\Vert C_x\Vert_2$ is bounded for all $x\in\mathbb{R}^d$, which concludes that $\nabla\log\pi$ is Lipschitz.

To complete the proof, we need show that $-\nabla\log\pi$ is strongly convex outside of a ball of radius $r$. Using the same technique as above, we analyse the spectral norm of $B_x$. Let us define the following sets
\begin{align*}
   D_5 &= \{\{i, j\}|1\leq i, j\leq M, \Sigma_i\mu_i\neq \Sigma_j\mu_j \;\text{and}\; \nexists\; u\; |\;\Vert u\Vert = 1\;\text{and}\; u^\intercal\left(\Sigma_i^{-1}-\Sigma_j^{-1}\right)u = 0\},\\
   D_6 &= \{\{i, j\}|1\leq i, j\leq M, \Sigma_i\mu_i\neq \Sigma_j \mu_j\;\text{and}\; \exists\; u\; |\;\Vert u\Vert = 1\;\text{and}\; u^\intercal\left(\Sigma_i^{-1}-\Sigma_j^{-1}\right)u = 0\;\text{and}\; (ii)\;\text{holds}\},\\
   D_7 &= \{\{i, j\}|1\leq i, j\leq M, \Sigma_i\mu_i\neq \Sigma_j \mu_j\;\text{and}\; \exists\; u\; |\;\Vert u\Vert = 1\;\text{and}\; u^\intercal\left(\Sigma_i^{-1}-\Sigma_j^{-1}\right)u = 0\;\text{and}\; (iii)\;\text{holds}\},
\end{align*}
it follows that
\begin{equation*}
    \Vert B_x\Vert_2 \leq \max_{\{i, j\}\in D_5\cup D_6\cup D_7} \left\Vert \Sigma_i^{-1}\mu_i-\Sigma_j^{-1}\mu_j \right\Vert^2 \sum_{\{i, j\}\in D_5\cup D_6\cup D_7} \frac{ w_iw_j p_i p_j}{(\sum_i w_i p_i)^2}.
\end{equation*}
By applying the same reasoning as above, we find that for each pair in the sum 
\begin{equation*}
    \lim_{\Vert x\Vert\to\infty}\frac{ w_iw_j p_i p_j}{(\sum_i w_i p_i)^2} = 0.
\end{equation*}
Since there is a finite number of pairs $\{i,j\}$ in $D_5\cup D_6\cup D_7$, we can conclude that 
\begin{equation*}
    \lim_{\Vert x\Vert\to\infty}\Vert B_x\Vert_2 = 0.
\end{equation*}
Thus, as $\Vert x\Vert$ tends to $\infty$, the only term whose spectral norm does not vanish is
\begin{equation*}
    A_x = \frac{\sum_i w_i p_i \Sigma_i^{-1}}{\sum_i w_i p_i}\succcurlyeq \frac{\sum_i w_i p_i \sigma_{i, \max}^{-1}}{\sum_i w_i p_i} I \succcurlyeq \min_i\{\sigma_{i, \max}^{-1}\} I.
\end{equation*}
Therefore, we can conclude that $\pi$ is strongly log-concave outside of a ball, and hence satisfies assumption \Cref{assumption:grad_log_lipschitzness_convexity_outside_of_a_ball}.
\end{proof}
\begin{remark}\label{remark:counter_example_mixture_gaussian}
A concurrent work \citep{gentilonisilveri2025logconcavityscoreregularityimproved} examines the smoothness of a mixture of Gaussians with covariances of the form $\Sigma_i=\sigma_i^2 I$, a specific case that satisfies the assumption of Lemma~\ref{lemma:example_mixture_gaussians_satisfies_assumption}. 
However, their result does not extend to the case of general covariance matrices. 
In particular, the following example, which falls outside the assumptions of Lemma~\ref{lemma:example_mixture_gaussians_satisfies_assumption}, serves as a counterexample. Let $\pi = \frac{1}{2}(p_1 + p_2)$, where $p_i=\mathcal{N}(\mu, \Sigma_i)$ with $\mu=(1,0)$ and 
\begin{equation*}
    \Sigma_1^{-1} = \begin{pmatrix}
        2 & 1\\
        1& 2\\
    \end{pmatrix}, \quad 
    \Sigma_2^{-1} = \begin{pmatrix}
        2 & 0\\
        0& 3\\
    \end{pmatrix}.
\end{equation*}
We can show that $\nabla^2\log\pi((x, 0))$ is unbounded when $\Vert x\Vert\to\infty$. To establish this, first note that 
\begin{equation*}
    p_1((x,0)) = (2\pi\sqrt{3})^{-1} e^{-(x-1)^2},\quad p_2((x,0)) = (2\pi\sqrt{6})^{-1} e^{-(x-1)^2}.
\end{equation*}
Substituting this into the expression of the Hessian given above \eqref{eq:hessian_expression_gaussian_mixture_different_covariances} we have
\begin{align*}
    -\nabla^2\log\pi((x,0)) = \frac{1}{1+\sqrt{2}}\begin{pmatrix}
        2\sqrt{2} +1&\sqrt{2}\\
        \sqrt{2}& 2\sqrt{2} +3
    \end{pmatrix}- (x-1)^2\frac{\sqrt{2}}{3 + \sqrt{2}}\begin{pmatrix}
        0&0\\
        0& 1
    \end{pmatrix},
\end{align*}
which is clearly unbounded as $\Vert x\Vert\to\infty$, meaning that $\nabla\log\pi$ is not Lipschitz continuous.
\end{remark}
Besides, in general a mixture of Gaussians with different covariances does not satisfy assumption \Cref{assumption:compact_plus_gaussian_target_1}. 
\begin{lemma}\label{lemma:d_mixture_gaussians_not_expressed_as_convoltuion_with_compactly_supported}
Let $\pi = \sum_{i=1}^M w_i\mathcal{N}(\mu_i, \Sigma_i)$ be a mixture of Gaussians in $\mathbb{R}^d$. If either one of the two following assumptions holds
\begin{enumerate}
    \item[(i)] There exists at least one covariance matrix $\Sigma_i$ that cannot be expressed as  $\Sigma_i= \sigma_i I$.
    \item[(ii)] There exists at least one pair $\{i, j\}$ such that $\Sigma_i\neq \Sigma_j$.
\end{enumerate}
Then, $\pi$ does not satisfy assumption \Cref{assumption:compact_plus_gaussian_target_1}.
\end{lemma}
\begin{proof}
We want to determine if $\pi(x)$ can be written as
\begin{equation*}
    \pi(x) = (h*\gamma)(x),
\end{equation*}
where $h$ is a compactly supported measure and $\gamma$ is a Gaussian distribution $\gamma = \mathcal{N}(0, \tau^2 I)$ for some $\tau^2$. 
Assume that $\pi(x) = (h*\gamma)(x)$, we will show that $h$ cannot be compactly supported for any $\tau^2$.

Since the convolution in real space corresponds to multiplication in Fourier space, we have
\begin{equation*}
    \hat{\pi}(k) = \hat{h}(k)\hat{\gamma}(k)
\end{equation*}
where $ \hat{\pi}(k), \hat{h}(k), \hat{\gamma}(k)$ denote the respective Fourier transforms, which have the following expressions
\begin{equation*}
    \hat{\pi}(k) = \sum_{i=1}^M w_i e^{-\frac{1}{2}k^{\intercal}\Sigma_i k - i \mu_i^{\intercal}k}, \quad  \hat{\gamma}(k) = e^{-\frac{1}{2}\tau^2 k^\intercal I k}.
\end{equation*}
Then, the function $\hat{h}(k)$ has to satisfy
\begin{equation*}
    \hat{h}(k) = \frac{\hat{\pi}(k) }{\hat{\gamma}(k)} = \sum_{i=1}^M w_i e^{-\frac{1}{2}k^\intercal (\Sigma_i-\tau^2 I) k- i \mu_i^\intercal k} .
\end{equation*}
Note that $\tau^2$ needs to satisfy $\tau^2 I\preccurlyeq \Sigma_i$ for $i = 1, \dots, M$ as otherwise the inverse Fourier transform of $\hat{h}$ would not yield a real-valued function. 
Under this condition, we have that 
\begin{equation*}
    h(x) = \sum_{i=1}^M w_i \mathcal{N}(\mu_i, \Sigma_i-\tau^2 I) 
\end{equation*}
and since, by assumption, there exist either an index $i$ such that $\Sigma_i \neq \sigma_i I$, or a pair $\{i, j\}$ such that $\Sigma_i \neq \Sigma_j$, then $h$ cannot be compactly supported for any choice of $\tau^2$.

Note that when neither condition $(i)$ nor $(ii)$ holds, we can take $h = \sum_{i=1}^M w_i\delta_{\mu_i}$, where $\delta_{\mu_i}$ denotes a Dirac delta function centred at $\mu_i$.
\end{proof}

One final implication of assumption \Cref{assumption:grad_log_lipschitzness_convexity_outside_of_a_ball} is provided in the following proposition.
\begin{proposition}[Proposition 1 \citep{doi:10.1073/pnas.1820003116}]\label{lemma:assumption_implies_LSI}
    If  $\pid$ satisfies assumption \Cref{assumption:grad_log_lipschitzness_convexity_outside_of_a_ball}, then it has a finite log-Sobolev constant $C_{\text{LSI}, \pid}\leq \frac{2}{M_\pi} e^{16 L_\pi r^2}$.
\end{proposition}
\begin{proof}
Recall that under \Cref{assumption:grad_log_lipschitzness_convexity_outside_of_a_ball} we have that $\pi_{\text{data}} \propto e^{-V_\pi}$ satisfies
     \begin{equation*}
         \inf_{\Vert x\Vert \geq r} \nabla^2 V_\pi \succcurlyeq M_{\pi} I, \quad -L_{\pi} I\preccurlyeq \nabla^2 V_\pi \preccurlyeq L_{\pi} I.
     \end{equation*}
By \citet[Lemma 1]{doi:10.1073/pnas.1820003116}, there exists $\Tilde{V}_\pi\in C^1(\mathbb{R}^d)$ such that $\Tilde{V}_\pi$ is $M_\pi/2$ strongly convex on $\mathbb{R}^d$ and has a Hessian $\nabla^2 \Tilde{V}_\pi$ that exists everywhere on $\mathbb{R}^d$. Therefore, using the Bakry-Émery criterion \citep{bakry_emery}, $\Tilde{\pi}\propto e^{-\Tilde{V}_\pi}$ satisfies log-Sobolev inequality with constant $C_{\text{LSI}, \Tilde{\pi}}\leq 2/M_\pi$. 
Moreover, Lemma 1 in \citet{doi:10.1073/pnas.1820003116} also guarantees that
\begin{equation*}
    \sup \left(\Tilde{V}_\pi(x) - V_\pi\right) - \inf\left(\Tilde{V}_\pi(x) - V_\pi\right) \leq 16 L_\pi r^2.
\end{equation*}
Applying  the Holley-Stroock perturbation principle \citep{RefWorks:RefID:85-holley1987logarithmic}, it follows that $\pid$ has a finite log-Sobolev constant satisfying 
\begin{equation*}
    C_{\text{LSI}, \pid}\leq\frac{2}{ M_{\pi}}e^{16 L_\pi r^2}.
\end{equation*}

\end{proof}

%% file: appendix/strong_convexity_outside_ball.tex
\begin{proof}
By Lemma \ref{lemma:implications_between_assumptions}, we have that $\nabla V_\pi$ is Lipschitz continuous. Thus, it remains to show that $V_\pi$ is strongly convex outside of a ball.

Recall that $\Tilde{\pi}$ is supported on a compact set $S$, that is, $\Tilde{\pi}(y)=0$ for $y\notin S$ and by assumption $S$ is also convex.
For $x\in\mathbb{R}^d$, define the function
\begin{equation*}
    d(x) = \min_{y\in S} \Vert x-y\Vert,
\end{equation*}
which is well defined by compactness. 
Let $y^*(x)\in S$ be the unique (by convexity of $S$) point where the minimum distance is achieved, i.e., $y^*(x)
$ is the projector of $x$ onto $S$. Then, for every $y\in S$ it holds that
\begin{equation*}
    \Vert x-y\Vert\geq d(x),
\end{equation*}
with equality if and only if $y = y^*(x)$. 
Consider the convolution kernel $G\sim\gamma(x-y)$ defined as
\begin{equation*}
    \gamma(x-y)=\frac{1}{(2\pi\tau^2)^{d/2}}e^{-\frac{\Vert x-y\Vert^2}{2\tau^2}}.
\end{equation*}
Note that the value at $y = y^*(x)$ is given by
\begin{equation*}
    \gamma(x-y^*(x))=\frac{1}{(2\pi\tau^2)^{d/2}}e^{-\frac{d(x)^2}{2\tau^2}}.
\end{equation*}
Besides, for any $y\in S$, we have that
\begin{equation*}
    d(x)\leq\Vert x-y\Vert\leq d(x) + \Vert y^*(x)-y\Vert.
\end{equation*}
Because $S$ is a compact set, the term $\Vert y^*(x)-y\Vert$ for $y\in S$ is bounded independently of $x$, therefore we can write
\begin{equation*}
    \Vert x-y\Vert = d(x) + \delta_x(y),
\end{equation*}
for some $\delta_x(y)\geq 0$, with $\delta_x(y)$ if and only if $y=y^*(x)$ and $\delta_x(y)\leq\Vert y^*(x)-y\Vert$, which implies that $\delta(y)$ remains uniformly bounded for all $x\in\mathbb{R}^d$, in particular as $\Vert x\Vert\to\infty$.
Using this, the convolution kernel can be written as
\begin{equation*}
    \gamma(x-y) = \frac{1}{(2\pi\tau^2)^{d/2}}e^{-\frac{(d(x) + \delta_x(y))^2}{2\tau^2}}.
\end{equation*}
Thus, we obtain the ratio
\begin{equation*}
    \frac{\gamma(x-y)}{\gamma(x-y^*(x))} = e^{-\frac{2d(x)\delta_x(y) + \delta_x(y)^2}{2\tau^2}}.
\end{equation*}
Since $\delta_x(y)$ is uniformly bounded for all $x$, we observe that for $y\neq y^*(x)$, as $\Vert x\Vert\to\infty$ the leading order of the exponent is $d(x)\delta_x(y)$, where $\delta_x(y)>0$ and $d(x)$ grows with order $\Vert x\Vert$. 
Meaning that the ratio becomes arbitrarily small as $\Vert x\Vert \to\infty$ when $y\neq y^*(x)$.
That is, the contribution from $y\neq y^*(x)$ becomes exponentially negligible compared to the contribution from $y^*(x)$.

Given now a bounded test function $f$, we have
\begin{equation*}
    \frac{\int f(y)\gamma(x-y)\Tilde{\pi}(y)\md y}{\int\gamma(x-y)\Tilde{\pi}(y)\md y}  = \frac{\int f(y) e^{-\frac{d(x)\delta_x(y) + \delta_x(y)^2}{2\tau^2}}\gamma(x-y^*(x))\Tilde{\pi}(y)\md y}{\int e^{-\frac{d(x)\delta_x(y) + \delta_x(y)^2}{2\tau^2}}\gamma(x-y^*(x))\Tilde{\pi}(y)\md y}.
\end{equation*}
By the dominated convergence theorem, the contribution in both integrals for $y\neq y^*(x)$ vanishes as $\Vert x\Vert\to \infty$. Therefore, we have that for any fixed $\varepsilon>0$ and any small radius $\delta>0$, there exists $r$ such that for all $\Vert x\Vert>r$, the conditional measure $\rho_x(y)\propto \Tilde{\pi}(y)\gamma(x-y)$ satisfies
\begin{equation*}
    \rho_x\left(S\setminus B(y^*(x), \delta) \right)<\varepsilon,
\end{equation*}
where $B(y^*(x), \delta)$ denotes the ball of radius $\delta$ centred at $y^*(x)$. 
Intuitively, this means that for sufficiently large $\Vert x\Vert$, almost all the mass of $\rho_x$ is concentrated within an arbitrarily small ball around $y^*(x)$. 
It is important to note that, due to the assumption that $S$ is dense in $\mathbb{R}^d$, for any $\delta>0$, there always exist a point $z\in B(y^*(x), \delta)$ such that $\Tilde{\pi}(z)>0$.

Consequently, the mean $\mu_x = \int_S y\rho_x(\md y)$ must be very close to $y^*(x)$, and for any point $y$ in the high-probability region, we have $\Vert y-\mu(x)\Vert\leq 2\delta$ (with the worst-case scenario occurring when $\mu(x)$ lies on the edge of $B(y^*(x), \delta)$). This implies that for sufficiently large $\Vert x\Vert$, the spread of $\rho_x$ becomes arbitrarily small. In particular, the covariance matrix satisfies
\begin{equation*}
    \Vert\text{Cov}_{\rho_x}(Y)\Vert\leq (2\gamma)^2.
\end{equation*}
Taking the limit as $\Vert x\Vert \to\infty$, we have that $\delta\to0$, thus, we obtain that 
\begin{equation*}
    \lim_{\Vert x\Vert\to\infty}\text{Cov}_{\rho_x}(Y) = 0.
\end{equation*}
The final step is to note that 
\begin{equation*}
    \nabla^2 V_\pi(x) = \frac{1}{\tau^2} I -\frac{1}{\tau^4}\text{Cov}_{\rho_x}(Y).
\end{equation*}

\end{proof}

%% file: appendix/assumption_hessian_decay.tex
We show below that assumption \Cref{assumption:grad_log_lipschitzness_hessian_decay} is satisfied by multivariate Student's $t$ distributions of the form
\begin{equation*}
    \pid(x) = C_{\pi} \left( 1 + \frac{1}{\alpha} (x-\mu)^{\intercal} \Sigma^{-1} (x-\mu)\right)^{-(\alpha + d)/2}, \;\; x\in\mathbb{R}^d,
\end{equation*}
where the covariance matrix $\Sigma$ is a positive definite matrix satisfying $\sigma_{\min} I\preccurlyeq\Sigma\preccurlyeq \sigma_{\max} I$ and $\alpha>0$ denotes the degrees of freedom.
The Hessian of the potential has the following expression
\begin{equation*}
    -\nabla^2 \log \pid(x) = \frac{\alpha + d}{\alpha} \frac{ \Sigma^{-1}}{1 + \frac{1}{\alpha} (x-\mu)^{\intercal}\Sigma^{-1}(x-\mu)} -\frac{2(\alpha + d)}{\alpha^2}\frac{ \Sigma^{-1}(x-\mu)(x-\mu)^{\intercal}\Sigma^{-1}}{\left(1 + \frac{1}{\alpha} (x-\mu)^{\intercal}\Sigma^{-1}(x-\mu)\right)^2}.
\end{equation*}
The matrix $(x-\mu)(x-\mu)^{\intercal}$ is positive semidefinite and satisfies
\begin{equation*}
    0 \preccurlyeq (x-\mu)(x-\mu)^{\intercal}\preccurlyeq \Vert x-\mu\Vert^2 I.
\end{equation*}
Since the eigenvalues of the product of symmetric positive semidefinite matrices satisfy the following
\begin{equation*}
    \lambda_{\min}(ABC)\geq \lambda_{\min}(A)\lambda_{\min}(B)\lambda_{\min}(C),\quad \lambda_{\max}(ABC)\leq \lambda_{\max}(A)\lambda_{\max}(B)\lambda_{\max}(C),
\end{equation*}
we have that 
\begin{equation*}
    0\preccurlyeq \Sigma^{-1}(x-\mu)(x-\mu)^{\intercal}\Sigma^{-1} \preccurlyeq \sigma_{\min}^{-2} \Vert x-\mu\Vert^2 I.
\end{equation*}
This leads to 
\begin{align*}
    -\nabla^2\log\pid(x)&\preccurlyeq \frac{\alpha + d}{\alpha} \frac{ \Sigma^{-1}}{1 + \frac{1}{\alpha} (x-\mu)^{\intercal}\Sigma^{-1}(x-\mu)}\preccurlyeq \frac{(\alpha + d) \sigma_{\min}^{-1}}{\alpha} \frac{I}{1 + \frac{1}{\alpha} \Vert S^{-1/2} U (x-\mu)\Vert^2}\\
    &\preccurlyeq \frac{(\alpha + d) \sigma_{\min}^{-1}}{\alpha + {\sigma_{\max}^{-1}} \Vert x-\mu\Vert^2} I,
\end{align*}
where we have used that $\Sigma^{-1} = U^\intercal S^{-1}U$, $U$ being an orthogonal matrix and $S$ being diagonal.
On the other hand,
\begin{align*}
    -\nabla^2\log\pid(x)&\succcurlyeq \frac{(\alpha + d)\sigma_{\max}^{-1}}{\alpha + {\sigma_{\min}^{-1}} \Vert x-\mu\Vert^2} I -\frac{2(\alpha + d)\sigma_{\min}^{-2}}{\sigma_{\max}^{-1}}\frac{\sigma_{\max}^{-1} \Vert x-\mu\Vert^2}{\left({\alpha + {\sigma_{\max}^{-1}} \Vert x-\mu\Vert^2}\right)^2} I\\
    &\succcurlyeq \frac{(\alpha + d)\sigma_{\max}^{-1}}{\alpha + {\sigma_{\min}^{-1}} \Vert x-\mu\Vert^2} I -\frac{2(\alpha + d)\sigma_{\min}^{-2}}{\sigma_{\max}^{-1}}\frac{I}{{\alpha + {\sigma_{\max}^{-1}} \Vert x-\mu\Vert^2}}.\\
    &\succcurlyeq -\frac{2(\alpha + d)\sigma_{\min}^{-2}}{\sigma_{\max}^{-1}}\frac{I}{{\alpha + {\sigma_{\max}^{-1}} \Vert x-\mu\Vert^2}},
\end{align*}
which concludes that $\pid$ satisfies assumption \Cref{assumption:grad_log_lipschitzness_hessian_decay}.

%% file: appendix/vector_pi.tex
Before jumping into the proof of Lemma \ref{lem:regularity_of_gaussian_path}, we provide a generalisation of the Poincaré inequality that will be required later.
\begin{lemma}\label{lemma:PI_for_vector_valued_functions}
Let $\pi$ be a probability distribution in $\mathbb{R}^d$ with finite Poincaré constant $C_{\text{PI}, \pi}$. Then, for all functions $f:\mathbb{R}^d\to\mathbb{R}^d$, $f\in L^2(\pi)$, it holds that
\emph{\begin{equation*}
    \text{Cov}_{X\sim\pi}\left[f(X)\right] \preccurlyeq C_{\text{PI}, \pi} \mathbb{E}_{\pi}\left[\nabla f(X)\nabla f(X)^\intercal\right].
\end{equation*}}
\end{lemma}
\begin{proof}
Given the function $f:\mathbb{R}^d\to\mathbb{R}^d$, $f\in L^2(\pi)$. 
For every unit vector $u\in\mathbb{R}^d$, consider the function $u^\intercal f:\mathbb{R}^d\to\mathbb{R}$. Since $\pi$ satisfies a Poincaré inequality we have
\begin{equation*}
    u^\intercal \text{Cov}_{X\sim\pi}\left[f(X)\right] u = \text{Cov}_{X\sim\pi}\left[u^\intercal f(X)\right] \leq C_{\text{PI}, \pi}\mathbb{E}_{X\sim \pi}\left[\left\Vert \sum_{i} u_i \nabla f_i(X)\right\Vert^2\right] = C_{\text{PI}, \pi} u^\intercal \mathbb{E}_{\pi}\left[\nabla f(X)\nabla f(X)^\intercal\right] u,
\end{equation*}
which concludes the proof.
\end{proof}

%% file: appendix/proof_smoothness.tex
\begin{proof}[Proof of Lemma~\ref{lem:regularity_of_gaussian_path} under assumption \Cref{assumption:grad_log_lipschitzness_convexity_outside_of_a_ball}]
To simplify notation in the proof, we denote $\pi_{\text{data}}$ as $\pi$.
    Let $\mu_t\propto e^{-U_t}$, our aim is to show that $\nabla U_t$ is Lipschitz, to do so we are going to show that the Hessian $\nabla^2 U_t$ is bounded for all $t\in[0, T]$. 
    Using that if $f$ and $g$ are differentiable functions then $\nabla(f*g) = (\nabla f)* g = f * (\nabla g)$, we have the following expressions for $\nabla U_t$ and $\nabla^2 U_t$
\begin{align}
    \nabla U_t(x) &= \frac{1}{\sigma^2(1-\lambda_t)}\left(x-\mathbb{E}_{\rho_{t, x}(y)}\left[Y\right]\right) = \frac{1}{\sqrt{\lambda_t}}\mathbb{E}_{\rho_{t, x}(y)}\left[\nabla V_{\pi}\left(\frac{Y}{\sqrt{\lambda_t}}\right)\right]\label{eq:score_conv_expressions} \\
\nabla^2 U_t(x) &=\frac{1}{\sigma^2(1-\lambda_t)}\left(I - \frac{1}{\sigma^2(1-\lambda_t)}\text{Cov}_{\rho_{t, x}}\left[Y\right]\right)\label{eq:hessian_conv_1}\\
\nabla^2 U_t(x) &=\frac{1}{\lambda_t}\left(\mathbb{E}_{\rho_{t, x}}\left[\nabla^2 V_\pi\left(\frac{Y}{\sqrt{\lambda_t}}\right)\right] - \text{Cov}_{\rho_{t, x}}\left[\nabla V_\pi\left(\frac{Y}{\sqrt{\lambda_t}}\right)\right]\right)\label{eq:hessian_conv_2}\\
    \nabla^2 U_t(x) &=\frac{1}{\sigma^2(1-\lambda_t)\lambda_t}\text{Cov}_{\rho_{t, x}}\left[Y, \nabla V_\pi\left(\frac{Y}{\sqrt{\lambda_t}}\right)\right],\label{eq:hessian_conv_3}
\end{align}
where $\rho_{t, x}(y)\propto e^{-V_\pi(y/\sqrt{\lambda_t})-\Vert x-y\Vert^2/(2\sigma^2(1-\lambda_t))}$.
It is worth mentioning that as $t\to 0$, $\rho_{t, x}$ tends to a Dirac delta centred at $0$, and as $t\to T$, it approaches a Dirac delta centred at $x$, both of which have zero variance.
Note that \eqref{eq:hessian_conv_1} and \eqref{eq:hessian_conv_2} admit the following upper bounds
\begin{align}
    \nabla^2 U_t(x) &\preccurlyeq\frac{1}{\sigma^2(1-\lambda_t)}I \label{eq:hessian_upper_bound_1}\\
\nabla^2 U_t(x) &\preccurlyeq\frac{L_\pi}{\lambda_t} I,\label{eq:hessian_upper_bound_2}
\end{align}
where we have used that the covariance matrix is positive semidefinite.
To find a lower bound for $\nabla^2 U_t$ we need to upper bound $\text{Cov}_{\rho_{t, x}}\left[Y\right]$ and $\text{Cov}_{\rho_{t, x}}\left[\nabla V_\pi(Y/\sqrt{\lambda_t})\right]$. Observe that if $\rho_{t, x}$ satisfies the Poincaré inequality with constant $C_{\text{PI}, \rho_{t}}$ independent of $x$, then using the generalisation of the Poincaré inequality for vector-valued random variables given in Lemma~\ref{lemma:PI_for_vector_valued_functions}, we have
\begin{align*}
    \text{Cov}_{\rho_{t, x}}&\left[Y\right]\preccurlyeq C_{\text{PI}, \rho_{t}} I\\
    \text{Cov}_{\rho_{t, x}}&\left[\nabla V_\pi\left(\frac{Y}{\sqrt{\lambda_t}}\right)\right]\preccurlyeq \frac{C_{\text{PI}, \rho_t}}{\lambda_t} \mathbb{E}_{\rho_{t, x}}\left[\nabla^2 V_\pi\left(\frac{Y}{\sqrt{\lambda_t}}\right)\nabla^2 V_\pi\left(\frac{Y}{\sqrt{\lambda_t}}\right)^{\intercal}\right]\preccurlyeq \frac{C_{\text{PI}, \rho_t} L_\pi^2}{\lambda_t} I.
\end{align*}
This implies that for each $t$ we have that 
\begin{align}
    \nabla^2 U_t(x) &\preccurlyeq \min\left\{\frac{1}{\sigma^2(1-\lambda_t)}, \frac{L_\pi}{\lambda_t}\right\} I =: a_t I \label{eq:hessian_final_bound_1}\\
    \nabla^2 U_t(x) &\succcurlyeq \max\left\{\frac{1}{\sigma^2(1-\lambda_t)}\left(1 - \frac{C_{\text{PI}, \rho_t}}{\sigma^2(1-\lambda_t)}\right), -\frac{L_\pi}{\lambda_t}\left(1 + \frac{C_{\text{PI}, \rho_t} L_\pi}{\lambda_t}\right)\right\} I =: b_t I.\label{eq:hessian_final_bound_2}
\end{align}
Therefore, the Lipschitz constant $L_t$ satisfies the following
\begin{equation}\label{eq:lipschitzness_constant_moving_target}
    L_t \leq \max\{a_t, \vert b_t\vert \}.
\end{equation}
To conclude the proof we need to check that $C_{\text{PI}, \rho_t}$ is independent of $x$, since otherwise the Poincaré constant can get arbitrarily large and we will not have a meaningful bound for the Hessian.
Note that if we denote $\rho_{t, x}\propto e^{-V_{\rho_{t, x}}}$ we have that 
\begin{equation}\label{eq:hessian_bound_joint_1}
    \left(-\frac{L_\pi}{\lambda_t} + \frac{1}{\sigma^2(1-\lambda_t)}\right) I\preccurlyeq\nabla^2 V_{\rho_{t, x}}(y) = \frac{1}{\lambda_t} \nabla^2 V_{\pi}\left(\frac{y}{\sqrt{\lambda_t}}\right) + \frac{I}{\sigma^2(1-\lambda_t)}\preccurlyeq \left(\frac{L_\pi}{\lambda_t} + \frac{1}{\sigma^2(1-\lambda_t)}\right) I.
\end{equation} 

Note that if $\left(-\frac{L_\pi}{\lambda_t}+\frac{1}{\sigma^2(1-\lambda_t)}\right)>0$, then $V_{\rho_{t, x}}$ is strongly convex and using Bakry-Émery criterion \citep{bakry_emery} we have that $\rho_{t, x}$ satisfies a log-Sobolev inequality which implies a Poincaré inequality with the same constant. Thus, for $t$ such that 
\begin{equation*}
    \Tilde{\lambda} = \frac{\sigma^2 L_\pi}{1+\sigma^2L_\pi} < \lambda_t \leq 1,
\end{equation*}
$\rho_{t, x}$ satisfies the Poincaré inequality with constant $\left(\frac{1}{\sigma^2(1-\lambda_t)}-\frac{L_\pi}{\lambda_t}\right)^{-1}$ independent of $x$, which tends to $0$ as $\lambda_t$ tends to $1$.

On the other hand, using that the potential $V_\pi$ is strongly convex outside of a ball of radius $r$, we have that for $\Vert y\Vert >\sqrt{\lambda_t} r$
\begin{equation}\label{eq:hessian_boudn_joint_2}
   \left( \frac{M_{\pi}}{\lambda_t} + \frac{1}{\sigma^2(1-\lambda_t)} \right) I \preccurlyeq \nabla^2 V_{\rho_{t, x}} (y) \preccurlyeq \left(\frac{L_\pi}{\lambda_t} + \frac{1}{\sigma^2(1-\lambda_t)}\right) I. 
\end{equation}
Equations \eqref{eq:hessian_bound_joint_1}-\eqref{eq:hessian_boudn_joint_2} imply that for $\lambda_t>0$, $\nabla\log \rho_{t, x}$ is Lipschitz continuous and $\rho_{t, x}$ is strongly log-concave outside of a ball of radius $r_t = \sqrt{\lambda_t} r$. 
Similarly to the proof in Lemma \ref{lemma:assumption_implies_LSI}, the existence of a smooth strongly convex approximation of $V_{\rho_{t, x}}$ and the Holley-Stroock perturbation lemma \citep{RefWorks:RefID:85-holley1987logarithmic} imply that $\rho_{t, x}$ satisfies a log-Sobolev inequality and hence a Poincaré inequality with constant
\begin{equation}\label{eq:poincare_inequality_joint}
    C_{\text{PI}, \rho_t} \leq 2\left(\frac{M_{\pi}}{\lambda_t}+ \frac{1}{\sigma^2(1-\lambda_t)} \right)^{-1} e^{16 \left(L_\pi + \frac{\lambda_t}{\sigma^2(1-\lambda_t)}\right) r^2},
\end{equation}
independent of $x$. Observe that when $\lambda_t$ tends to $0$ the upper bound of the Poincaré constant $C_{\text{PI}, \rho_t}$ also tends to $0$. Therefore, for $\lambda_t\in[0, \Tilde{\lambda}]$ $C_{\text{PI}, \rho_t}$ is bounded by \eqref{eq:poincare_inequality_joint}, while for $\lambda_t\in(\Tilde{\lambda}, 1]$, $C_{\text{PI}, \rho_t}$ is bounded by 
\begin{equation*}
    C_{\text{PI}, \rho_t} \leq \min\left\{\left(\frac{1}{\sigma^2(1-\lambda_t)}-\frac{L_\pi}{\lambda_t}\right)^{-1}, 2\left(\frac{M_{\pi}}{\lambda_t}+ \frac{1}{\sigma^2(1-\lambda_t)} \right)^{-1} e^{16 \left(L_\pi + \frac{\lambda_t}{\sigma^2(1-\lambda_t)}\right) r^2}\right\}.
\end{equation*}
\end{proof}

%% file: appendix/proof_smoothness_heavy_tail_target.tex
It is important to note that, in the proof of Lemma~\ref{lem:regularity_of_gaussian_path} under assumption \Cref{assumption:grad_log_lipschitzness_hessian_decay} below, we only rely on the lower bound of the Hessian from \Cref{assumption:grad_log_lipschitzness_hessian_decay}. If we omit the upper bound on the Hessian,  \Cref{assumption:grad_log_lipschitzness_hessian_decay} becomes a generalisation of \Cref{assumption:grad_log_lipschitzness_convexity_outside_of_a_ball}. The stronger assumption in \Cref{assumption:grad_log_lipschitzness_convexity_outside_of_a_ball} results in tighter bounds for the Lipschitz constants along the diffusion path, thereby improving upon those in \citet{gao2024gaussian}, as our bounds are non-vacuous for all $t\in[0, T]$.

\begin{proof}[Proof of Lemma~\ref{lem:regularity_of_gaussian_path} under assumption \Cref{assumption:grad_log_lipschitzness_hessian_decay}]
Let $\mu_t\propto e^{-U_t}$. 
Note that the expressions for the Hessian of $U_t$ given in \eqref{eq:hessian_conv_1}-\eqref{eq:hessian_conv_3} remain valid in this case. Consequently, the bounds provided in \eqref{eq:hessian_upper_bound_1}-\eqref{eq:hessian_upper_bound_2} also hold.

On the other hand, if $\rho_{t, x}(y)\propto e^{-V_\pi(y/\sqrt{\lambda_t})-\Vert x-y\Vert^2/(2\sigma^2(1-\lambda_t))}$ satisfies a Poincaré inequality with constant $C_{\text{PI}, \rho_t}$ independent of $x$, then the bounds \eqref{eq:hessian_final_bound_1}-\eqref{eq:lipschitzness_constant_moving_target} are valid. Therefore, to conclude that $\nabla\log\mu_t$ is Lipschitz we need to show that $C_{\text{PI}, \rho_t}$ is independent of $x$. Note that under assumption \textbf{A}\ref{assumption:grad_log_lipschitzness_hessian_decay}, we have that
\begin{equation*}
    \left(-\frac{L_\pi}{\lambda_t} + \frac{1}{\sigma^2(1-\lambda_t)}\right) I\preccurlyeq\nabla^2 V_{\rho_{t, x}}(y) \preccurlyeq \left(\frac{L_\pi}{\lambda_t} + \frac{1}{\sigma^2(1-\lambda_t)}\right) I,
\end{equation*}
and for $y$ such that $\Vert y\Vert>\sqrt{\lambda_t }r$
\begin{equation*}
    \left( - \frac{1}{\lambda_t\alpha_1 +\alpha_2 \Vert y\Vert^{2}}+\frac{1}{\sigma^2(1-\lambda_t)}\right) I \preccurlyeq \nabla^2 V_{\rho_{t, x}} (y)\preccurlyeq \left(\frac{1}{\lambda_t\beta_1 +\beta_2 \Vert y\Vert^{2}}+\frac{1}{\sigma^2(1-\lambda_t)}\right) I.
\end{equation*}
Note that if $1\geq\lambda_t>\frac{\sigma^2 L_\pi}{1+\sigma^2L_\pi}= \Tilde{\lambda}$, then $V_{\rho_{t, x}}$ is strongly convex and $\rho_{t, x}$ satisfies a Poincaré inequality with constant $\left(\frac{1}{\sigma^2(1-\lambda_t)}-\frac{L_\pi}{\lambda_t}\right)^{-1}$, which tends to $0$ as $\lambda_t$ tends to $1$. On the other hand, define 
\begin{equation*}
    \tilde{r}_t^2 = \max \{\lambda_t r^2, (2\sigma^2(1-\lambda_t)-\lambda_t\alpha_1)\alpha_2^{-1}\}.
\end{equation*}
We have that $V_{\rho_{t, x}}$ is strongly convex outside of a ball of radius $\Tilde{r}_t$. That is, for $\Vert y\Vert>\Tilde{r}_t$, it follows that $\nabla^2 V_{\rho_{t, x}}\succcurlyeq I/(2\sigma^2(1-\lambda_t))$. Therefore, as in Lemma~\ref{lemma:assumption_implies_LSI}, leveraging the existence of a smooth strongly convex approximation of $V_{\rho_{t, x}}$ and Holley-Stroock perturbation lemma \citep{RefWorks:RefID:85-holley1987logarithmic}, $\rho_{t, x}$ satisfies a Poincaré inequality with constant 
\begin{equation}
    C_{\text{PI}, \rho_t} \leq 2 \left(\frac{1}{2\sigma^2(1-\lambda_t)}\right)^{-1} e^{16\left(\frac{L_\pi}{\lambda_t} + \frac{1}{\sigma^2(1-\lambda_t)}\right)\Tilde{r}_t^2}, \label{eq:bound_C_poincare_inequality}
\end{equation}
independent of $x$. Therefore, for $\lambda_t\in(0, \Tilde{\lambda}]$ $C_{\text{PI}, \rho_t}$ is bounded by \eqref{eq:bound_C_poincare_inequality}, while for $\lambda_t\in(\Tilde{\lambda}, 1]$, $C_{\text{PI}, \rho_t}$ is bounded by
\begin{equation*}
    C_{\text{PI}, \rho_t}\leq \min\left\{\left(\frac{1}{\sigma^2(1-\lambda_t)}-\frac{L_\pi}{\lambda_t}\right)^{-1},  2 \left(\frac{1}{2\sigma^2(1-\lambda_t)}\right)^{-1} e^{16\left(\frac{L_\pi}{\lambda_t} + \frac{1}{\sigma^2(1-\lambda_t)}\right)\Tilde{r}_t^2}\right\}.
\end{equation*}
Finally, observe that if $\lambda_0=0$, then $\mu_0 = \nu = \mathcal{N}(0, \sigma^2 I)$, which implies that $\nabla\log\mu_0$ is Lipschitz continuous. 
This concludes that $\nabla\log\mu_t$ is Lipschitz continuous for all $t\in[0,T]$.
\end{proof}

%% file: appendix/action_analysis.tex
\begin{proof}
When the schedule satisfies $\max_{t\in[0,T]}\vert \partial_t{\log \lambda_t}\vert \leq C_\lambda$, we consider the reparametrised version of $\mu_t$ in terms of the schedule $\lambda_t$, denoted as $\Tilde{\mu}_\lambda$ and let $X_{\lambda}\sim\Tilde{\mu}_{\lambda}$ and $X_{\lambda + \delta}\sim\Tilde{\mu}_{\lambda + \delta}$.
Recall that 
\begin{equation}\label{appendix:eq_conv_path_random_variables}
    X_{\lambda} = \sqrt{\lambda} X + \sqrt{1-\lambda}\sigma Z 
\end{equation}
where $X\sim\pid$ and $Z\sim \mathcal{N}(0, I)$ are independent from each other. 
We introduce a new random variable $\Tilde{X}$, independent from $Z$, that follows a Gaussian distribution,  $\Tilde{X}\sim N(0, \sigma_\pi^2 I)$, satisfying 
\begin{equation*}
    \sigma_\pi = \argmin_{{\hat{\sigma}}\geq \sigma} W_2\left( \mathcal{L}(X),\mathcal{L}(Y_{\hat{\sigma}})\right), \quad \text{where} \;\; Y_{\hat{\sigma}}\sim \mathcal{N}(0, \hat{\sigma}^2 I), \; Z \independent Y_{\hat{\sigma}}.
\end{equation*}
Furthermore, we select $\tilde{X}$ to be the specific random variable that attains the minimal coupling with $X$, that is, $W_2^2\left( \mathcal{L}(X),\mathcal{L}(\tilde{X})\right) = \mathbb{E}\left[\Vert X-\tilde{X}\Vert^2\right]$.
Using the random variable $\Tilde{X}$, we can rewrite \eqref{appendix:eq_conv_path_random_variables} as
\begin{align*}
    X_\lambda &= \sqrt{\lambda} (X-\Tilde{X}) + \sqrt{\lambda} \Tilde{X} +\sqrt{1-\lambda}\sigma Z\\
    & \overset{d}{=} \sqrt{\lambda} (X-\Tilde{X}) + \sqrt{\lambda\sigma_\pi^2 + (1-\lambda)\sigma^2} Y,
\end{align*}
where $Y\sim \mathcal{N}(0, I)$.  
The Wasserstein-2 distance between $\Tilde{\mu}_\lambda$ and $\Tilde{\mu}_{\lambda+\delta}$ is given by
\begin{align*}
     W_2^2(\Tilde{\mu}_\lambda, &\Tilde{\mu}_{\lambda+\delta}) \leq \mathbb{E}\left[\Vert X_\lambda-X_{\lambda + \delta}\Vert^2\right]\\
&\leq 2\mathbb{E}\left[\left\Vert(\sqrt{\lambda}-\sqrt{\lambda+\delta})(X-\Tilde{X})\right\Vert^2\right] +2 \mathbb{E}\left[\left\Vert\left(\sqrt{\lambda\sigma_{\pi}^2 + (1-\lambda)\sigma^2}-\sqrt{(\lambda + \delta)\sigma_{\pi}^2 +(1-\lambda-\delta)\sigma^2}\right) Y\right\Vert^2\right]\\
&= 2(\sqrt{\lambda}-\sqrt{\lambda+\delta})^2 \mathbb{E}\left[\Vert X-\Tilde{X}\Vert^2\right] + 2\left(\sqrt{\lambda\sigma_{\pi}^2 + (1-\lambda)\sigma^2}-\sqrt{(\lambda + \delta)\sigma_{\pi}^2 +(1-\lambda-\delta)\sigma^2}\right)^2d.
\end{align*}
Using the definition of the metric derivative we have
\begin{equation*}
    \left\vert\Dot{\Tilde{\mu}}\right\vert_\lambda^2 = \lim_{\delta\to 0}\frac{ W_2^2(\Tilde{\mu}_\lambda, \Tilde{\mu}_{\lambda+\delta})}{\delta^2}\leq \frac{\mathbb{E}\left[\Vert X-\Tilde{X}\Vert^2\right]}{2 \lambda} + \frac{(\sigma_\pi^2-\sigma^2)^2}{2(\lambda\sigma_\pi^2 + \sigma^2(1-\lambda))}d = \frac{\mathbb{E}\left[\Vert X-\Tilde{X}\Vert^2\right]}{2 \lambda} + \frac{(\sigma_\pi^2-\sigma^2)^2}{2(\sigma^2 + \lambda(\sigma_{\pi}^2-\sigma^2))}d.
\end{equation*}
Since $\mu_t = \Tilde{\mu}_{\lambda_t}$, we have that $\vert\Dot{\mu}\vert_t = \left\vert\Dot{\Tilde{\mu}}\right\vert_\lambda\left\vert\partial_t{\lambda_t}\right\vert$. Using the assumption on the schedule we have the following expression for the action
\begin{align}
    \mathcal{A}_{\lambda}(\mu) &= \int_0^T \vert\Dot{\mu}\vert_t^2 \md t = \int_0^T \left\vert\Dot{\Tilde{\mu}}\right\vert_\lambda^2\left\vert\partial_t{\lambda_t}\right\vert^2 \md t\nonumber\\
    &\leq \int_0^T \left(\frac{\mathbb{E}\left[\Vert X-\Tilde{X}\Vert^2\right]}{2 \lambda_t} + \frac{(\sigma_\pi^2-\sigma^2)^2}{2(\sigma^2 + \lambda_t(\sigma_{\pi}^2-\sigma^2))}d\right)\left\vert \partial_t{\lambda_t}\right\vert^2\md t \nonumber\\
    &= \int_0^T \left(\frac{\mathbb{E}\left[\Vert X-\Tilde{X}\Vert^2\right]}{2} + \frac{(\sigma_\pi^2-\sigma^2)^2}{2(\sigma^2/\lambda_t + \sigma_{\pi}^2-\sigma^2)}d\right)\left\vert \partial_t{\log\lambda_t}\right\vert \left\vert \partial_t{\lambda_t}\right\vert\md t \nonumber\\    
    &\leq C_\lambda \int_0^T \left(\frac{\mathbb{E}\left[\Vert X-\Tilde{X}\Vert^2\right]}{2} +\frac{(\sigma_\pi^2-\sigma^2)^2}{2(\sigma^2/\lambda_t + \sigma_{\pi}^2-\sigma^2)}d\right) \left\vert \partial_t{\lambda_t}\right\vert\md t \nonumber\\ 
    &= C_\lambda \int_{\lambda_0}^{1} \left(\frac{\mathbb{E}\left[\Vert X-\Tilde{X}\Vert^2\right]}{2} + \frac{(\sigma_\pi^2-\sigma^2)^2}{2(\sigma^2/\lambda+(\sigma_\pi^2-\sigma^2))} d\right) \md \lambda \nonumber\\
    &\leq \frac{C_\lambda}{2} \left(\mathbb{E}\left[\Vert X-\Tilde{X}\Vert^2\right] + d\left(\sigma_\pi^2-\sigma^2 + \sigma^2\log\frac{\sigma^2}{\sigma_\pi^2}\right)\right)\nonumber\\ 
    &\leq \frac{C_\lambda}{2} \left(2\mathbb{E}\left[\Vert X\Vert^2\right] + d(3\sigma_\pi^2-\sigma^2)\right),\nonumber 
\end{align}
where in the last line we have used that $\sigma_\pi\geq\sigma$.
Note that by setting $\sigma=\sigma_\pi$, the second term in the penultimate expression cancels out, resulting in
\begin{equation*}
    \mathcal{A}_\lambda(\mu)\leq \frac{C_\lambda}{2}\mathbb{E}\left[\Vert X-\Tilde{X}\Vert^2\right],
\end{equation*}
where we chose $\Tilde{X}$ such that $\mathbb{E}\left[\Vert X-\Tilde{X}\Vert^2\right] = W_2^2(\mathcal{L}(X), \mathcal{L}(\tilde{X}))$ is minimised.

On the other hand, if the schedule satisfies $\max_{t\in[0,T]}\left\vert \frac{\partial_t{\lambda_t}}{\sqrt{\lambda_t(1-\lambda_t)}}\right\vert \leq C_\lambda$, the Wasserstein-2 distance between $\Tilde{\mu}_\lambda$ and $\Tilde{\mu}_{\lambda+\delta}$ is given by
\begin{align*}
     W_2^2(\Tilde{\mu}_\lambda, &\Tilde{\mu}_{\lambda+\delta}) \leq \mathbb{E}\left[\Vert X_\lambda-X_{\lambda + \delta}\Vert^2\right]\\
&=\mathbb{E}\left[\left\Vert(\sqrt{\lambda}-\sqrt{\lambda+\delta})X\right\Vert^2\right] + \mathbb{E}\left[\left\Vert\left(\sqrt{1-\lambda}-\sqrt{1-\lambda-\delta}\right)\sigma Z\right\Vert^2\right]\\
&= (\sqrt{\lambda}-\sqrt{\lambda+\delta})^2 \mathbb{E}\left[\Vert X\Vert^2\right] + \left(\sqrt{1-\lambda}-\sqrt{1-\lambda-\delta}\right)^2\sigma^2d.
\end{align*}
Using the definition of the metric derivative we have
\begin{align*}
    \left\vert\Dot{\Tilde{\mu}}\right\vert_\lambda^2 = \lim_{\delta\to 0}\frac{ W_2^2(\Tilde{\mu}_\lambda, \Tilde{\mu}_{\lambda+\delta})}{\delta^2} \leq \frac{\mathbb{E}\left[\Vert X\Vert^2\right]}{4\lambda} + \frac{\sigma^2d}{4(1-\lambda)}.
\end{align*}
Therefore, we have the following expression for the action
\begin{align}
    \mathcal{A}_{\lambda}(\mu) &= \int_0^T \vert\Dot{\mu}\vert_t^2 \md t = \int_0^T \left\vert\Dot{\Tilde{\mu}}\right\vert_\lambda^2\left\vert\partial_t{\lambda_t}\right\vert^2 \md t\nonumber\\
    &\lesssim \int_0^T \left(\frac{\mathbb{E}\left[\Vert X\Vert^2\right]}{4 \lambda_t} + \frac{\sigma^2}{4(1-\lambda_t)}d\right)\left\vert \partial_t{\lambda_t}\right\vert^2\md t \nonumber\\
    &= \int_0^T \left(\frac{\mathbb{E}\left[\Vert X\Vert^2\right]\sqrt{1-\lambda_t}}{4\sqrt{\lambda_t}} + \frac{\sigma^2\sqrt{\lambda_t}}{4\sqrt{1-\lambda_t}}d\right)\left\vert \frac{\partial_t{\lambda_t}}{\sqrt{\lambda_t(1-\lambda_t)}}\right\vert \left\vert \partial_t{\lambda_t}\right\vert\md t \nonumber\\    
    &\leq C_\lambda \int_0^T \left(\frac{\mathbb{E}\left[\Vert X\Vert^2\right]\sqrt{1-\lambda_t}}{4\sqrt{\lambda_t}} +\frac{\sigma^2\sqrt{\lambda_t}}{4\sqrt{1-\lambda_t}}d\right) \left\vert \partial_t{\lambda_t}\right\vert\md t \nonumber\\ 
    &\leq C_\lambda \int_{0}^{1} \left(\frac{\mathbb{E}\left[\Vert X\Vert^2\right]\sqrt{1-\lambda}}{4\sqrt{\lambda}} + \frac{\sigma^2\sqrt{\lambda}}{4\sqrt{1-\lambda}} d\right) \md \lambda \nonumber\\
    &\leq \frac{C_\lambda\pi}{8} \left(\mathbb{E}\left[\Vert X\Vert^2\right] + \sigma^2 d\right).\nonumber
\end{align}
\end{proof}

%% file: appendix/theorem_discretisation_reparametrised.tex
\begin{proof}
First, consider a modified version of the \gls*{DALMC} algorithm with exact scores, that is,
\begin{equation}\label{eq:auxiliary_algorithm_exact_scores}
    X_{l+1} = X_l + h_l \nabla\log\hat{\mu}_l(X_l) + \sqrt{2 h_l} \xi_l,
\end{equation}
where $h_l > 0$ is the step size, $\xi_k\sim \mathcal{N}(0,I)$, $\hat{\mu}_{t} = \mu_{\kappa t}$, $l\in\{1,\dots, M\}$ and $0=t_0<t_1<\dots<t_M=T/\kappa$ is a discretisation of the interval $[0,T/\kappa]$.
Let $\mathbb{Q}$ be the path measure associated with the continuous-time interpolation of this auxiliary algorithm which corresponds to the SDE 
\begin{equation*}
    \md X_t = \nabla\log \hat{\mu}_{t_-}(X_{t_-})\md t + \sqrt{2}\md B_t,\; t\in[0,T/\kappa],
\end{equation*}
where given a discretisation of the interval $[0,T/\kappa]$, $0=t_0<t_1<\dots<t_M=T/\kappa$, we define $t_-:=t_{l-1}$ when $t\in[t_{l-1}, t_l)$ for $l=1,\dots, M$. On the other hand, let $\mathbb{P}$ be the path measure corresponding to the following reference SDE
\begin{equation*}
    \md Y_t = (\nabla\log \hat{\mu}_t + v_t)(Y_t)\md t + \sqrt{2}\md B_t,\; t\in[0,T/\kappa].
\end{equation*}
The vector field $v = (v_t)_{t\in[0,T/\kappa]}$ is designed such that $Y_t\sim\hat{\mu}_t$ for all $t\in[0, T/\kappa]$. 
Using the Fokker-Planck equation, we have that 
\begin{equation*}
    \partial\hat{\mu}_t = \nabla\cdot\left(\hat{\mu}_t(\nabla\log\hat{\mu}_t + v_t)\right)  + \Delta \hat{\mu}_t = -\nabla\cdot(\hat{\mu}_t v_t), \; t\in[0,T/\kappa].
\end{equation*}
This implies that $v_t$ satisfies the continuity equation and hence generates the curve of probability measures $(\hat{\mu}_t)_t$. Leveraging Lemma \ref{lemma:optimal_vector_field}, we choose $v$ to be the one that minimises the $L^2(\hat{\mu}_t)$ norm, resulting in $\Vert v_t\Vert_{L^2(\hat{\mu}_t)} = \left\vert\Dot{\hat{\mu}}\right\vert_t$ being the metric derivative.
Using the form of Girsanov's theorem given in Lemma \ref{lemma:girsanov_theorem} we have
\begin{align}
    \kl&\left(\mathbb{P}\;||\mathbb{Q}\right)= \frac{1}{4}\int_0^{T/\kappa}\mathbb{E}_{\mathbb{P}}\left[\left\Vert \nabla\log \hat{\mu}_t(X_t)-\nabla \log \hat{\mu}_{t_-}(X_{t_-}) + v_t(X_t)\right\Vert^2\right]\md t\nonumber\\
    \lesssim& \int_0^{T/\kappa}\mathbb{E}_{\mathbb{P}} \left[\left\Vert \nabla\log \hat{\mu}_t(X_t)-\nabla \log \hat{\mu}_t(X_{t_-})\right\Vert^2\right]\md t + \int_0^{T/\kappa} \Vert v_t\Vert^2_{L^2({\hat{\mu}}_t)} \md t \nonumber\\
    &+ \int_0^{T/\kappa}\mathbb{E}_{\mathbb{P}} \left[\left\Vert \nabla\log \hat{\mu}_t(X_{t_-})-\nabla \log \hat{\mu}_{t_-}(X_{t_-})\right\Vert^2\right]\md t \nonumber\\
    \leq&  \sum_{l=1}^M\int_{t_{l-1}}^{t_l} L_{\kappa t}^2 \ \mathbb{E}_{\mathbb{P}} \left[\left\Vert X_t-X_{t_-}\right\Vert^2\right]\md t+ \int_0^{T/\kappa} \Vert v_t\Vert^2_{L^2({\hat{\mu}}_t)} \md t+ \int_0^{T/\kappa}\mathbb{E}_{\mathbb{P}} \left[\left\Vert \nabla\log \frac{\hat{\mu}_t(X_{t_-})}{\hat{\mu}_{t_-}(X_{t_-})}\right\Vert^2\right]\md t, \label{eq:kl_path_bound_intermediate}
\end{align}
where we have used that $\nabla\log\hat{\mu}_t$ is Lipschitz with constant $L_{\kappa t}$.
First, we bound the change in the score function $\mathbb{E}_{\mathbb{P}} \left[\left\Vert \nabla\log \frac{\hat{\mu}_t(X_{t_-})}{\hat{\mu}_{t_-}(X_{t_-})}\right\Vert^2\right]$. Let $t\geq t_{-}$, we can write
\begin{equation*}
    \hat{\mu}_{t_-} = T_{\sqrt{\frac{\lambda_{\kappa t}}{\lambda_{\kappa t_-}}}} \#\hat{\mu}_t * \mathcal{N}\left(0, \left(\sqrt{1-\lambda_{\kappa t_-}}-\sqrt{\frac{(1-\lambda_{\kappa t})\lambda_{\kappa t_-}}{\lambda_{\kappa t}}}\right)^2\sigma^2 I\right),
\end{equation*}
where the pushforward $T_{\alpha}\#$ is defined as $T_{\alpha}\#\mu(x) = \alpha^d \mu(\alpha x)$. Using \citet[Lemma C.12]{lee2022convergence} we have
\begin{align*}
    \left\Vert \nabla\log \frac{\hat{\mu}_t(X_{t_-})}{\hat{\mu}_{t_-}(X_{t_-})}\right\Vert^2 \lesssim L_{\kappa t}^2d \gamma_t + L_{\kappa t}^2\left(\sqrt{\frac{\lambda_{\kappa t}}{\lambda_{\kappa t_-}}}-1\right)^2 \Vert X_{t_-}\Vert^2 + \left(\sqrt{\frac{\lambda_{\kappa t}}{\lambda_{\kappa t_-}}} -1 + L_{\kappa t} \gamma_t\right)^2 \left\Vert \nabla \log\hat{\mu}_t(X_{t_-})\right\Vert^2,
\end{align*}
where
\begin{equation*}
    \gamma_t  = \left(\sqrt{1-\lambda_{\kappa t_-}}-\sqrt{\frac{(1-\lambda_{\kappa t})\lambda_{\kappa t_-}}{\lambda_{\kappa t}}}\right)^2\lesssim 1-\frac{\lambda_{\kappa t_-}}{\lambda_{\kappa t}}.
\end{equation*}
Let $C_\lambda$ introduced in \Cref{assumption:schedule_form}, we have
\begin{align*}
\gamma_t \leq 1-\frac{\lambda_{\kappa t_-}}{\lambda_{\kappa t}}\lesssim C_\lambda h_l, \quad \quad \left(\sqrt{\frac{\lambda_{\kappa t}}{\lambda_{\kappa t_-}}}-1\right)^2 \lesssim C_\lambda^2 h_l^2,\quad\quad
    \left(\sqrt{\frac{\lambda_{\kappa t}}{\lambda_{\kappa t_-}}} -1+ L_{\kappa t} \gamma_t\right)^2\lesssim h_l^2 L_{\kappa t}^2.
\end{align*}
In addition, by choosing an appropriate step size, as will be shown in Corollary~\ref{corollary:complexity_bounds}, we can bound $h_l^2 L_{\kappa t}^2\lesssim 1$.

Given that $X_t = \sqrt{\lambda_t} X + \sqrt{1-\lambda_t} \sigma^2 Z$ for $X_t\sim\hat{\mu}_t$, we derive the following moment bound
\begin{align*}
    \mathbb{E}_{\mathbb{P}}\left[\left\Vert X_{t_-}\right\Vert^2\right] =& \mathbb{E}_{\mathbb{P}}\left[\left\Vert \sqrt{\lambda_{\kappa t_-}} X + \sqrt{1-\lambda_{\kappa t_-}} Z\right\Vert^2\right] = \lambda_{\kappa t_-} \mathbb{E}_{\pid}\left[\left\Vert X\right\Vert^2\right] + (1-\lambda_{\kappa t_-})\sigma^2 d\lesssim \mathbb{E}_{\pid}\left[\left\Vert X\right\Vert^2\right] + d.
\end{align*}
To bound $\mathbb{E}_{\mathbb{P}}\left[\left\Vert \nabla\log\hat{\mu}_t(X_{t_-})\right\Vert^2\right]$, recall from \eqref{eq:hessian_final_bound_1} that $\mu_t\propto e^{-U_t}$ satisfies $\nabla^2 U_t \preccurlyeq  L_t I$. 
Therefore, using \citet[Lemma 4.0.1]{sinho_book} it holds that 
\begin{align*}
    \mathbb{E}_{\mathbb{P}}\left[\left\Vert \nabla\log\hat{\mu}_t(X_{t_-})\right\Vert^2\right] \leq & \mathbb{E}_{\mathbb{P}}\left[\left\Vert \nabla\log\hat{\mu}_t(X_{t})\right\Vert^2\right] + \mathbb{E}_{\mathbb{P}}\left[\left\Vert \nabla\log\hat{\mu}_t(X_{t})-\nabla\log\hat{\mu}_t(X_{t_-})\right\Vert^2\right]\\
    \leq & L_{\kappa t} d  + L_{\kappa t}^2\mathbb{E}_{\mathbb{P}} \left[\left\Vert X_t-X_{t_-}\right\Vert^2\right]. 
\end{align*}
This implies that
\begin{align*}
    \mathbb{E}_{\mathbb{P}} \left[\left\Vert \nabla\log \frac{\hat{\mu}_t(X_{t_-})}{\hat{\mu}_{t_-}(X_{t_-})}\right\Vert^2\right]\lesssim &d h_lL_{\kappa t}^2 + h_l^2L_{\kappa t}^2 \left(\mathbb{E}_{\pid}\left[\left\Vert X\right\Vert^2\right] + d\right) +  h_l^2L_{\kappa t}^2 \left(L_{\kappa t} d 
    + L_{\kappa t}^2\mathbb{E}_{\mathbb{P}} \left[\left\Vert X_t-X_{t_-}\right\Vert^2\right]\right).
\end{align*}
Substituting this expression into \eqref{eq:kl_path_bound_intermediate}, we have
\begin{align*}
    \kl\left(\mathbb{P}\;||\mathbb{Q}\right)\lesssim& \sum_{l=1}^M\int_{t_{l-1}}^{t_l} L_{\kappa t}^2 \ \mathbb{E}_{\mathbb{P}} \left[\left\Vert X_t-X_{t_-}\right\Vert^2\right]\md t+ \int_0^{T/\kappa} \Vert v_t\Vert^2_{L^2({\hat{\mu}}_t)} \md t\\
    &+ \sum_{l=1}^M\int_{t_{l-1}}^{t_l} \left(d h_l L_{\kappa t}^2 + h_l^2 L_{\kappa t}^2\left(\mathbb{E}_{\pid}\left[\left\Vert X\right\Vert^2\right] + d\right) +  dh_l^2 L_{\kappa t}^3\right)\md t.
\end{align*}
To bound $\mathbb{E}_{\mathbb{P}} \left[\left\Vert X_t-X_{t_-}\right\Vert^2\right]$ we note that under $\mathbb{P}$, for $t\in [t_{l-1}, t_l)$, we have
\begin{equation*}
    X_t - X_{t_-} = \int_{t_{l-1}}^t (\nabla\log {\hat{\mu}}_\tau + v_\tau) (X_\tau) \md \tau + \sqrt{2}(B_t - B_{t_{l-1}}).
\end{equation*}
Therefore, 
\begin{align*}
    \mathbb{E}_{\mathbb{P}} \left[\left\Vert X_t-X_{t_-}\right\Vert^2\right] &\lesssim \mathbb{E}_{\mathbb{P}}\left[\left\Vert\int_{t_{l-1}}^t (\nabla\log {\hat{\mu}}_\tau + v_\tau) (X_\tau) \md \tau\right\Vert^2\right]  + \mathbb{E}_{\mathbb{P}} \left[\left\Vert \sqrt{2}(B_t - B_{t_{l-1}})\right\Vert^2\right]\\
    &\lesssim (t-t_{l-1}) \int_{t_{l-1}}^t \left(\left\Vert\nabla\log {\hat{\mu}}_\tau \right\Vert_{L^2({\hat{\mu}}_\tau)}^2 + \left\Vert v_\tau \right\Vert_{L^2({\hat{\mu}}_\tau)}^2\right)\md \tau  + d(t-t_{l-1})\\
    &\lesssim h_l \int_{t_{l-1}}^{t_l} \left(\left\Vert\nabla\log {\hat{\mu}}_\tau \right\Vert_{L^2({\hat{\mu}}_\tau)}^2 + \left\Vert v_\tau \right\Vert_{L^2({\hat{\mu}}_\tau)}^2\right)\md \tau  + d h_l,
\end{align*}
where the second inequality arises from the application of the Cauchy-Schwarz inequality, and the last inequality is due to the definition $h_l = t_l - t_{l-1}$.
Taking the integral over $t\in[t_{l-1}, t_l]$, it follows
\begin{align*}
    \int_{t_{l-1}}^{t_l}{L}_{\kappa t}^2 \ \mathbb{E}_{\mathbb{P}} \left[\left\Vert X_t-X_{t_-}\right\Vert^2\right] \md t \lesssim &  \left(\int_{t_{l-1}}^{t_l}{L}_{\kappa t}^2 \ \md t\right) \left(h_l \int_{t_{l-1}}^{t_l} \left(\left\Vert\nabla\log {\hat{\mu}}_t \right\Vert_{L^2({\hat{\mu}}_t)}^2 + \left\Vert v_t \right\Vert_{L^2({\hat{\mu}}_t)}^2\right)\md t  + d h_l\right).
\end{align*}
Putting this together we have
\begin{align}
&\sum_{l=1}^M\int_{t_{l-1}}^{t_l} L_{\kappa t}^2 \ \mathbb{E}_{\mathbb{P}} \left[\left\Vert X_t-X_{t_-}\right\Vert^2\right]\md t+ \int_0^{T/\kappa} \Vert v_t\Vert^2_{L^2({\hat{\mu}}_t)} \md t \nonumber\\
    &\lesssim \sum_{l=1}^M \left(\int_{t_{l-1}}^{t_l} L_{\kappa t}^2 \ \md t\right) \left(h_l \int_{t_{l-1}}^{t_l} \left(\left\Vert\nabla\log {\hat{\mu}}_t \right\Vert_{L^2({\hat{\mu}}_t)}^2 + \left\Vert v_t \right\Vert_{L^2({\hat{\mu}}_t)}^2\right)\md t  + d h_l\right)+ \int_0^{T/\kappa} \Vert v_t\Vert^2_{L^2({\hat{\mu}}_t)} \md t \nonumber\\
   &\lesssim \sum_{l=1}^M \left(\int_{t_{l-1}}^{t_l}L_{\kappa t}^2 \ \md t\right) \left(h_l \int_{t_{l-1}}^{t_l} \left(d L_{\kappa t}\ + \left\Vert v_t \right\Vert_{L^2({\hat{\mu}}_t)}^2\right)\md t  + d h_l\right)+ \int_0^{T/\kappa} \Vert v_t\Vert^2_{L^2({\hat{\mu}}_t)} \md t \nonumber\\
   &\lesssim \sum_{l=1}^M \left( 1 + h_l^2\max_{[t_{l-1}, t_l]}L_{t}^2\right)
   \int_{t_{l-1}}^{t_l} \left\vert\Dot{\hat{\mu}}\right\vert_t^2\ \md t + \left(d h_l\int_{t_{l-1}}^{t_l}L_{\kappa t}^2 \ \md t\right) \left(1 + h_l\max_{[t_{l-1}, t_l]}L_{t}\right). \label{eq:intermediate_bound_discretisation_paths}
\end{align}
This results in the following bound for the \gls*{KL} divergence between $\mathbb{P}$ and $\mathbb{Q}$:
\begin{align*}
    \kl\left(\mathbb{P}\;||\mathbb{Q}\right)\lesssim& \sum_{l=1}^M \left( 1 + h_l^2\max_{[t_{l-1}, t_l]} L_{t}^2\right)
   \int_{t_{l-1}}^{t_l} \left\vert\Dot{\hat{\mu}}\right\vert_t^2\ \md t + \left(d h_l\int_{t_{l-1}}^{t_l} L_{\kappa t}^2 \ \md t\right) \left(1 + h_l\max_{[t_{l-1}, t_l]} L_{t}\right) \\
   & + \sum_{l=1}^M\int_{t_{l-1}}^{t_l} \left(d h_l L_{\kappa t}^2 + h_l^2 L_{\kappa t}^2 \left(\mathbb{E}_{\pid}\left[\left\Vert X\right\Vert^2\right] + d\right) + dh_l^2  L_{\kappa t}^3\right)\md t
\end{align*}
Note that intuitively we want to take smaller steps $h_l= t_l-t_{l-1}$ when $L_t$ is larger. 
Define $h = \max_{l\in\{1, \dots, M\}} h_l$, we can further simplify the previous expression to obtain
\begin{align*}
    \kl\left(\mathbb{P}\;||\mathbb{Q}\right)\lesssim &\sum_{l=1}^M (1+h^2 L_{\max}^2)\int_{t_{l-1}}^{t_l} \left\vert\Dot{\hat{\mu}}\right\vert_t^2 \ \md t\ + d h_l(1+ h L_{\max}) \int_{t_{l-1}}^{t_l}L_{\kappa t}^2 \ \md t+ \frac{Th^2L_{\max}^2}{\kappa}\left(\mathbb{E}_{\pid}\left[\left\Vert X\right\Vert^2\right] + d\right)\\
     =& (1+h^2 L_{\max}^2) \int_{0}^{T/\kappa} \left\vert\Dot{\hat{\mu}}\right\vert_t^2 \ \md t\ + 
 d h(1+ h L_{\max})\int_{0}^{T/\kappa} L_{\kappa t}^2 \ \md t+ \frac{Th^2 L_{\max}^2}{\kappa}\left(\mathbb{E}_{\pid}\left[\left\Vert X\right\Vert^2\right] + d\right)\\
 =&  (1+h^2 L_{\max}^2) \kappa \mathcal{A}_{\lambda}(\mu)  + 
\frac{d h}{\kappa} (1+ hL_{\max})\int_{0}^{T}L_{ t}^2 \ \md t+ \frac{Th^2L_{\max}^2}{\kappa} \left(\mathbb{E}_{\pid}\left[\left\Vert X\right\Vert^2\right] + d\right).
\end{align*}
The step size $h$ can be expressed in terms of the number of steps $M$ and $\kappa$ as $h\asymp \frac{1}{M\kappa}$. Therefore, we have
\begin{align*}
    \kl\left(\mathbb{P}\;||\mathbb{Q}\right)&\lesssim\left(1+\frac{L_{\max}^2}{M^2\kappa^2}\right) \kappa \mathcal{A}_\lambda(\mu) + \frac{d}{M\kappa^2}\left(1+  \frac{L_{\max}}{M\kappa}\right)\int_{0}^{T}L_{ t}^2 \ \md t+ \frac{L_{\max}^2}{M^2\kappa^3} \left(\mathbb{E}_{\pid}\left[\left\Vert X\right\Vert^2\right] + d\right)\\
    &\lesssim \left(1+\frac{L_{\max}^2}{M^2\kappa^4}\right) \kappa C_\lambda \left(\mathbb{E}_{\pid}\left[\Vert X\Vert^2\right] +  d\right) + \frac{d}{M\kappa^2}\left(1+  \frac{L_{\max}}{M\kappa}\right)\int_{0}^{T}L_{ t}^2 \ \md t,
\end{align*}
where we have used the bound on the action derived in Lemma \ref{lemma:action_bound} and $T = \mathcal{O}(1)$.

To derive the previous bound, we have assumed that the score of the intermediate distributions $\nabla\log\hat{\mu}_t$, can be computed exactly. In practice, however, we use an approximation, introducing an additional error term into the analysis. Let $s_\theta(\cdot, t)$ denote our estimator for $\nabla\log\hat{\mu}_t$  and let $\mathbb{Q}_\theta$ be the path measure of the continuous-time interpolation of the \gls*{DALMC} algorithm \eqref{eq:annealed_langevin_mcmc_algorithm_score_approx}. We conclude that
\begin{align*}
    \kl\left(\mathbb{P}\;||\mathbb{Q}_\theta\right)=& \frac{1}{4}\int_0^{T/\kappa}\mathbb{E}_{\mathbb{P}}\left[\left\Vert \nabla\log \hat{\mu}_t(X_t)- s_\theta(X_{t_-}, t_-) + v_t(X_t)\right\Vert^2\right]\md t\\
    \lesssim& \int_0^{T/\kappa}\mathbb{E}_{\mathbb{P}}\left[\left\Vert \nabla\log \hat{\mu}_t(X_t)-\nabla \log \hat{\mu}_{t_-}(X_{t_-}) + v_t(X_t)\right\Vert^2\right]\md t \\
    &+ \int_0^{T/\kappa}\mathbb{E}_{\mathbb{P}}\left[\left\Vert \nabla \log \hat{\mu}_{t_-}(X_{t_-}) - s_\theta(X_{t_-}, t_-)\right\Vert^2\right]\md t\\
    \lesssim& \left(1+\frac{L_{\max}^2}{M^2\kappa^4}\right) \kappa C_\lambda \left(\mathbb{E}_{\pid}\left[\Vert X\Vert^2\right] + d\right) + \frac{d}{M\kappa^2}\left(1+  \frac{L_{\max}}{M\kappa}\right)\int_{0}^{T}L_{ t}^2 \ \md t \\
    &+\sum_{l=0}^{M-1} h_l\mathbb{E}_{\hat{\mu}_t}\left[\left\Vert \nabla \log \hat{\mu}_l(X_{t_l}) - s_\theta(X_{t_l}, t_l)\right\Vert^2\right]\\
    \lesssim& \;\left(1+\frac{L_{\max}^2}{M^2\kappa^4}\right) \kappa C_\lambda \left(\mathbb{E}_{\pid}\left[\Vert X\Vert^2\right] + d\right) + \frac{d}{M\kappa^2}\left(1+  \frac{L_{\max}}{M\kappa}\right)\int_{0}^{T}L_{ t}^2 \ \md t + \varepsilon_{\text{score}}^2\\
    \lesssim& \;\left(1+\frac{L_{\max}^2}{M^2\kappa^4}\right) \kappa (M_2\vee d) + \frac{d}{M\kappa^2}\left(1+  \frac{L_{\max}}{M\kappa}\right)L_{\max}^2 + \varepsilon_{\text{score}}^2,
\end{align*}
where we have used the control of the score approximation given in assumption \Cref{assumption:score_approximation}.
\end{proof}

%% file: appendix/corollary_discretisation_score_approx.tex
\begin{proof}
Based on the \gls*{KL} bound established in Theorem~\ref{theorem:discretisation_analysis_convolutional_path}, we can obtain the iteration complexity of the \gls*{DALMC} algorithm \eqref{eq:annealed_langevin_mcmc_algorithm_score_approx}. Observe that by selecting 
\begin{equation*}
    \kappa = \mathcal{O}\left(\frac{\varepsilon_{\text{score}}^2}{M_2\vee d}\right),\quad M = \mathcal{O}\left(\frac{d(M_2\vee d)^2L_{\max}^2}{\varepsilon_{\text{score}}^6}\right),
\end{equation*}
it follows that $\kl\left(\mathbb{P}\;||\mathbb{Q}_\theta\right)\leq \varepsilon_{\text{score}}^2$. Therefore, for any $\varepsilon = \mathcal{O}(\varepsilon_{\score})$, the \gls*{DALMC} algorithm requires at most 
\begin{equation*}
     M = \mathcal{O}\left(\frac{d (M_2 \vee d)^2 L_{\max}^2}{\varepsilon^6}\right)
\end{equation*}
steps to approximate $\pid$ to within $\varepsilon^2$ in \gls*{KL} divergence.
\end{proof}

%% file: appendix/comments_new_less_restrictiv_assumption.tex
    We show below that both assumptions \Cref{assumption:grad_log_lipschitzness_convexity_outside_of_a_ball} and \Cref{assumption:grad_log_lipschitzness_hessian_decay} independently imply \Cref{assumption:pi_data_conditions_less_restrictive}. 

\textbf{\Cref{assumption:grad_log_lipschitzness_convexity_outside_of_a_ball} $\Rightarrow$ \Cref{assumption:pi_data_conditions_less_restrictive}}: First, recall that we proved in Lemma~\ref{lemma:assumption_implies_LSI} that if $\pid$ satisfies \Cref{assumption:grad_log_lipschitzness_convexity_outside_of_a_ball}, then it has finite log-Sobolev and Poincaré constants, thereby ensuring $\mathbb{E}_{\pid} \Vert X\Vert^2$ is finite. Besides, observe that 
\begin{equation*}
    \mathbb{E}_{\pid} \Vert \nabla V_\pi(X)\Vert^8 \lesssim \Vert \nabla V_\pi(0) \Vert^8 +  L_\pi^8\mathbb{E}_{\pid} \Vert X\Vert^8.
\end{equation*}
Using Poincaré inequality, it follows
\begin{align*}
    \mathbb{E}_{\pid} \Vert X\Vert^4 &= \text{Var}_{\pid} \left(\Vert X\Vert^2\right) + (\mathbb{E}_{\pid} \Vert X\Vert^2)^2 \lesssim C_{\text{PI}}\ \mathbb{E}_{\pid}\Vert X\Vert^2 + (\mathbb{E}_{\pid} \Vert X\Vert^2)^2,\\
    \mathbb{E}_{\pid} \Vert X\Vert^6 &= \text{Var}_{\pid} \left(\Vert X\Vert^3\right) + (\mathbb{E}_{\pid} \Vert X\Vert^3)^2 \lesssim C_{\text{PI}}\ \mathbb{E}_{\pid}\Vert X\Vert^4 + (\mathbb{E}_{\pid} \Vert X\Vert^4)^{3/2},\\
    \mathbb{E}_{\pid} \Vert X\Vert^8 &= \text{Var}_{\pid} \left(\Vert X\Vert^4\right) + (\mathbb{E}_{\pid} \Vert X\Vert^4)^2 \lesssim C_{\text{PI}}\ \mathbb{E}_{\pid}\Vert X\Vert^6 + (\mathbb{E}_{\pid} \Vert X\Vert^4)^2,
\end{align*}
which demonstrates that \Cref{assumption:grad_log_lipschitzness_convexity_outside_of_a_ball} implies \Cref{assumption:grad_log_lipschitzness_hessian_decay}.

\textbf{\Cref{assumption:grad_log_lipschitzness_hessian_decay} $\Rightarrow$ \Cref{assumption:pi_data_conditions_less_restrictive}}: If $\pid$ satisfies \Cref{assumption:grad_log_lipschitzness_hessian_decay} we have that $\nabla^2 V_{\pid}$ has asymptotic order $\mathcal{O}\left(\Vert x\Vert^{-2}I\right)$, hence by \citet{RefWorks:RefID:90-1973infinitesimal} $\Vert \nabla V_{\pid}\Vert$ has asymptotic order $\mathcal{O}\left(\Vert x\Vert^{-1}\right)$. Therefore, $\mathbb{E}_{\pid} \Vert \nabla V_{\pi} (X)\Vert^8$ is guaranteed to be bounded, which concludes that \Cref{assumption:grad_log_lipschitzness_hessian_decay} implies \Cref{assumption:pi_data_conditions_less_restrictive}.

%% file: appendix/additional_proof_less_assumption.tex
\begin{proof}
    Let $\mathbb{Q}$ denote the path measure associated with the continuous-time interpolation of the modified \gls*{DALMC} algorithm  with exact scores \eqref{eq:auxiliary_algorithm_exact_scores}, which corresponds to the SDE 
\begin{equation*}
    \md X_t = \nabla\log \hat{\mu}_{t_-}(X_{t_-})\md t + \sqrt{2}\md B_t,\; t\in\left[0,T/\kappa\right],
\end{equation*}
where given a discretisation of the interval $\left[0,T/\kappa\right]$, $0=t_0<t_1<\dots<t_M=T/\kappa$, we define $t_-:=t_{l-1}$ when $t\in [t_{l-1}, t_l)$ for $l=1,\dots, M$. On the other hand, let $\mathbb{P}$ be the path measure corresponding to the following reference SDE
\begin{equation*}
    \md X_t = (\nabla\log \hat{\mu}_t + v_t)(X_t)\md t + \sqrt{2}\md B_t,\; t\in\left[0,T/\kappa\right].
\end{equation*}
As in the proof of Theorem~\ref{theorem:discretisation_analysis_convolutional_path}, the vector field $v = (v_t)_{t\in\left[0,T/\kappa\right]}$ is designed such that $X_t\sim\hat{\mu}_t$ for all $t\in\left[0, T/\kappa\right]$ and $\Vert v_t\Vert_{L^2(\hat{\mu}_t)} = \left\vert\Dot{\hat{\mu}}\right\vert_t$.
Using Girsanov's theorem we have
\begin{align}
    \kl\left(\mathbb{P}\;||\mathbb{Q}\right)=& \frac{1}{4}\int_0^{T/\kappa}\mathbb{E}_{\mathbb{P}}\left[\left\Vert \nabla\log \hat{\mu}_t(X_t)-\nabla \log \hat{\mu}_{t_-}(X_{t_-}) + v_t(X_t)\right\Vert^2\right]\md t\nonumber\\
    \lesssim& \int_0^{T/\kappa}\mathbb{E}_{\mathbb{P}} \left[\left\Vert \nabla\log \hat{\mu}_t(X_t)-\nabla \log \hat{\mu}_{t_-}(X_{t_-})\right\Vert^2\right]\md t + \int_0^{T/\kappa} \Vert v_t\Vert^2_{L^2({\hat{\mu}}_t)} \md t. \nonumber
\end{align}
Note that the first term involves both a time and space discretisation error. 
Inspired by \citet{Chen2022ImprovedAO}, we start analysing
\begin{equation*}
    E_{t, s} = \mathbb{E}_{\mathbb{P}}\left[\left\Vert \nabla\log \hat{\mu}_t(X_t) - \nabla\log \hat{\mu}_s(X_s)\right\Vert^2\right],
\end{equation*}
with $0\leq s< t\leq 1$. Recall that we can write
\begin{equation*}
    X_s = \sqrt{\frac{\lambda_{\kappa s}}{\lambda_{\kappa t}}} X_t + \sigma \left(\sqrt{1-\lambda_{\kappa s}} - \sqrt{\frac{(1-\lambda_{\kappa t})\lambda_{\kappa s}}{\lambda_{\kappa t}}}\right) Z = \alpha_{t, s} X_t + \Tilde{\sigma}_{t,s} Z,
\end{equation*}
where $Z\sim \mathcal{N}(0, I)$ independent of $X_t$. So, we have
\begin{equation*}
    \hat{\mu}_s(x) \propto \int \alpha_{t, s}^{-d} \ \hat{\mu}_t\left(\frac{y}{\alpha_{t,s}}\right) e^{-\frac{\Vert x-y\Vert^2}{2\Tilde{\sigma}_{t,s}}} \md y = \int \hat{\mu}_t(y)e^{-\frac{\Vert x-\alpha_{t, s}y\Vert^2}{2\Tilde{\sigma}_{t,s}}} \md y.
\end{equation*}
Therefore, similarly to the proof of Lemma  \ref{lem:regularity_of_gaussian_path}, we can express the score of $\hat{\mu}_s$ in terms of that of $\hat{\mu}_t$,
\begin{equation*}
    \nabla\log\hat{\mu}_s (x) = \alpha_{t,s}^{-1} \ \mathbb{E}_{Y\sim\varphi_{t| x, s}} \left[\nabla \log \hat{\mu}_t(Y)| x, s\right],
\end{equation*}
where $\varphi_{t| x,s}(y)\propto \hat{\mu}_t(y)e^{-\frac{\Vert x - \alpha_{t,s}y\Vert^2}{2\Tilde{\sigma}_{t,s}^2}}$.
Substituting this we have
\begin{align}
    E_{t, s} \leq& 2 \mathbb{E}_{\mathbb{P}}\left[\left\Vert \nabla\log \hat{\mu}_s(X_s) - \nabla\log \hat{\mu}_t(\alpha_{t, s}^{-1}X_s)\right\Vert^2\right] + 2 \mathbb{E}_{\mathbb{P}}\left[\left\Vert \nabla\log \hat{\mu}_t(X_t) - \nabla\log \hat{\mu}_t(\alpha_{t, s}^{-1}X_s)\right\Vert^2\right] \nonumber \\
    =& 2 \mathbb{E}_{\mathbb{P}}\left[\left \Vert \mathbb{E}_{\varphi_{t|X_s,s}}\left[(\alpha_{t,s}^{-1}-1)\nabla \log \hat{\mu}_t(Y) + \nabla\log \hat{\mu}_t(Y) - \nabla\log\hat{\mu}_{t}(\alpha_{t,s}^{-1}X_s)\right]\right\Vert^2\right] \nonumber\\
    & +  2 \mathbb{E}_{\mathbb{P}}\left[\left\Vert \nabla\log \hat{\mu}_t(X_t) - \nabla\log \hat{\mu}_t(\alpha_{t, s}^{-1}X_s)\right\Vert^2\right] \nonumber\\
    \leq & 4(\alpha_{t,s}^{-1} -1)^2 \ \mathbb{E} \left[\Vert \nabla\log \hat{\mu}_t(X_t)\Vert^2\right] + 6 \ \mathbb{E}_{\mathbb{P}}\left[\left\Vert \nabla\log \hat{\mu}_t(X_t) - \nabla\log \hat{\mu}_t(\alpha_{t, s}^{-1}X_s)\right\Vert^2\right]. \label{eq:auxiliary_time_space_discretisation_error}
\end{align}
Using the score expressions provided in Lemma \ref{lem:regularity_of_gaussian_path} Eq. \eqref{eq:score_conv_expressions}, we can bound $\mathbb{E} \left[\Vert \nabla\log \hat{\mu}_t(X_t)\Vert^2\right]$ as follows
\begin{align*}
    \mathbb{E}_{\hat{\mu}_t}\left[\Vert \nabla\log 
 \hat{\mu}_t(X_t)\Vert^2\right] \leq& \frac{1}{\sigma^2(1-\lambda_{\kappa t})} \mathbb{E}_{\hat{\mu}_t}\mathbb{E}_{\rho_{t, X_t}}\left[\left\Vert\frac{X_t - Y}{\sigma\sqrt{1-\lambda_{\kappa t}}}\right\Vert^2\right] = \frac{d}{\sigma^2(1-\lambda_{\kappa t})},  \\
\mathbb{E}_{\hat{\mu}_t}\left[\Vert \nabla\log 
 \hat{\mu}_t(X_t)\Vert^2\right] \leq & \frac{1}{\lambda_{\kappa t}} \mathbb{E}_{\hat{\mu}_t}\mathbb{E}_{\rho_{t, X_t}}\left[\left\Vert\nabla V_{\pi}\left(\frac{Y}{\sqrt{\lambda_{\kappa t}}}\right)\right\Vert^2\right] = \frac{1}{\lambda_{\kappa t}} \mathbb{E}_{\pid}\left[\left\Vert\nabla V_{\pi}\left(Y\right)\right\Vert^2\right] \leq \frac{L_\pi}{\lambda_{\kappa t}} d,
\end{align*}
where we have used the Lipschitzness of $\pid$.
This provides, $\mathbb{E} \left[\Vert \nabla\log \hat{\mu}_t(X_t)\Vert^2\right]\leq \min\left\{\frac{d}{\sigma^2(1-\lambda_{\kappa t})}, \frac{L_\pi d}{\lambda_{\kappa t}}\right\}$.
Following a similar argument to \citet[Lemma 13]{Chen2022ImprovedAO},  we  study the second term in \eqref{eq:auxiliary_time_space_discretisation_error},
\begin{align*}
    \mathbb{E}\left[\left\Vert \nabla\log \hat{\mu}_t(X_t) - \nabla\log \hat{\mu}_t(\alpha_{t, s}^{-1}X_s)\right\Vert^2\right] =& \mathbb{E}\left[\left\Vert \int_{0}^1 \nabla^2\log \hat{\mu}_t(X_t + aZ_{t,s})Z_{t,s} \md a\right\Vert^2\right] \\
    \leq & \int_{0}^1 \mathbb{E}\left[\left\Vert  \nabla^2\log \hat{\mu}_t(X_t + aZ_{t,s})Z_{t,s} \right\Vert^2\right] \md a,
\end{align*}
where $Z_{t,s} = \alpha_{t,s}^{-1} X_s - X_t\sim \mathcal{N}\left(0, \left(\frac{\Tilde{\sigma}_{t,s}}{\alpha_{t,s}}\right)^2 I\right)$ is independent of $X_t$. For simplicity, we denote
\begin{equation*}
    \hat{\sigma}_{t,s} = \frac{\Tilde{\sigma}_{t,s}}{\alpha_{t,s}} =\sigma\left( \sqrt{\frac{(1-\lambda_{\kappa s})\lambda_{\kappa t}}{\lambda_{\kappa s}}} - \sqrt{1-\lambda_{\kappa t}}\right).
\end{equation*}
To bound $\mathbb{E}\left[\left\Vert  \nabla^2\log \hat{\mu}_t(X_t + aZ_{t,s})Z_{t,s} \right\Vert^2\right]$ we consider the following change of variable
\begin{align*}
    \mathbb{E}\left[\left\Vert  \nabla^2\log \hat{\mu}_t(X_t + aZ_{t,s})Z_{t,s} \right\Vert^2\right] &= \mathbb{E} \left[\left\Vert  \nabla^2\log \hat{\mu}_t(X_t)Z_{t,s} \right\Vert^2 \frac{\md P_{X_t + a Z_{t,s}, Z_{t,s}}(X_t, Z_{t,s})}{\md P_{X_t, Z_{t,s}}(X_t, Z_{t,s})}\right] \\
    &\lesssim \left(\mathbb{E}\left\Vert  \nabla^2\log \hat{\mu}_t(X_t)Z_{t,s} \right\Vert^4\mathbb{E}\left(\frac{\md P_{X_t + a Z_{t,s}, Z_{t,s}}(X_t, Z_{t,s})}{\md P_{X_t, Z_{t,s}}(X_t, Z_{t,s})}\right)^2\right)^{1/2}.
\end{align*}
Let $M_t = \nabla^2\log \hat{\mu}_t(X_t)\left(\nabla^2\log \hat{\mu}_t(X_t)\right)^\intercal$ and $N_{t,s} = Z_{t,s}Z_{t,s}^\intercal$, which by definition they are independent. We now need to bound the two previous factors. By the properties of the tensor product, we have
\begin{equation*}
    \mathbb{E}\left\Vert  \nabla^2\log \hat{\mu}_t(X_t )Z_{t,s} \right\Vert^4 = \mathbb{E}\left[\text{Tr}\left(M_t^\intercal Z_{t,s}\right)^2\right] = \left\langle \mathbb{E} M_t \otimes M_t, \mathbb{E} Z_{t,s}\otimes Z_{t,s}\right\rangle.
\end{equation*}
Using the properties of the $\chi^2$ distribution, we obtain
\begin{equation*}
    \mathbb{E} (Z_{t,s}\otimes Z_{t,s})_{i_1, i_2, i_3, i_4} = \begin{cases}
        3 \Hat{\sigma}_{t,s}^4, \; i_1=i_2=i_3=i_4, \\
        \Hat{\sigma}_{t,s}^4\; i_1 \neq i_2, (i_1, i_2) = (i_3, i_4) \;\text{or}\; (i_1, i_2) = (i_4, i_3),\\
        0, \;\text{otherwise}.
    \end{cases}
\end{equation*}
Substituting this we have
\begin{align*}
    \mathbb{E}&\left\Vert  \nabla^2\log \hat{\mu}_t(X_t )Z_{t,s} \right\Vert^4 \lesssim \Hat{\sigma}_{t,s}^4 \left(\sum_{ (i_1, i_2) = (i_3, i_4)} + \sum_{ (i_1, i_2) = (i_4, i_3)}\right) \mathbb{E}(M_t\otimes M_t)_{i_1, i_2, i_3, i_4}\\
    &\lesssim \Hat{\sigma}_{t,s}^4 \sum_{ (i_1, i_2) = (i_3, i_4)} \mathbb{E}(M_t\otimes M_t)_{i_1, i_2, i_3, i_4}\lesssim \Hat{\sigma}_{t,s}^4  \mathbb{E}\Vert M_t\Vert_F^2 \lesssim \Hat{\sigma}_{t,s}^4 \mathbb{E}\Vert \nabla^2\log \hat{\mu}_t(X_t)\Vert_F^4.
\end{align*}
Applying Lemma 12 from \citet{Chen2022ImprovedAO}, it follows that
\begin{equation*}
    \mathbb{E}\Vert \nabla^2\log \hat{\mu}_t(X_t)\Vert_F^4 \lesssim \left(\frac{d}{\sigma^2(1-\lambda_{\kappa t})}\right)^4.
\end{equation*}
This bound becomes arbitrarily large as $\lambda_{\kappa t}$ tends to $1$, however, using the alternative expression for the Hessian provided in \eqref{eq:hessian_conv_2}, we can write
\begin{align*}
    \nabla^2 \log \hat{\mu}_t(X_t) &=-\frac{1}{\lambda_{\kappa t}}\left(\mathbb{E}_{\rho_{t, X_t}}\left[\nabla^2 V_\pi\left(\frac{Y}{\sqrt{\lambda_{\kappa t}}}\right)\right] - \text{Cov}_{\rho_{t, X_t}}\left[\nabla V_\pi\left(\frac{Y}{\sqrt{\lambda_{\kappa t}}}\right)\right]\right),
\end{align*}
where $\rho_{t, X_t}(y)\propto e^{-V_\pi(y/\sqrt{\lambda_{\kappa t}})-\Vert x-y\Vert^2/(2\sigma^2(1-\lambda_{\kappa t}))}$.
Due to the Lipschitzness of $\pid$, we have
\begin{equation*}
    - L_\pi I \preccurlyeq\mathbb{E}_{\rho_{t, X_t}}\left[\nabla^2 V_\pi\left(\frac{Y}{\sqrt{\lambda_{\kappa t}}}\right)\right]\preccurlyeq L_\pi I,
\end{equation*}
where the Frobenius norm of the identity matrix is $\Vert I\Vert_F=\sqrt{d}$.
For the covariance term we proceed as follows
\begin{align*}
    &\frac{1}{\lambda_{\kappa t}^4}\mathbb{E}_{\hat{\mu}_t} \left\Vert \text{Cov}_{\rho_{t, X_t}}\left[\nabla V_\pi\left(\frac{Y}{\sqrt{\lambda_{\kappa t}}}\right)\right]\right\Vert_F^4 \leq \frac{1}{\lambda_{\kappa t}^4} \mathbb{E}_{\hat{\mu}_t}\mathbb{E}_{\rho_{t,X_t}} \left\Vert\nabla V_\pi\left(\frac{Y}{\sqrt{\lambda_{\kappa t}}}\right) \nabla V_\pi\left(\frac{Y}{\sqrt{\lambda_{\kappa t}}}\right)^\intercal\right\Vert_F^4\\
    &= \frac{1}{\lambda_{\kappa t}^4} \mathbb{E}_{T_{\lambda_{\kappa t}^{-1/2}}\#\pid} \left\Vert\nabla V_\pi\left(\frac{Y}{\sqrt{\lambda_{\kappa t}}}\right) \nabla V_\pi\left(\frac{Y}{\sqrt{\lambda_{\kappa t}}}\right)^\intercal\right\Vert_F^4 = \frac{1}{\lambda_{\kappa t}^4}  \mathbb{E}_{T_{\lambda_{\kappa t}^{-1/2}}\#\pid} \left\Vert\nabla V_\pi\left(\frac{Y}{\sqrt{\lambda_{\kappa t}}}\right)\right\Vert^{8} \\
    & =\frac{1}{\lambda_{\kappa t}^4}  \mathbb{E}_{\pid} \left\Vert\nabla V_\pi\left(Y\right)\right\Vert^{8} \leq \frac{K_\pi^2}{\lambda_{\kappa t}^4}.
\end{align*}
Therefore, we have
\begin{equation*}
    \mathbb{E}\Vert \nabla^2\log \hat{\mu}_t(X_t)\Vert_F^4 \lesssim \min \left\{\frac{d^4}{\sigma^8(1-\lambda_{\kappa t})^4}, \frac{L_\pi^4 d^2 + K_\pi^2}{\lambda_{\kappa t}^4}\right\}.
\end{equation*}
Next, we bound the term concerning the change of variable
\begin{align*}
    \mathbb{E}\left(\frac{\md P_{X_t + a Z_{t,s}, Z_{t,s}}(X_t, Z_{t,s})}{\md P_{X_t, Z_{t,s}}(X_t, Z_{t,s})}\right)^2 =& \mathbb{E}\left(\frac{\md P_{X_t + a Z_{t,s}| Z_{t,s}}(X_t| Z_{t,s})}{\md P_{X_t | Z_{t,s}}(X_t| Z_{t,s})}\right)^2 \leq \mathbb{E}\left(\frac{\md P_{X_t + a Z_{t,s}| Z_{t,s}, X_{T_{\kappa}}}(X_t| Z_{t,s}, X_{T_{\kappa}})}{\md P_{X_t | Z_{t,s}, X_{T_{\kappa}}}(X_t| Z_{t,s}, X_{T_{\kappa}})}\right)^2\\
    =& \mathbb{E}\left(\frac{\md P_{X_t + a Z_{t,s}| Z_{t,s}, X_{T_{\kappa}}}(X_t| Z_{t,s}, X_{T_{\kappa}})}{\md P_{X_t |  X_{T_{\kappa}}}(X_t|  X_{T_{\kappa}})}\right)^2,
\end{align*}
where we have used the data processing inequality and $X_{T_{\kappa}}\sim \pid$. Since $X_t + a Z_{t,s}|(Z_{t,s}, X_{T_{\kappa}})\sim \mathcal{N}(\sqrt{\lambda_{\kappa t}} X_{T_{\kappa}}+aZ_{t,s}, \sigma^2(1-\lambda_{\kappa t}))$ and $X_t|X_{T_{\kappa}}\sim\mathcal{N}(\sqrt{\lambda_{\kappa t}}X_{T_\kappa}, \sigma^2(1-\lambda_{\kappa t}))$, we can explicitly compute the previous expression, as it corresponds to the $\chi^2$ divergence between two Gaussians, that is,
\begin{align*}
    \mathbb{E}\left(\frac{\md P_{X_t + a Z_{t,s}, Z_{t,s}}(X_t, Z_{t,s})}{\md P_{X_t, Z_{t,s}}(X_t, Z_{t,s})}\right)^2\leq \mathbb{E}\exp \left(\frac{a^2\Vert Z_{t,s}\Vert^2}{\sigma^2(1-\lambda_{\kappa t})}\right)\leq \left(1-2a^2\left(1-\sqrt{\frac{\lambda_{\kappa t}(1-\lambda_{\kappa s})}{\lambda_{\kappa s}(1-\lambda_{\kappa t})}}\right)^2\right)^{-d/2},
\end{align*}
where we have used the expression of the moment generating function of a $\chi^2$ distribution. Under the assumption on the schedule and for $t-s<<1$, it follows that 
\begin{equation*}
    \mathbb{E}\left(\frac{\md P_{X_s + a Z_{t,s}, Z_{t,s}}(X_s, Z_{t,s})}{\md P_{X_s, Z_{t,s}}(X_s, Z_{t,s})}\right)^2\lesssim 1.
\end{equation*}
Putting all this together, we have
\begin{align*}
    \mathbb{E}\left[\left\Vert \nabla\log \hat{\mu}_t(X_t) - \nabla\log \hat{\mu}_t(\alpha_{t, s}^{-1}X_s)\right\Vert^2\right] \lesssim & \Hat{\sigma}_{t,s}^2 \min \left\{\frac{d^2}{\sigma^4(1-\lambda_{\kappa t})^2}, \frac{L_\pi^2 d + K_\pi}{\lambda_{\kappa t}^2}\right\}\\
    \lesssim &\sigma^2 (t-s) \min \left\{\frac{d^2}{\sigma^4(1-\lambda_{\kappa t})^2}, \frac{L_\pi^2 d\vee K_\pi}{\lambda_{\kappa t}^2}\right\}
\end{align*}
for $t-s<<1$. Substituting this into \eqref{eq:auxiliary_time_space_discretisation_error}, it follows
\begin{align*}
    E_{t,s}\lesssim (t-s)^2d\min\left\{\frac{1}{\sigma^2(1-\lambda_{\kappa t})}, \frac{L_\pi}{\lambda_{\kappa t}}\right\} + \sigma^2 (t-s) \min \left\{\frac{d^2}{\sigma^4(1-\lambda_{\kappa t})^2}, \frac{L_\pi^2 d \vee K_\pi}{\lambda_{\kappa t}^2}\right\}.
\end{align*}
Therefore, this results into the following bound for the \gls*{KL}
\begin{align*}
    \kl\left(\mathbb{P}\;||\mathbb{Q}\right)\lesssim & \sum_{l=1}^M \int_{t_{l-1}}^{t_l} \mathbb{E}_{\mathbb{P}}  \left[\left\Vert \nabla\log \hat{\mu}_t(X_t)-\nabla \log \hat{\mu}_{t_-}(X_{t_-})\right\Vert^2\right]\md t + \int_0^{T/\kappa} \Vert v_t\Vert^2_{L^2({\hat{\mu}}_t)} \md t\\
    &\lesssim \sum_{l=1}^M h_l^3 d\min\left(\frac{1}{\sigma^2(1-\lambda_{\kappa t_{l}})}, \frac{L_\pi}{\lambda_{\kappa t_{l-1}}}\right) + \sigma^2 h_l^2 \min \left(\frac{d^2}{\sigma^4(1-\lambda_{\kappa t_{l}})^2}, \frac{L_\pi^2 d \vee K_\pi}{\lambda_{\kappa t_{l-1}}^2}\right)+  \kappa \mathcal{A}_{\lambda}(\mu).  
\end{align*}
Let $\mathbb{Q}_\theta$ be the path measure of the continuous-time interpolation of the \gls*{DALMC} algorithm \eqref{eq:annealed_langevin_mcmc_algorithm_score_approx}. When implementing the algorithm with step sizes $h_l\asymp 1/( M \kappa)$, we have
\begin{align*}
    \kl\left(\mathbb{P}\;||\mathbb{Q}_\theta\right)\lesssim & \kl\left(\mathbb{P}\;||\mathbb{Q}\right) + \sum_{l=0}^{M-1} h_l\mathbb{E}_{\hat{\mu}_t}\left[\left\Vert \nabla \log \hat{\mu}_l(X_{t_l}) - s_\theta(X_{t_l}, t_l)\right\Vert^2\right]
    \\\lesssim &\frac{dL_\pi}{M^2\kappa^3} + \frac{(d^2\vee L_\pi^2d \vee K_\pi)}{M\kappa^2} + \kappa (\mathbb{E}_{\pid}[\Vert X\Vert^2] + d) + \varepsilon_{\text{score}}^2,
\end{align*}
where we have used the bound on the action given in Lemma~\ref{lemma:action_bound_heavy_tail_diffusion}.
We can conclude that by taking
\begin{equation*}
    \kappa = \mathcal{O}\left(\frac{\varepsilon_{\text{score}}^2}{M_2 \vee d}\right),\quad M = \mathcal{O}\left(\frac{(M_2 \vee d)^2(d^2\vee L_\pi^2d \vee K_\pi) L_\pi}{\varepsilon_{\text{score}}^6}\right),
\end{equation*}
we guarantee that $\kl(\mathbb{P}\ \vert \mathbb{Q}_\theta)\leq \varepsilon_{\text{score}}^2$. 
Therefore, for any $\varepsilon = \mathcal{O}(\varepsilon_{\score})$, the \gls*{DALMC} algorithm under relaxed assumptions requires at most 
\begin{equation*}
     M = \mathcal{O}\left(\frac{(M_2 \vee d)^2(d^2\vee L_\pi^2d \vee K_\pi) L_\pi}{\varepsilon^6}\right)
\end{equation*}
steps to approximate $\pid$ to within $\varepsilon^2$ in \gls*{KL} divergence.
Note that if $M_2= \mathcal{O}(d)$, $L_\pi=\mathcal{O}(\sqrt{d})$ and $K_\pi= \mathcal{O}(d^2)$, then the number of steps is of order 
\begin{equation*}
    M = \mathcal{O}\left(\frac{T^2 d^4L_\pi}{\varepsilon^6}\right).
\end{equation*}
\end{proof}

%% file: appendix/assumption_compact_plus_heavy_tail.tex
\begin{lemma}\label{lemma:wegihted_pi}
    If $\Tilde{\pi}$ is supported in a closed Euclidean ball $B_d(0, R)$ and $\gamma\sim t(0, \tau^2 I, \alpha)$, then $\pi= \Tilde{\pi}*\gamma$ satisfies a weighted Poincaré inequality.
\end{lemma}
\begin{proof}
Recall that
\begin{equation*}
    \pi = \Tilde{\pi}*\gamma = \int_{\mathbb{R}^d} \gamma_{x}\md \Tilde{\pi}(x),
\end{equation*}
where $\gamma_x\sim t(x, \sigma^2I, \alpha)$.
Following \cite{functional_inequalities_compactly_supported_assumption}, the variance of a function $f\in L^2(\pi)$ can be decomposed as
\begin{equation*}
    \text{Var}_{\Tilde{\pi}*\gamma}(f) = \int_{\mathbb{R}^d} \text{Var}_{\gamma_x}(f)\md \Tilde{\pi}(x) + \text{Var}_{\Tilde{\pi}}\left(x\mapsto\int f \md \gamma_x\right):= A+B.
\end{equation*}
Since $\gamma_x$ satisfies a weighted Poincaré inequality with constant $C_{\text{PI}, \gamma}$ and weight function $\omega = 1 + \Vert x\Vert^2$ \citep[Proposition 2.17]{heavy_tail_functional_inequalities}, the first term $A$ is bounded by
\begin{equation*}
    A\leq C_{\text{PI}, \gamma}\int_{\mathbb{R}^d}\int_{\mathbb{R}^d}\vert\nabla f\vert^2 \omega \md \gamma_x\md\Tilde{\pi}(x) = C_{\text{PI}, \gamma}\int_{\mathbb{R}^d} \vert\nabla f\vert^2 \omega \md(\Tilde{\pi}*\gamma).
\end{equation*}
For the second term $B$, consider $g:x\mapsto\int_{\mathbb{R}^d} f \md \gamma_x$. Using this, $B$ can be rewritten as
\begin{equation*}
    B=\frac{1}{2} \int\int_{\mathbb{R}^d\times\mathbb{R}^d}(g(x)-g(y))^2\md\Tilde{\pi}(x)\md\Tilde{\pi}(y),
\end{equation*}
where using Cauchy-Schwartz inequality
\begin{equation*}
    (g(x)-g(y))^2 \leq \text{Var}_{\gamma_x}(f)\text{Var}_{\gamma_x}\left(1-\frac{\md \gamma
_y}{\md\gamma_x}\right).
\end{equation*}
For the first factor, we reapply the weighted Poincaré inequality for the $t$ distribution $\gamma_x$. The second factor is the $\chi^2$ divergence between the $t$ distributions $\gamma_x$ and $\gamma_y$.
\begin{align*}
    \chi^2(\gamma_x, \gamma_y)=& \int_{\mathbb{R}^d} \left(\frac{\gamma_x(z)}{\gamma_y(z)}-1\right)^2\gamma_y(z)\;\md z=  \int_{\mathbb{R}^d} \frac{\gamma_x(z)^2}{\gamma_y(z)}\;\md z-1
    =\int_{\mathbb{R}^d} \left(\frac{\alpha + (z-x)^{\intercal}\Sigma^{-1}(z-x)}{\alpha + (z-y)^{\intercal}\Sigma^{-1}(z-y)}\right)^{-(\alpha+ d)/2} \gamma_{x}(z)\md z -1 \\
    =& \int_{\mathbb{R}^d}\left(1+\frac{(z-y)^{\intercal}\Sigma^{-1}(z-y)-(z-x)^{\intercal}\Sigma^{-1}(z-x)}{\alpha + (z-x)^{\intercal}\Sigma^{-1}(z-x)}\right)^{(\alpha+d)/2}\gamma_x(z)\md z -1\\
     =& \int_{\mathbb{R}^d}\left(1+\frac{y^{\intercal}\Sigma^{-1}y-x^{\intercal}\Sigma^{-1}x-2z^{\intercal}\Sigma^{-1}(y-x)}{\alpha + (z-x)^{\intercal}\Sigma^{-1}(z-x)}\right)^{(\alpha+d)/2}\gamma_x(z)\md z -1\\
     =&  \int_{A=\{\Vert z-x\Vert^2 > 1\}}\left(1+\frac{y^{\intercal}\Sigma^{-1}y-x^{\intercal}\Sigma^{-1}x-2z^{\intercal}\Sigma^{-1}(y-x)}{\alpha + (z-x)^{\intercal}\Sigma^{-1}(z-x)}\right)^{(\alpha+d)/2}\gamma_x(z)\md z\\
    &+  \int_{\mathbb{R}^d\setminus A}\left(1+\frac{y^{\intercal}\Sigma^{-1}y-x^{\intercal}\Sigma^{-1}x-2z^{\intercal}\Sigma^{-1}(y-x)}{\alpha + (z-x)^{\intercal}\Sigma^{-1}(z-x)}\right)^{(\alpha+d)/2}\gamma_x(z)\md z-1.
\end{align*}
For $z\in \mathbb{R}^d$, the following holds 
\begin{equation*}
    (z-x)^{\intercal}\Sigma^{-1}(z-x) = (z-x)^{\intercal}U^{\intercal}S^{-1/2}S^{-1/2}U(z-x) = \Vert S^{-1/2}U(z-x)\Vert^2 \geq 0.
\end{equation*}
The second integral can be upper bounded as follows
\begin{align*}
    &\int_{\mathbb{R}^d\setminus A}\left(1+\frac{y^{\intercal}\Sigma^{-1}y-x^{\intercal}\Sigma^{-1}x-2z^{\intercal}\Sigma^{-1}(y-x)}{\alpha + (z-x)^{\intercal}\Sigma^{-1}(z-x)}\right)^{(\alpha+d)/2}\gamma_x(z)\md z \\
    & = \int_{\mathbb{R}^d\setminus A}\left(\frac{\alpha + (z-y)^{\intercal}\Sigma^{-1}(z-y)}{\alpha + (z-x)^{\intercal}\Sigma^{-1}(z-x)}\right)^{(\alpha+ d)/2} \gamma_{x}(z)\md z \leq \int_{\mathbb{R}^d\setminus A}\left(\frac{\alpha + \sigma_{\min}^{-1}\Vert z-y\Vert^2}{\alpha}\right)^{(\alpha+ d)/2} \gamma_{x}(z)\md z \\
    &\leq\int_{\mathbb{R}^d\setminus A} \left(\frac{\alpha + \sigma_{\min}^{-1}2(\Vert x-y\Vert^2 + \Vert z-x\Vert^2)}{\alpha}\right)^{(\alpha+ d)/2}  \gamma_{x}(z)\md z \leq \left(\frac{\alpha + \sigma_{\min}^{-1}2(4R^2 + 1)}{\alpha}\right)^{(\alpha+ d)/2} = \kappa_R,
\end{align*}
where we have used that $\Vert x-y\Vert$ is bounded by $2R$.
On the other hand, for $z\in A$ we have
\begin{align*}
    &1+\frac{y^{\intercal}\Sigma^{-1}y-x^{\intercal}\Sigma^{-1}x-2z^{\intercal}\Sigma^{-1}(y-x)}{\alpha + (z-x)^{\intercal}\Sigma^{-1}(z-x)} = 1+\frac{(y-x)^{\intercal}\Sigma^{-1}(y-x)-2(z-x)^{\intercal}\Sigma^{-1}(y-x)}{\alpha + (z-x)^{\intercal}\Sigma^{-1}(z-x)} \\
    & \leq 1 + \frac{(y-x)^{\intercal}\Sigma^{-1}(y-x)}{\alpha + (z-x)^{\intercal}\Sigma^{-1}(z-x)} + \frac{2\Vert (z-x)^{\intercal}\Sigma^{-1}(y-x)\Vert }{ (z-x)^{\intercal}\Sigma^{-1}(z-x)}\\
    & \leq 1 + \frac{(y-x)^{\intercal}\Sigma^{-1}(y-x)}{\alpha + (z-x)^{\intercal}\Sigma^{-1}(z-x)} + \frac{2\Vert S^{-1/2}U(z-x)\Vert\Vert S^{-1/2}U(y-x)\Vert }{ \Vert S^{-1/2}U(z-x)\Vert^2}\\
    & = 1 + \frac{(y-x)^{\intercal}\Sigma^{-1}(y-x)}{\alpha + (z-x)^{\intercal}\Sigma^{-1}(z-x)} + \frac{2\Vert S^{-1/2}U(y-x)\Vert }{ \Vert S^{-1/2}U(z-x)\Vert} \leq 1 + \frac{4R^2 \Vert S^{-1/2}\Vert^2}{\alpha} + \frac{4R\Vert S^{-1/2}\Vert }{\sigma_{\max}^{-1/2}}\leq 1 + \frac{4R^2}{\alpha\sigma_{\min}} + \frac{4R\sigma_{\max}^{1/2}}{\sigma_{\min}^{1/2}}.
\end{align*}
Therefore, we obtain
\begin{equation*}
    \int_{A=\{\Vert z-x\Vert^2 > 1\}}\left(1+\frac{y^{\intercal}\Sigma^{-1}y-x^{\intercal}\Sigma^{-1}x-2z^{\intercal}\Sigma^{-1}(y-x)}{\alpha + (z-x)^{\intercal}\Sigma^{-1}(z-x)}\right)^{(\alpha+d)/2}\gamma_x(z)\md z\leq \left( 1 + \frac{4R^2}{\alpha\sigma_{\min}} + \frac{4R\sigma_{\max}^{1/2}}{\sigma_{\min}^{1/2}}\right)^{(\alpha+d)/2} = \beta_R
\end{equation*}
This shows that $\chi^2(\gamma_x, \gamma_y)$ is upper bounded with dependence on $R$ of the form $R^{(\alpha+d)}$. In particular, $B$ satisfies the following bound
\begin{equation*}
    B\leq  C_{\text{PI}, \gamma} (\kappa_R + \beta_R -1)\int_{\mathbb{R}^d} \vert\nabla f\vert^2 \omega \md(\Tilde{\pi}*\gamma).
\end{equation*}
Therefore, the measure $\Tilde{\pi}*\gamma$ satisfies a weighted Poincaré inequality with constant 
\begin{equation*}
    C_{\text{PI}} \leq C_{\text{PI}, \gamma} (\kappa_R + \beta_R).
\end{equation*}

\end{proof}

%% file: appendix/proof_smoothness_of_heavy_tail_path_alternative.tex
\begin{proof}
    First note that a $d$-dimensional Student's $t$-distribution $\varphi_{\mu,\sigma^2}\sim t(\mu, \sigma^2 I, \alpha)$ satisfies
    \begin{align*}
         \nabla^2\log \varphi_{\mu, \sigma^2}(x)  &= (\alpha + d) \left( \frac{I}{\alpha \sigma^2 + \Vert x-\mu\Vert^2} - \frac{(x-\mu)(x-\mu)^\intercal}{(\alpha \sigma^2 + \Vert x-\mu\Vert^2)^2}\right).
    \end{align*}
    Hence, we have
    \begin{align*}
    \nabla^2\log \varphi_{\mu, \sigma^2}(x) &\preccurlyeq \frac{\alpha + d}{\alpha \sigma^2} I,\\
         \nabla^2\log \varphi_{\mu, \sigma^2}(x) &\succcurlyeq  - \frac{(\alpha + d)\Vert x-\mu\Vert^2}{(\alpha \sigma^2 + \Vert x-\mu\Vert^2)^2} I \succcurlyeq  -\frac{\alpha + d}{2\alpha \sigma^2} I,
    \end{align*}
    which shows that $\nabla \log\varphi_{\mu, \sigma^2}$ is Lipschitz. Let $X_t\sim \mu_t$, $\varphi_{\sigma^2}\sim t(0, \sigma^2 I, \alpha)$ and for simplicity denote $\pi = \pid$, we have that
    \begin{equation*}
        \mu_t(x) = \frac{1}{\lambda_t^{d/2}} \pi\left(\frac{x}{\sqrt{\lambda_t}}\right) * \frac{1}{(1-\lambda_t)^{d/2}}\varphi_{\sigma^2}\left(\frac{x}{\sqrt{1-\lambda_t}}\right) = \int \frac{1}{\lambda_t^{d/2}} \pi\left(\frac{y}{\sqrt{\lambda_t}}\right) \varphi_{\sigma^2(1-\lambda_t)}(x-y) \md y,
    \end{equation*}
    where $\varphi_{\sigma^2(1-\lambda_t)}\sim t(0, \sigma^2(1-\lambda_t)I, \alpha)$.
The Hessian $\nabla^2 \log\mu_t(x)$ is then given by
\begin{equation*}
    \nabla^2\log\mu_t(x) = \mathbb{E}_{Y\sim \rho_{t, x}} \left[\nabla^2 \log \varphi_{\sigma^2(1-\lambda_t)}(x - Y)\right] + \text{Cov}_{Y\sim \rho_{t, x}} [\nabla \log\varphi_{\sigma^2(1-\lambda_t) }(x- Y)],
\end{equation*}
where $\rho_{t, x} (y) \propto \frac{1}{\lambda_t^{d/2}} \pi\left(\frac{y}{\sqrt{\lambda_t}}\right) \varphi_{\sigma^2(1-\lambda_t)}(x-y)$.
Using that $\nabla\log\varphi_{\sigma^2}$ is Lipschitz with constant $L_\sigma$, the first term can be as follows
\begin{equation*}
    -\frac{L_\sigma}{1-\lambda_t}I\preccurlyeq\mathbb{E}_{Y\sim \rho_{t, x}} \left[\nabla^2 \log \varphi_{\sigma^2(1-\lambda_t)}(x - Y)\right] \preccurlyeq \frac{L_\sigma}{1-\lambda_t} I.
\end{equation*}
Focusing on the covariance term we have that
\begin{align*}
   0\preccurlyeq \text{Cov}_{Y\sim \rho_{t, x}} [\nabla \log\varphi_{\sigma^2(1-\lambda_t)}(x- Y)] \preccurlyeq \mathbb{E}_{Y\sim \rho_{t, x}}\left[ \Vert \nabla \log\varphi_{\sigma^2 (1-\lambda_t)}(x- Y)\Vert^2 \right] I.
\end{align*}
The expectation can be bounded independently of $x$ as follows
\begin{align}
    \mathbb{E}_{Y\sim \rho_{t, x}}&\left[ \Vert \nabla \log\varphi_{\sigma^2 (1-\lambda_t)}(x- Y)\Vert^2 \right] =  (\alpha +d)^2 \ \mathbb{E}_{Y\sim \rho_{t, x}}\left[  \frac{\frac{\left \Vert x- Y\right\Vert^2}{(\alpha\sigma^2(1-\lambda_t))^2}}{\left(1 + \frac{\Vert x - Y\Vert^2}{\alpha\sigma^2(1-\lambda_t)}\right)^2} \right]\nonumber \\
    &\leq \frac{(\alpha +d)^2}{2}\ \mathbb{E}_{Y\sim \rho_{t, x}}\left[  \frac{1}{\alpha\sigma^2(1-\lambda_t) + \Vert x - Y\Vert^2} \right] \leq \frac{(\alpha + d)^2}{2\alpha\sigma^2(1-\lambda_t)}.\label{eq:auxiliary_condition_heavy_tailed_distribution}
\end{align}
Therefore, we have that 
\begin{equation}\label{eq:smoothness_bound_heavy_tailed_diffusion}
    -\frac{L_\sigma}{1-\lambda_t} I \preccurlyeq \nabla^2\log\mu_t(x) \preccurlyeq \left(\frac{L_\sigma}{1-\lambda_t}+ \frac{(\alpha + d)^2}{2\alpha\sigma^2(1-\lambda_t)}\right) I.
\end{equation}
On the other hand, under assumption \Cref{assumption:heavy_tailed_target_condition} we have that $\nabla\log\pi$ is $L_\pi$-Lipschitz and $\Vert\nabla \log \pi \Vert^2\leq C_\pi$, then 
\begin{align*}
    -\frac{L_\pi}{\lambda_t} I\preccurlyeq \nabla^2\log\mu_t(x) &= \frac{1}{\lambda_t} \mathbb{E}_{Y\sim {\rho}_{t, x}} \left[\nabla^2 \log \pi\left(\frac{Y}{\sqrt{\lambda_t}}\right)\right] + \frac{1}{\lambda_t}\text{Cov}_{Y\sim {\rho}_{t, x}} \left[\nabla \log\pi\left(\frac{Y}{\lambda_t}\right)\right]\\
    &\preccurlyeq  \frac{1}{\lambda_t}\left({L_\pi}+\mathbb{E}_{\rho_{t,x}}\left[ \left\Vert\nabla\log\pi\left(\frac{Y}{\lambda_t}\right)\right\Vert^2\right]\right) I \preccurlyeq \frac{L_\pi+ C_\pi}{\lambda_t} I,
\end{align*}
where $\rho_{t, x} (y) \propto \frac{1}{\lambda_t^{d/2}} \pi\left(\frac{y}{\sqrt{\lambda_t}}\right) \varphi_{\sigma^2(1-\lambda_t)}(x-y)$. Putting this together with \eqref{eq:smoothness_bound_heavy_tailed_diffusion}, we obtain
\begin{equation}
    \max\left\{-\frac{L_\sigma}{1-\lambda_t}, -\frac{L_\pi}{\lambda_t}\right\}I \preccurlyeq \nabla^2\log\mu_t(x) \preccurlyeq \min\left\{ \left(\frac{L_\sigma}{1-\lambda_t}+ \frac{(\alpha + d)^2}{2\alpha\sigma^2(1-\lambda_t)}\right) ,\left(\frac{L_\pi + C_\pi}{\lambda_t}\right) \right\}I, \label{eq:lipschitz_constant_t_heavy_tailed_diffusion}
\end{equation}
which concludes that $\nabla\log\mu_t$ is Lipschitz for all $t\in[0, T]$ with constant
\begin{equation*}
    L_t\leq \min\left\{ \left(\frac{L_\sigma}{1-\lambda_t}+ \frac{(\alpha + d)^2}{2\alpha\sigma^2(1-\lambda_t)}\right) ,\left(\frac{L_\pi + C_\pi}{\lambda_t}\right) \right\}.
\end{equation*}

We now prove the second part of Lemma~\ref{lem:regularity_of_heavy_tailed_path}, that is, Assumption \Cref{assumption:heavy_tailed_target_condition} is satisfied when  $\pi(x) = \Tilde{\pi} * \varphi_{\tau^2}(x)$, where $\Tilde{\pi}$ is compactly supported and $\varphi_{\tau^2}\sim t(0, \tau^2 I, \Tilde{\alpha})$ (Assumption \Cref{assumption:compact_plus_t_distribution_target}). In this case, we can write
\begin{equation*}
    \Vert\nabla\log\pi(x)\Vert^2 = \left \Vert\mathbb{E}_{Y\sim\hat{\rho}_{t, x}} \left[\nabla\log\varphi_{\tau^2}(x-Y)\right]\right\Vert^2 \leq \mathbb{E}_{Y\sim\hat{\rho}_{t, x}} \left[\left \Vert\nabla\log\varphi_{\tau^2}(x-Y)\right\Vert^2\right] \leq \frac{(\alpha + d)^2}{2\Tilde{\alpha}\tau^2} = C_\pi,
\end{equation*}
where $\hat{\rho}_{t, x}\propto \hat{\pi}(y)\varphi_{\tau^2}(x-y)$ and we have used the same trick as in \eqref{eq:auxiliary_condition_heavy_tailed_distribution}. Denote by $L_\tau$ the Lipschitz constant of $\varphi_{\tau^2}$. The Hessian can be upper and lower bounded as follows
\begin{align*}
    -L_\tau I\preccurlyeq \nabla^2\log\pi(x) &= \mathbb{E}_{Y\sim\hat{\rho}_{t,x}} \left[\nabla^2\log\varphi_{\tau^2}(x-Y)\right] + \text{Cov}_{Y\sim\Hat{\rho}_{t,x}}[\nabla\log\varphi_{\tau^2}(x-Y)]\\
    & \preccurlyeq \left(L_\tau + \mathbb{E}_{Y\sim\hat{\rho}_{t,x}}\left[\Vert\nabla\log\varphi_{\tau^2}(x-Y)\Vert^2 \right] \right)I \leq \left(L_{\tau} + \frac{(\Tilde{\alpha} + d)^2}{2\Tilde{\alpha}\tau^2}\right) I,
\end{align*}
which shows that $\nabla\log\pi$ is Lipschitz with constant $L_\pi = L_{\tau} + \frac{(\Tilde{\alpha} + d)^2}{2\Tilde{\alpha}\tau^2}$.
Finally, observe that exploiting Assumption~\Cref{assumption:compact_plus_t_distribution_target} we can get a more refined Lipschitz constant for $\nabla\log\mu_t$ than that of \eqref{eq:lipschitz_constant_t_heavy_tailed_diffusion}. That is,
\begin{align*}
    -\frac{L_\tau}{\lambda_t}\preccurlyeq \nabla^2\log\mu_t(x) &= \mathbb{E}_{Y\sim \Tilde{\rho}_{t, x}} \left[\nabla^2 \log \varphi_{\tau^2\lambda_t}(x - Y)\right] + \text{Cov}_{Y\sim \Tilde{\rho}_{t, x}} [\nabla \log\varphi_{\tau^2 \lambda_t }(x- Y)]\preccurlyeq\left(\frac{L_\tau}{\lambda_t}+ \frac{(\Tilde{\alpha} + d)^2}{2\Tilde{\alpha}\tau^2\lambda_t}\right) I,
\end{align*}
where $\Tilde{\rho}_{t, x} (y) \propto \left(\frac{1}{\lambda_t^{d/2}} \pi\left(\frac{y}{\sqrt{\lambda_t}}\right) *\frac{1}{(1-\lambda_t)^{d/2}}\varphi_{\sigma^2}\left(\frac{y}{\sqrt{1-\lambda_t}}\right)\right)\varphi_{\tau^2\lambda_t}(x-y)$. This combined with \eqref{eq:smoothness_bound_heavy_tailed_diffusion} leads to
\begin{equation*}
    \max\left\{-\frac{L_\sigma}{1-\lambda_t}, -\frac{L_\tau}{\lambda_t}\right\}I \preccurlyeq \nabla^2\log\mu_t(x) \preccurlyeq \min\left\{ \left(\frac{L_\sigma}{1-\lambda_t}+ \frac{(\alpha + d)^2}{2\alpha\sigma^2(1-\lambda_t)}\right) ,\left(\frac{L_\tau}{\lambda_t}+ \frac{(\Tilde{\alpha} + d)^2}{2\Tilde{\alpha}\tau^2\lambda_t}\right) \right\}I,
\end{equation*}
which shows that $\nabla\log\mu_t$ is Lipschitz for all $t\in[0, T]$.
\end{proof}

%% file: appendix/action_bound_heavy_tail_diffusion.tex
\begin{proof}
Consider the reparametrised version of $\mu_t$ in terms of the schedule $\lambda_t$, denoted as $\Tilde{\mu}_\lambda$ and let $X_{\lambda}\sim\Tilde{\mu}_{\lambda}$ and $X_{\lambda + \delta}\sim\Tilde{\mu}_{\lambda + \delta}$.
Recall that 
\begin{equation*}
    X_{\lambda} = \sqrt{\lambda} X + \sqrt{1-\lambda}\sigma Z
\end{equation*}
where $X\sim\pid$ and $Z\sim t(0, I, \alpha)$.  The Wasserstein-2 distance between $\Tilde{\mu}_\lambda$ and $\Tilde{\mu}_{\lambda+\delta}$ is given by
\begin{align*}
     W_2^2(\Tilde{\mu}_\lambda, &\Tilde{\mu}_{\lambda+\delta}) \leq \mathbb{E}\left[\Vert X_\lambda-X_{\lambda + \delta}\Vert^2\right]\\
&=\mathbb{E}\left[\left\Vert(\sqrt{\lambda}-\sqrt{\lambda+\delta})X\right\Vert^2\right] + \mathbb{E}\left[\left\Vert\left(\sqrt{1-\lambda}-\sqrt{1-\lambda-\delta}\right)\sigma Z\right\Vert^2\right]\\
&= (\sqrt{\lambda}-\sqrt{\lambda+\delta})^2 \mathbb{E}\left[\Vert X\Vert^2\right] + \left(\sqrt{1-\lambda}-\sqrt{1-\lambda-\delta}\right)^2 \frac{\sigma^2d\alpha}{\alpha -2}.
\end{align*}
Using the definition of the metric derivative we have
\begin{align*}
    \left\vert\Dot{\Tilde{\mu}}\right\vert_\lambda^2 = \lim_{\delta\to 0}\frac{ W_2^2(\Tilde{\mu}_\lambda, \Tilde{\mu}_{\lambda+\delta})}{\delta^2} \leq \frac{\mathbb{E}\left[\Vert X\Vert^2\right]}{4\lambda} + \frac{1}{4(1-\lambda)}\frac{\sigma^2 d\alpha}{\alpha-2}.
\end{align*}
Since $\mu_t = \Tilde{\mu}_{\lambda_t}$, we have that $\vert\Dot{\mu}\vert_t = \left\vert\Dot{\Tilde{\mu}}\right\vert_\lambda\left\vert\partial_t{\lambda_t}\right\vert$. Using assumption \Cref{assumption:schedule_form_heavy_tail_diffusion} for the schedule, we have the following expression for the action
\begin{align}
    \mathcal{A}_{\lambda}(\mu) &= \int_0^T \vert\Dot{\mu}\vert_t^2 \md t = \int_0^T \left\vert\Dot{\Tilde{\mu}}\right\vert_\lambda^2\left\vert\partial_t{\lambda_t}\right\vert^2 \md t\nonumber\\
    &\lesssim \int_0^T \left(\frac{\mathbb{E}\left[\Vert X\Vert^2\right]}{4 \lambda_t} + \frac{\sigma^2\alpha}{4(1-\lambda_t)(\alpha-2)}d\right)\left\vert \partial_t{\lambda_t}\right\vert^2\md t \nonumber\\
    &= \int_0^T \left(\frac{\mathbb{E}\left[\Vert X\Vert^2\right]\sqrt{1-\lambda_t}}{4\sqrt{\lambda_t}} + \frac{\sigma^2\alpha\sqrt{\lambda_t}}{4\sqrt{1-\lambda_t}(\alpha-2)}d\right)\left\vert \frac{\partial_t{\lambda_t}}{\sqrt{\lambda_t(1-\lambda_t)}}\right\vert \left\vert \partial_t{\lambda_t}\right\vert\md t \nonumber\\    
    &\leq C_\lambda \int_0^T \left(\frac{\mathbb{E}\left[\Vert X\Vert^2\right]\sqrt{1-\lambda_t}}{4\sqrt{\lambda_t}} +\frac{\sigma^2\alpha\sqrt{\lambda_t}}{4\sqrt{1-\lambda_t}(\alpha-2)}d\right) \left\vert \partial_t{\lambda_t}\right\vert\md t \nonumber\\ 
    &\leq C_\lambda \int_{0}^{1} \left(\frac{\mathbb{E}\left[\Vert X\Vert^2\right]\sqrt{1-\lambda}}{4\sqrt{\lambda}} + \frac{\sigma^2\alpha\sqrt{\lambda}}{4\sqrt{1-\lambda}(\alpha-2)} d\right) \md \lambda \nonumber\\
    &\leq \frac{C_\lambda\pi}{8} \left(\mathbb{E}\left[\Vert X\Vert^2\right] + \frac{\sigma^2 d\alpha}{\alpha-2}\right).\nonumber
\end{align}
\end{proof}

%% file: appendix/theorem_discretisation_reparametrised_heavy_tailed_diffusion.tex
\begin{lemma}\label{lemma:auxiliary_change_score_1}
    Suppose that $p(x) \propto e^{-V(x)}$ is a probability density on $\mathbb{R}^d$, where $\nabla V(x)$ is Lipschitz continuous with constant $L$ and let $\varphi_{\sigma^2}(x)$ be the density function of a Student's t distribution $t(0, \sigma^2 I, \alpha)$. 
   Then 
    \begin{equation*}
        \left\Vert \nabla \log \frac{p(x)}{p*\varphi_{\sigma^2}(x)} \right\Vert\leq  L\ \mathbb{E}_{Y|x} \left[\left \Vert Y- x\right\Vert\right],
    \end{equation*}
    where the distribution of $Y|x$ is of the form $\frac{p(y)\varphi_\sigma^2(x-y)}{p(x)*\varphi_\sigma^2(x)}$.
\end{lemma}
\begin{proof}
Observe that 
\begin{align*}
   \nabla \log p*\varphi_{\sigma^2}(x) &= -\frac{\int_{\mathbb{R}^d} \nabla V(y) e^{-V(y)}\left(1+\frac{1}{\alpha}\frac{\Vert y-x\Vert^2}{\sigma^2}\right)^{-\frac{\alpha + d}{2}} \md y}{\int_{\mathbb{R}^d}e^{-V(y)} \left(1+\frac{1}{\alpha}\frac{\Vert y-x\Vert^2}{\sigma^2}\right)^{-\frac{\alpha + d}{2}} \md y} = -\frac{\int_{\mathbb{R}^d} \nabla V(y) p(y) \varphi_{\sigma^2}(x-y) \md y}{\int_{\mathbb{R}^d}p(y) \varphi_{\sigma^2}(x-y) \md y} = -\mathbb{E}_{\gamma_{x, \sigma^2}}\left[\nabla V(Y)\right],
\end{align*}
where $\gamma_{x, \sigma^2}$ denotes the probability density
\begin{equation*}
    \gamma_{x, \sigma^2}(y)=\frac{ p(y) \varphi_{\sigma^2}(x-y)}{ p(x)* \varphi_{\sigma^2}(x)}.
\end{equation*}
Using the Lipschitzness of $\nabla V$, we have
\begin{align*}
     \left\Vert\nabla \log \frac{p(x)}{p*\varphi_{\sigma^2}(x)} \right\Vert &= \left\Vert  \mathbb{E}_{\gamma_{x, \sigma^2}}\left[\nabla V(Y)- \nabla V(x)\right]\right\Vert \leq L\ \mathbb{E}_{\gamma_{x, \sigma^2}} \left[\left \Vert Y- x\right\Vert\right].
\end{align*}
\end{proof}

\begin{lemma}\label{lemma:auxiliary_change_score_2}
 With the setting in Lemma~\ref{lemma:auxiliary_change_score_1}. Denote $p_\lambda(x)=\lambda^d p(\lambda x)$ for $\lambda\geq 1$. Then 
    \begin{align*}
        \left\Vert \nabla \log \frac{p(x)}{p_\lambda *\varphi_{\sigma^2}(x)} \right\Vert\lesssim& \lambda (\lambda-1) \Vert x\Vert + (\lambda-1)\Vert \nabla V(x)\Vert+ \lambda^2 L\ \mathbb{E}_{Y|x}\left[\Vert Y- x\Vert\right],
    \end{align*}
\end{lemma}
where the law of $Y|x$ is given by $\frac{p_\lambda(y)\varphi_\sigma^2(x-y)}{p_\lambda(x)*\varphi_\sigma^2(x)}$.
\begin{proof}
Using the triangle inequality,
\begin{equation*}
    \left\Vert \nabla \log \frac{p(x)}{p_{\lambda}*\varphi_{\sigma^2}(x)}\right\Vert \leq \left\Vert \nabla \log \frac{p(x)}{p_\lambda(x)}\right\Vert + \left\Vert \nabla \log \frac{p_\lambda(x)}{p_{\lambda}*\varphi_{\sigma^2}(x)}\right\Vert.
\end{equation*}
The first term can be bounded as
\begin{align*}
    \left\Vert \nabla \log \frac{p(x)}{p_\lambda(x)}\right\Vert &= \left\Vert \lambda \nabla V(\lambda x)-\nabla V(x)\right\Vert \leq \left\Vert \lambda \nabla V(\lambda x)-\lambda\nabla V(x)\right\Vert+ \left\Vert \lambda \nabla V(x)-\nabla V(x)\right\Vert \\
    & \leq \lambda (\lambda-1) \Vert x\Vert + (\lambda-1)\Vert \nabla V(x)\Vert.
\end{align*}
By the result in Lemma~\ref{lemma:auxiliary_change_score_1}, we have the following bound for the second term
\begin{equation*}
    \left\Vert \nabla \log \frac{p_\lambda(x)}{p_{\lambda}*\varphi_{\sigma^2}(x)}\right\Vert \lesssim \lambda^2 L\ \mathbb{E}_{Y|x}\left[\Vert Y-x\Vert\right],
\end{equation*}
where we have used that $\lambda \nabla V(\lambda x)$ is $\lambda^2 L$-Lipschitz and $Y|x$ has a distribution of the form
\begin{equation*}
    \frac{p_\lambda(y)\varphi_\sigma^2(x-y)}{p_\lambda(x)*\varphi_\sigma^2(x)}.
\end{equation*}
\end{proof}

\begin{proof}[Proof of Theorem \ref{theorem:discretisation_analysis_convolutional_heavy_tailed_diffusion}]
Similarly to the proof of Theorem~\ref{theorem:discretisation_analysis_convolutional_path}, using Girsanov's theorem, we have that the following bound for $\kl(\mathbb{P}||\mathbb{Q})$.
\begin{align}
    \kl&\left(\mathbb{P}\;||\mathbb{Q}\right)= \frac{1}{4}\int_0^{T/\kappa}\mathbb{E}_{\mathbb{P}}\left[\left\Vert \nabla\log \hat{\mu}_t(X_t)-\nabla \log \hat{\mu}_{t_-}(X_{t_-}) + v_t(X_t)\right\Vert^2\right]\md t\nonumber\\
    \leq&  \sum_{l=1}^M\int_{t_{l-1}}^{t_l} L_{\kappa t}^2 \ \mathbb{E}_{\mathbb{P}} \left[\left\Vert X_t-X_{t_-}\right\Vert^2\right]\md t+ \int_0^{T/\kappa} \Vert v_t\Vert^2_{L^2({\hat{\mu}}_t)} \md t+ \int_0^{T/\kappa}\mathbb{E}_{\mathbb{P}} \left[\left\Vert \nabla\log \frac{\hat{\mu}_t(X_{t_-})}{\hat{\mu}_{t_-}(X_{t_-})}\right\Vert^2\right]\md t, \label{eq:kl_path_bound_intermediate_heavy_tail_diffusion}
\end{align}
where we have used that  $\nabla\log\hat{\mu}_{t}$ is Lipschitz with constant $L_{\kappa t}$. 
First, we bound the change in the score function $\mathbb{E}_{\mathbb{P}} \left[\left\Vert \nabla\log \frac{\hat{\mu}_t(X_{t_-})}{\hat{\mu}_{t_-}(X_{t_-})}\right\Vert^2\right]$. Let $t\geq t_{-}$, we can write
\begin{equation*}
    \hat{\mu}_{t_-} = T_{\sqrt{\frac{\lambda_{\kappa t}}{\lambda_{\kappa t_-}}}} \#\hat{\mu}_t * t\left(0, \left(\sqrt{1-\lambda_{\kappa t_-}}-\sqrt{\frac{(1-\lambda_{\kappa t})\lambda_{\kappa t_-}}{\lambda_{\kappa t}}}\right)^2\sigma^2 I, \alpha\right) = T_{\sqrt{\frac{\lambda_{\kappa t}}{\lambda_{\kappa t_-}}}} \#\hat{\mu}_t * t\left(0, \gamma_t\sigma^2 I, \alpha\right) ,
\end{equation*}
where the pushforward $T_{\lambda}\#$ is defined as $T_{\lambda}\#\mu(x) = \lambda^d \mu(\lambda x)$ and we have abused notation by identifying $t(0, \gamma_t\sigma^2 I, \alpha)$ with its density function. Using the result in Lemma~\ref{lemma:auxiliary_change_score_2}, we have
\begin{align*}
        \left\Vert \nabla\log \frac{\hat{\mu}_t(X_{t_-})}{\hat{\mu}_{t_-}(X_{t_-})}\right\Vert^2\lesssim& \frac{\lambda_{\kappa t}}{\lambda_{\kappa t_-}} \left(\sqrt{\frac{\lambda_{\kappa t}}{\lambda_{\kappa t_-}}}-1\right)^2 \Vert X_{t_-}\Vert^2 + \left(\sqrt{\frac{\lambda_{\kappa t}}{\lambda_{\kappa t_-}}}-1\right)^2\Vert \nabla \log \hat{\mu}_t(X_{t_-})\Vert^2 \\
        &+  \left(\frac{\lambda_{\kappa t}}{\lambda_{\kappa t_-}}\right)^2 L_{\kappa t}^2 \mathbb{E}_{Y|X_{t_-}}\left[\Vert Y-X_{t_-}\Vert^{2}\right],
    \end{align*}
    where the distribution of $Y|X_t$ is given by
    \begin{equation*}
        Y|X_{t_-}\sim \frac{T_{\sqrt{\frac{\lambda_{\kappa t}}{\lambda_{\kappa t_-}}}} \#\hat{\mu}_t (y)\ \varphi_{\gamma_t\sigma^2}(X_{t_-}-y)}{T_{\sqrt{\frac{\lambda_{\kappa t}}{\lambda_{\kappa t_-}}}} \#\hat{\mu}_t * \varphi_{\gamma_t\sigma^2} (X_{t_-})} = \frac{T_{\sqrt{\frac{\lambda_{\kappa t}}{\lambda_{\kappa t_-}}}} \#\hat{\mu}_t (y)\ \varphi_{\gamma_t\sigma^2}(X_{t_-}-y)}{\hat{\mu}_{t_-} (X_{t_-})},
    \end{equation*}
    where $\varphi_{\gamma_t\sigma^2}$ is the density function of a Student's $t$ distribution of the form $t\left(0, \gamma_t\sigma^2 I, \alpha\right)$.
Therefore, we have that
\begin{equation*}
    \mathbb{E}_{\mathbb{P}}\left[\mathbb{E}_{Y|X_{t_-}}\left[\Vert Y-X_{t_-}\Vert^{2}\right]\right] = \mathbb{E}_{X_{t_-}, Y} \Vert Y-X_{t_-}\Vert^2,
\end{equation*}
where the joint distribution of $( X_{t_-}, Y)\sim \rho_{( X_{t_-}, Y)}(x, y)$ is of the form
\begin{equation*}
    \rho_{( X_{t_-}, Y)}(x, y)\propto T_{\sqrt{\frac{\lambda_{\kappa t}}{\lambda_{\kappa t_-}}}} \#\hat{\mu}_t (y)\ \varphi_{\gamma_t\sigma^2}(x-y).
\end{equation*}
Using a change of measure, it follows that $Y$ is independent of $X_{t_-}- Y$ and the distribution of $X_{t_-}- Y$ is $t(0, \gamma_t\sigma^2 I, \alpha)$ with $\alpha>2$. This results into
\begin{equation*}
\mathbb{E}_{\mathbb{P}}\left[\mathbb{E}_{Y|X_{t_-}}\left[\Vert Y-X_{t_-}\Vert^{2}\right]\right] = \mathbb{E}_{Z\sim t(0, \gamma_t\sigma^2 I, \alpha)}\left[\Vert Z\Vert^2\right] = \gamma_t\sigma^2 d\frac{\alpha}{\alpha-2}.
\end{equation*}
By assumption on the schedule 
\begin{align*}
    \frac{\lambda_{\kappa t_-}}{\lambda_{\kappa t}} = \mathcal{O}(1+ h_l)
, \quad \quad \left(\sqrt{\frac{\lambda_{\kappa t}}{\lambda_{\kappa t_-}}}-1\right)^2 = \mathcal{O}(h_l^2), \quad\quad \gamma_t \lesssim 1-\frac{\lambda_{\kappa t_-}}{\lambda_{\kappa t}}= \mathcal{O}\left( h_l\right).
\end{align*}
Given that $X_t = \sqrt{\lambda_t} X + \sqrt{1-\lambda_t} \sigma^2 Z$ for $X_t\sim\hat{\mu}_t$, we derive the following moment bound
\begin{align*}
    \mathbb{E}_{\mathbb{P}}\left[\left\Vert X_{t_-}\right\Vert^2\right] =& \mathbb{E}_{\mathbb{P}}\left[\left\Vert \sqrt{\lambda_{\kappa t_-}} X + \sqrt{1-\lambda_{\kappa t_-}} Z\right\Vert^2\right] = \lambda_{\kappa t_-} \mathbb{E}_{\pid}\left[\left\Vert X\right\Vert^2\right] + (1-\lambda_{\kappa t_-})\sigma^2 \frac{\alpha d}{\alpha-2} \lesssim \mathbb{E}_{\pid}\left[\left\Vert X\right\Vert^2\right] + d.
\end{align*}
Similarly to the proof of Theorem~\ref{theorem:discretisation_analysis_convolutional_path}, it holds that 
\begin{align*}
    \mathbb{E}_{\mathbb{P}}\left[\left\Vert \nabla\log\hat{\mu}_t(X_{t_-})\right\Vert^2\right] \leq  L_{\kappa t} d  + L_{\kappa t}^2\mathbb{E}_{\mathbb{P}} \left[\left\Vert X_t-X_{t_-}\right\Vert^2\right]. 
\end{align*}
This results into
\begin{align*}
    \mathbb{E}_{\mathbb{P}} \left[\left\Vert \nabla\log \frac{\hat{\mu}_t(X_{t_-})}{\hat{\mu}_{t_-}(X_{t_-})}\right\Vert^2\right]\lesssim & h_l^2 \left(\mathbb{E}_{\pid}\left[\left\Vert X\right\Vert^2\right] + d\right) + d h_l^2 L_{\kappa t}+ h_l^2 L_{\kappa t}^2 \mathbb{E}_{\mathbb{P}} \left[\left\Vert X_t-X_{t_-}\right\Vert^2\right]  +h_l L_{\kappa t}^2\sigma^2 d \frac{\alpha}{\alpha-2}.
\end{align*}
Substituting this expression into \eqref{eq:kl_path_bound_intermediate_heavy_tail_diffusion}, we have
\begin{align*}
    \kl\left(\mathbb{P}\;||\mathbb{Q}\right)\lesssim& \sum_{l=1}^M\int_{t_{l-1}}^{t_l} L_{\kappa t}^2 \ \mathbb{E}_{\mathbb{P}} \left[\left\Vert X_t-X_{t_-}\right\Vert^2\right]\md t+ \int_0^{T/\kappa} \Vert v_t\Vert^2_{L^2({\hat{\mu}}_t)} \md t\\
    &+ \sum_{l=1}^M\int_{t_{l-1}}^{t_l} \left(d h_l^2 L_{\kappa t} + h_l^2 \left(\mathbb{E}_{\pid}\left[\left\Vert X\right\Vert^2\right] + d\right) +h_l L_{\kappa t}^2\sigma^2 d \frac{\alpha}{\alpha-2}\right) \md t.
\end{align*}
Using the bound derived in \eqref{eq:intermediate_bound_discretisation_paths}, it follows
\begin{align*}
    \kl\left(\mathbb{P}\;||\mathbb{Q}\right)\lesssim& \sum_{l=1}^M \left( 1 + h_l^2\max_{[t_{l-1}, t_l]} L_{t}^2\right)
   \int_{t_{l-1}}^{t_l} \left\vert\Dot{\hat{\mu}}\right\vert_t^2\ \md t + \left(d h_l\int_{t_{l-1}}^{t_l} L_{\kappa t}^2 \ \md t\right) \left(1 + h_l\max_{[t_{l-1}, t_l]} L_{t}\right) \\
   &+ \sum_{l=1}^M\int_{t_{l-1}}^{t_l} \left(d h_l^2 L_{\kappa t} + h_l^2 \left(\mathbb{E}_{\pid}\left[\left\Vert X\right\Vert^2\right] + d\right) +h_l L_{\kappa t}^2\sigma^2 d \frac{\alpha}{\alpha-2}\right) \md t.\\
\end{align*}
Let $h = \max_{l\in\{1, \dots, M\}} h_l$, we can further simplify the previous expression to obtain
\small
\begin{align*}
    \kl&\left(\mathbb{P}\;||\mathbb{Q}\right)\lesssim \sum_{l=1}^M (1+h^2 L_{\max}^2)\int_{t_{l-1}}^{t_l} \left\vert\Dot{\hat{\mu}}\right\vert_t^2 \ \md t\ + d\ h(1+ h L_{\max}) \int_{t_{l-1}}^{t_l} L_{\kappa t}^2 \ \md t\\
    &\quad\quad\quad\quad+ \frac{T}{\kappa}h^2\left(\mathbb{E}_{\pid}\left[\left\Vert X\right\Vert^2\right] + d\right) + dh \frac{\sigma^2\alpha}{\alpha-2} \int_{0}^{T/\kappa} L_{\kappa t}^2 \ \md t\\
     \lesssim& (1+h^2 L_{\max}^2) \int_{0}^{T/\kappa} \left\vert\Dot{\hat{\mu}}\right\vert_t^2 \ \md t\ + 
 d\ h\left(1+ \frac{\sigma^2\alpha}{\alpha-2}+ h L_{\max}\right)\int_{0}^{T/\kappa} L_{\kappa t}^2 \ \md t + \frac{T}{\kappa} h^2\left(\mathbb{E}_{\pid}\left[\left\Vert X\right\Vert^2\right] + d\right)\\
 =&  (1+h^2 L_{\max}^2) \kappa \mathcal{A}_{\lambda}(\mu)  + 
\frac{d\ h}{\kappa} (1+ h L_{\max})\int_{0}^{T} L_{ t}^2 \ \md t+ \frac{T}{\kappa} h^2\left(\mathbb{E}_{\pid}\left[\left\Vert X\right\Vert^2\right] + d\right).
\end{align*}
\normalsize
The step size $h$ can be expressed in terms of the number of steps $M$ and $\kappa$ as $h\asymp \frac{1}{M\kappa}$. Therefore, we have
\small
\begin{align*}
    \kl\left(\mathbb{P}\;||\mathbb{Q}\right)\lesssim&\left(1+\frac{L_{\max}^2}{M^2\kappa^2}\right) \kappa \mathcal{A}_\lambda(\mu) + \frac{d}{M\kappa^2}\left(1+\frac{\sigma^2\alpha}{\alpha -2} +   \frac{L_{\max}}{M\kappa}\right)\int_{0}^{T} L_{ t}^2 \ \md t+ \frac{1}{M^2\kappa^3}\left(\mathbb{E}_{\pid}\left[\left\Vert X\right\Vert^2\right] + d\right)\\
    \lesssim& \left(1+\frac{L_{\max}^2}{M^2\kappa^2}+\frac{1}{M^2\kappa^4}\right) \kappa \left(\mathbb{E}_{\pid}\left[\Vert X\Vert^2\right] + d\right) + \frac{d}{M\kappa^2}\left(1 + \frac{\sigma^2\alpha}{\alpha -2}+  \frac{L_{\max}}{M\kappa}\right)\int_{0}^{T} L_{ t}^2 \ \md t,
\end{align*}
\normalsize
where we have used the bound on the action obtained in Lemma \ref{lemma:action_bound_heavy_tail_diffusion} and $T = \mathcal{O}(1)$. To conclude, note that 
\begin{align*}
    \kl\left(\mathbb{P}\;||\mathbb{Q}_\theta\right)\lesssim& \int_0^{T/\kappa}\mathbb{E}_{\mathbb{P}}\left[\left\Vert \nabla\log \hat{\mu}_t(X_t)-\nabla \log \hat{\mu}_{t_-}(X_{t_-}) + v_t(X_t)\right\Vert^2\right]\md t \\
    &+ \int_0^{T/\kappa}\mathbb{E}_{\mathbb{P}}\left[\left\Vert \nabla \log \hat{\mu}_{t_-}(X_{t_-}) - s_\theta(X_{t_-}, t_-)\right\Vert^2\right]\md t\\
    =& \kl\left(\mathbb{P}\;||\mathbb{Q}\right)+ \sum_{l=0}^{M-1} h_l\mathbb{E}_{\hat{\mu}_t}\left[\left\Vert \nabla \log \hat{\mu}_l(X_{t_l}) - s_\theta(X_{t_l}, t_l)\right\Vert^2\right] =\kl\left(\mathbb{P}\;||\mathbb{Q}\right)+\varepsilon_{\text{score}}^2\\
    \lesssim& \left(1+\frac{L_{\max}^2}{M^2\kappa^2}+\frac{1}{M^2\kappa^4}\right) \kappa \left(\mathbb{E}_{\pid}\left[\Vert X\Vert^2\right] + d\right) + \frac{d L_{\max}^2}{M\kappa^2}\left(1 + \frac{\sigma^2\alpha}{\alpha -2}+  \frac{L_{\max}}{M\kappa}\right)+\varepsilon_{\text{score}}^2\\
    \lesssim& \left(1+\frac{L_{\max}^2}{M^2\kappa^2}+\frac{1}{M^2\kappa^4}\right) \kappa \left(M_2 \vee d\right) + \frac{d L_{\max}^2}{M\kappa^2}\left(1 + \frac{\alpha}{\alpha -2}+  \frac{L_{\max}}{M\kappa}\right)+\varepsilon_{\text{score}}^2.
\end{align*}
We can conclude that by taking 
\begin{align*}
    \kappa = \mathcal{O}\left(\frac{\varepsilon_{\text{score}}^2}{M_2 \vee d}\right),\quad M = \mathcal{O}\left(\frac{d (M_2 \vee d)^2 L_{\max}^2}{\varepsilon_{\text{score}}^6}\right),
\end{align*}
we have that $\kl\left(\mathbb{P}\;||\mathbb{Q}_\theta\right)\lesssim\varepsilon_{\text{score}}^2$. Therefore, for any $\varepsilon = \mathcal{O}(\varepsilon_{\score})$, the heavy-tailed \gls*{DALMC} algorithm requires at most 
\begin{equation*}
     M = \mathcal{O}\left(\frac{d (M_2 \vee d)^2 L_{\max}^2}{\varepsilon^6}\right)
\end{equation*}
steps to approximate $\pid$ to within $\varepsilon^2$ in \gls*{KL} divergence.
\end{proof}